\documentclass[preprint,12pt,authoryear,table]{elsarticle}

\date{}
\usepackage{amsthm,amssymb,amsfonts}
\usepackage{amsmath}
 \usepackage{graphicx}
\usepackage{booktabs}
\usepackage{float}
\usepackage{comment}
\usepackage{natbib}
\usepackage{listings}
\usepackage[ruled,vlined]{algorithm2e}
\usepackage{mathrsfs} 
\usepackage{tikz}

\usepackage{microtype}
\usepackage{subcaption}

\usepackage[skip=0.5\baselineskip]{caption}

\usepackage{marginnote,datetime,enumitem,rotating,fancyvrb}
\usepackage{longtable, multirow, colortbl, graphicx, array}
\usepackage{xcolor}
\usepackage[margin=1.3in]{geometry} 

\usepackage{algorithmic}
\usepackage{wrapfig}
\input xy
\xyoption{all}

\usepackage{amsmath}
\usepackage{amssymb}
\usepackage{mathtools}
\usepackage{amsthm}

\usepackage[most]{tcolorbox}

\newtcolorbox{shadowboxtext}{
  colback=gray!10,       
  boxrule=1pt,           
  frame hidden,          
  arc=0pt,               
  drop shadow,           
  left=5pt,              
  right=5pt,
  top=5pt,
  bottom=5pt,
}

\newtheorem{assumption}{Assumption}
\newtheorem{definition}{Definition}
\newtheorem{proposition}{Proposition}

\newtheorem{remark}{Remark}
\newtheorem{lemma}{Lemma}

\newtheorem{theorem}{Theorem}

\usepackage[colorlinks,pagebackref=true]{hyperref}
\begin{document}

\begin{frontmatter}


\title{Learning Nonlinear Factor Models with Unknown Monotone Links from Incomplete and Noisy Data}





\author[tum]{Yutong Chao}
\ead{cyut@cit.tum.de}
\author[tum]{Resat Gökhan}
\ead{resat.goekhan@tum.de}
\author[tum,mirmi]{Jalal Etesami}
\ead{j.etesami@tum.de}
\author[vt]{Ali Habibnia}
\ead{habibnia@vt.edu}

\address[tum]{School of Computation, Information and Technology, Technical University of Munich, Germany}
\address[vt]{Department of Economics, Virginia Tech, USA}
\address[mirmi]{Munich Institute of Robotics and Machine Intelligence}

\tnotetext[t1]{}

\begin{abstract}
We study a nonlinear factor model in which observed responses depend on low-rank latent factors through an unknown monotone link function. This setting is challenging and largely underexplored due to severe nonconvexity and identifiability issues. The link function is assumed to lie in a reproducing kernel Hilbert space (RKHS), enabling flexible nonparametric modeling while preserving identifiability. We formulate the problem as the joint recovery of the low-rank factors, loadings, and the nonlinear link function from possibly incomplete and noisy observations and propose a projected block coordinate descent (BCD) algorithm with explicit regularization to address scale and rotational ambiguities. Under mild incoherence of factors and standard sampling conditions, we establish convergence guarantees in both noiseless and noisy regimes, along with sublinear regret bounds for the link-function updates. Our results extend classical linear factor models to a broad nonlinear regime and provide a principled framework for learning nonlinear latent structures. We evaluate the proposed approach using controlled synthetic experiments, indicating promising performance.
\end{abstract}




\end{frontmatter}



\section{Introduction}\label{sec:intro}
Advances in data collection technologies have led to the increasing availability of high-dimensional data, which are typically characterized by a large number of cross-sectional units and a high-dimensional feature space. Analyzing such data poses significant statistical and computational challenges. Factor models provide an effective dimension-reduction framework to address these challenges. Originally introduced in the field of psychometrics, factor models have since become a fundamental tool for summarizing and modeling complex, high-dimensional data sets, with applications in statistics, economics, and other data science disciplines \cite{ding2021high,fan2021recent,barigozzi2023quasi}.

At their core, factor models assume that high-dimensional observations are generated from an underlying low-dimensional latent structure, often represented by a low-rank matrix. In high-dimensional settings, this low-rank structure acts as a natural form of regularization and often leads to more interpretable representations, particularly in scientific applications. Beyond traditional statistical modeling, low-rank assumptions play a central role in many machine learning problems. Prominent examples include matrix regression \cite{negahban2011estimation,chen2015fast}, rank-$r$ principal component analysis \cite{johnstone2009consistency,birnbaum2013minimax}, low-rank and sparse matrix decomposition \cite{candes2011robust,chandrasekaran2011rank,hsu2011robust}, matrix completion \cite{candes2010matrix,sun2016guaranteed,zheng2016convergence}, and recommendation systems \cite{alshbanat2025survey,sankagiri2025recommendations}.
These problems can be unified under the following optimization problem with $\mathcal{L}$ as a loss function and  $r\ll n,T$,

\begin{align*}
\min_{B\in\mathcal{S}_1}\min_{F\in\mathcal{S}_2}\mathcal{L}(B^\top F),\quad \mathcal{S}_1\subseteq\mathbb R^{r\times n},\ \mathcal{S}_2\subseteq\mathbb R^{r\times T}. 
\end{align*}
Classical (linear) factor models, which have been extensively studied over the past decades, assume the following representation for the observed vector at time $t$:
\begin{align}\label{eq:linear_factors}
    y_t = B^\top f_t + u_t, \quad t\in[T],
\end{align}
where $y_t, u_t\in\mathbb R^{n}$, $f_t$ is the $t$-th column of $F$, and $u_t$ is the error term, commonly referred to as the \textit{idiosyncratic component}, which is assumed to be uncorrelated with both $B$ and $f_t$. The matrices $B$ and $F$ denote the latent \textit{loading factor} and latent \textit{factors}, respectively, and are the primary parameters of interest in literature of linear factor models.  Within this framework with Gaussian errors, the loss function $\mathcal{L}$ is often a quadratic loss which leads to the following optimization: 
$\min_{B,F} \sum_{i,t}(y_{i,t}-b_i^\top f_t)^2+\text{Re}(B,F)$, where $\text{Re}(B,F)$ denotes a regularization terms that could impose additional structure on the factors and loading factors. 

Linear factor models are widely used for estimating high-dimensional covariance and precision matrices. In particular, suppose that the factors $\{f_t\}$ in \eqref{eq:linear_factors} are realizations of a multivariate distribution with unknown stationary covariance $\text{cov}(f_t)$. It is straightforward to see from \eqref{eq:linear_factors} that the covariance matrix of the observed data can be expressed as $\hat\Sigma\!:=\!\text{cov}(y_t)\!\!=\!\!B^\top \text{cov}(f_t) B\!+\!\text{cov}(u_t)$, where the first term is low-rank and the second is often sparse, reflecting the weak dependence of idiosyncratic errors. 
A central objective in the linear factor model literature is to estimate $\text{cov}(y_t)$ accurately, especially in high-dimensional settings where the number of variables $n$ exceeds the number of observations $T$. 
Consequently, different methods have been proposed to exploit the factor structure of $\hat\Sigma$ in order to produce more accurate estimators. For instance, formulating the covariance matrix estimation under factor models as a low-rank plus sparse matrix decomposition problem \cite{barratt2023covariance}. 
Prominent application of this framework include portfolio management and risk assessment in financial economics \cite{de2021factor}, graphical models \cite{meinshausen2006high,zhao2024high}, high dimensional classification \cite{hastie2009elements}, covariance estimation \cite{kereta2021estimating,sihag2022covariance}, and testing
the capital asset pricing model \cite{sentana2009econometrics}.

A key assumption underlying linear factor models is that the observed data are complete; that is, $y_{i,,t}$ is observed for all $i\in[n]$ and $t\in[T]$. In practice, however, this assumption is often violated due to missing observations. One contribution of this work is to relax this assumption, allowing entries to be observed only at random indices and random time points.

While successful in many applications, linear factor models cannot capture nonlinear dynamics in real-world data. As an illustrative example, consider the problem of learning from comparison data, where the goal is to estimate users’ preferences over items based on observed pairwise comparisons. A widely used approach models such comparisons using probabilistic choice models, such as the Bradley–Terry–Luce model \cite{maystre2015fast,negahban2012iterative,sankagiri2025recommendations}. In this setting, the probability that user $u$ prefers item $i$ over item $j$ is given by $\sigma(x_{u,i}\!-\!x_{u,j})$, where $\sigma$ is the sigmoid \textit{link function}. It is further assumed that users and items are embedded in a low-dimensional latent space $\mathbb R^r$. 
Interpreting $B$ as the matrix of user feature vectors and $F$ as the matrix of item feature vectors, implies $x_{u,i}\!-\!x_{u,j}\!=\!b^\top_u(f_i\!-\!f_j)$. Consequently, the recommendation problem reduces to learning $B$ and $F$ by optimizing a loss function of the form $\mathcal{L}(\sigma(\langle S_k,B^\top F\rangle))$, where $S_k$ is the sampling matrix encoding the random observed comparisons \cite{sankagiri2025recommendations}. 
A crucial distinction in this setting is that the observations are no longer linear due to the presence of the link function. To manage this complexity, existing works typically assume that either this function is known a priori \cite{saha2021optimal,sankagiri2025recommendations} or the latent factors are observed \cite{caner2025deep}. 
In contrast, this work strengthens the existing literature by removing both of these assumptions.


\section{Related Work}\label{sec:related}

\textbf{Linear Factor Models.}
As noted earlier, linear factor models have been extensively studied in the context of covariance matrix estimation. For example, \citet{fan2013large} introduce the Principal Orthogonal Complement Thresholding (POET) estimator, which exploits a low-rank plus sparse decomposition of the covariance matrix.
\citet{caner2023sharpe} propose a node-wise regression approach for covariance matrix estimation and establish consistency results for precision matrix estimation under a linear factor model. 
An alternative is the shrinkage-based methods. Linear and nonlinear shrinkage estimators, developed by \citet{ledoit2012nonlinear, ledoit2022power}, improve the sample covariance matrix by shrinking it toward a structured target. In the linear framework, the estimator takes the form 
$\lambda I + (1-\lambda)\hat\Sigma$, where $\lambda$ is the shrinkage intensity. In contrast, nonlinear shrinkage estimators apply data-driven, nonlinear transformations to each eigenvalue of $\hat\Sigma$ individually.

Recently, \citet{fan2024factor} propose the Factor-Augmented Sparse Throughput (FAST) model, which combines latent factors with sparse idiosyncratic components for nonparametric regression. Using diversified projections to estimate the latent factors, the authors employ truncated deep ReLU networks for nonparametric factor regression. This approach, however, relies on a linear factor model assumption with known factors. Beyond estimation, there is a growing literature on the statistical properties and applications of linear factor models; see, for example, \cite{filipovic2024fundamental,fletcher2024examination}.
Dynamic linear factor models extend static linear models by allowing factor dependence to operate through time-varying filters rather than fixed loading factors, in a manner analogous to autoregressive models \cite{breitung2006dynamic, altissimo2010new}. Such dynamic models have been  applied to high-dimensional time-series data, including structural economic analysis and macro-financial modeling \cite{barigozzi2016non}.

\textbf{Matrix Completion \& Factorization.}
One component of our problem, namely, recovering the factors $F$ and loadings $B$ from their subsampled and masked (due to the link function) product, that is, from a small subset of entries of $\phi(B^\top\! F)$, is closely related to the matrix completion problem. Matrix completion concerns the recovery of a low-rank matrix from a small subset of its entries, potentially corrupted by noise. 
The low-rank structure is typically enforced either through nuclear-norm regularization \citep{candes2009exact, candes2010matrix, negahban2012restricted} or via an explicit matrix factorization formulation. The latter is motivated by the observation that any low-rank matrix $X$ can be factorized as $X\!=\! U^\top V$ \citep{mnih2007probabilistic}. Although this reformulation leads to a nonconvex optimization problem, it is computationally more efficient and has demonstrated strong empirical performance on real-world datasets \citep{koren2009matrix}. 
This empirical success has motivated several theoretical studies on the convergence properties of the resulting nonconvex optimization problem \cite{keshavan2010matrix,chen2015fast,sun2016guaranteed, zheng2016convergence}. 
A common technical ingredient across this literature is a concentration result by \citet{candes2009exact}, which relies on two key assumptions: (i) the ground-truth matrix is incoherent, ensuring that no single row or column dominates the matrix, and (ii) the observed entries are sampled uniformly at random. Moreover, these methods require that the iterates remain incoherent throughout the optimization process. To enforce this property, prior works introduce either explicit regularization \citep{sun2016guaranteed} or projection steps \citep{chen2015fast,zheng2016convergence, ma2018implicit}.
In this work, we adopt a similar strategy to recover the latent factors and loadings and thus rely on the same assumptions with the additional challenge of simultaneously learning the unknown link function.

\textbf{Nonlinear Factor Models.}
To overcome the linearity assumption, \citet{caner2025deep} consider a nonlinear factor model $y_t\!=\!\varphi(f_t)\!+\!u_t$ (slightly different from our model in \eqref{eq:model_1}) for asset returns in large portfolios. In their model the link function $\varphi$ in unknown but the responses $\{y_t\}$ and the factors $\{f_t\}$ are observed. As a result, their problem reduces to a regression task, which they solve using neural networks and prove the consistency of their solution. 
In fact, several works study similar formulations. For example, the multi-index model aims to recover both an unknown index space $B$ (loading factors in our formulation) and an unknown link function from observed pairs $(y_t,f_t)$. Multi-index models have been widely used to mitigate the curse of dimensionality in high-dimensional regressions \citep{li2022dimension, klock2021estimating, steffen2025pac}. Notably, \citet{klock2021estimating} establish a connection between multi-index models and sufficient dimension reduction methods \citep{lee2013general, li2018sufficient}, and propose an estimator for the index space along with sharp concentration bounds.

The nonlinear factor model considered in this paper, e.g., Equation \eqref{eq:model_1} has also been studied by \citet{chen2014nonlinear, chen2021nonlinear}, with the key distinction that the nonlinear link function is assumed to be \textit{known} in those works. Specifically, \citet{chen2014nonlinear} study estimation and inference in semiparametric nonlinear panel single-index models via maximum likelihood. By proposing an iterative two-step likelihood maximization procedure and showing that the objective is concave in each step, they establish convergence to a local optimum for this class of models. In contrast, \citet{chen2021nonlinear} develop an EM-type algorithm, building on \cite{chen2016estimation}, to jointly estimate the factors and loading factors. However, neither of these works provides theoretical guarantees for the convergence  of their proposed algorithms.

\paragraph{Our Contributions} 
We study nonlinear factor models with an unknown link function and propose a projected gradient–based method that jointly learns the latent factors, loading factors, and the link function. We consider several observation settings, including complete, random, noisy, and noiseless. The random observation setting, which allows for randomly missing entries, is more realistic than even the classical linear factor model assumption of fully observed data. 
By assuming a nonparametric link function that lies in an RKHS with smooth and bounded kernels, and imposing mild incoherence conditions on the factors and loadings (which are standard assumptions for handling missing data, borrowed from the matrix completion literature), we provide a theoretical analysis of the proposed algorithm. 
While the primary contribution of this work is to lay the foundation for learning nonlinear factor models from incomplete and noisy data, we also present experimental results that support our theoretical findings and analyze the sensitivity of the method with respect to its hyperparameters.


\section{Model \& Algorithm }\label{problem:setting}

\subsection{The Data Generation Process}
There is a finite number of factors, $\{f^*_1,...,f^*_T\}$ and finite number of loading factors $\{b^*_1,...,b^*_n\}$ such that $f_t, b_i\in\mathbb R^r$ for all $t\in[T]$ and $i\in[n]$ and $n,T\gg r$.
The date generating process is given by
\begin{align}\label{eq:model_1}
y_{i,t}&=\phi^*\big( (b^*_i)^\top f^*_{t}\big) + u_{i,t}, 
\end{align}
where $\phi^*:\mathbb R\rightarrow\mathbb R$ is the \textit{link function} and $u_{i,t}$ denotes the error, also referred to as the idiosyncratic component, and is assumed to be uncorrelated with the factors. Errors are distributed according to a distribution $P$ with variance bounded by $\sigma^2$. 
Depending on wether $\sigma^2$ is positive or zero, we consider the \textit{noisy} and \textit{noiseless} settings, respectively. 
$y_{i,t}$ is the observed response for the $i$-th individual at time $t$. 

We distinguish between two settings depending on the observation pattern: \textit{complete observation}, where $y_{i,t}$ is observed for all 
$i\in[n]$ and $t\in[T]$, and \textit{random observation}, where observations are available only for random subsets of $[n]\times [T]$. 
The random observation setting is intended to model realistic scenarios with randomly missing data. 

To have a unified model for both complete and random observation settings, we use $y_k$ to denote the $k$-th observation,
\begin{align}\label{eq:model_2}
        y_{k}&=
        \phi^*\big([X^*]_{i_k,t_k}\big)+ u_{k},\quad k\in[M],
\end{align}
where $X^*= (B^*)^\top F^*\in\mathbb R^{n\times T}$, $B^* = [b_1,...,b_n]$, and $F^* = [f_1,...,f_T]$. 
In the above equation, $M (\leq nT)$ denotes the number of observations. Indices $i_k\in[n]$ and $t_k\in[T]$ represent the factor and loading factor, respectively in the  $k$-th observation.  
We also use $\Omega:=\{(i_k,t_k): \ k\in [M]\}$ to denote the set of  observed indices. Then, the complete observation setting is when $\Omega=[n]\times [T]$.

\subsection{Problem Setting}
In this paper, we focus on recovering $(\phi^*, B^*, F^*)$ from the observed data $\{y_k\}$ and observed indices $\Omega$. 
To this end, let $U^* \Sigma^*_r (V^*)^\top$ be the rank-$r$ SVD of matrix $X^*$, that is $U^*\in\mathbb R^{n\times r}, V^*\in\mathbb R^{T\times r}$ and $\Sigma^*\in\mathbb R^{r\times r}$ such that $(U^*)^\top U^*=(V^*)^\top V^*=I_r$ and $\Sigma^*$ is a diagonal matrix with entries $\sigma_1^* \geq \ldots \geq \sigma_r^* > 0$.

We define $Z^*$ and $Y^*$ as
\begin{align*}
    Z^*&:=\left[\begin{matrix}
        U^*\\
        V^*
    \end{matrix}\right](\Sigma_r^*)^{1/2}\ \  \mathbb\in R^{(n+T)\times r},\\ 
    Y^*&:=Z^*(Z^*)^\top = \left[\begin{matrix}
        U^*\Sigma^* (U^*)^\top& X^*\\
        (X^*)^\top & V^*\Sigma^* (V^*)^\top
    \end{matrix}\right].
\end{align*}
Using this notation, we can rewrite the $k$-th observation in \eqref{eq:model_2} as follows
\begin{align}\label{eq:model_c}
        y_{k}&=\phi^*\big(\langle A_{k}, Y^*\rangle \big) + u_{k}, \quad k\in [M],
\end{align}
where $\langle A, B \rangle :=\text{tr}(A^\top B)=\sum_{i,j}a_{i,j}b_{i,j}$ denotes the matrix inner product and $A_k$ is the sa\textit{mpling matrix} for the pair $k:=(i_k,t_k)\in\Omega$ defined by
\begin{align}\label{eq:sample_matrix}
    A_k:=\left[\begin{matrix}
        0 & e_{i_k} \tilde{e}^\top_{t_k}\\
        0 & 0
    \end{matrix}\right]\in\mathbb R^{(n+T)\times (n+T)},
\end{align}
where $e_{i_k}$ and $\tilde{e}_{t_k}$ denote unit vectors in $\mathbb R^n$ and $\mathbb R^{T}$, respectively. In the above equation, $0$ denotes matrices with all entries zero of the appropriate size. 
In the random observation setting, $\{A_k\}$ are i.i.d. random matrices of form \eqref{eq:sample_matrix}, with the index $i_k$ being chosen uniformly at random from $[n]$, and $t_k$ being chosen uniformly at random from $[T]$.

Given the relation between matrices $X^*$, $Y^*$, and $Z^*$, estimating the ground-truth factors and loading factors is equivalent to estimating $Z^*$.
The major advantage of this reformulation is that it reduces the number of parameters from $nT$ in $X^*$ to $(n + T)r$ in $Z^*$.

\paragraph{\textbf{Condition Number}} 
We define $\kappa:=\sigma^*_1/\sigma^*_r$ as the \textit{condition number} of the data. Note that $\kappa$ is also the condition number of $Z^*$ as the singular values of $Z^*$ are $\sqrt{2\sigma^*_1}, \ldots \sqrt{2\sigma^*_r}$.

\paragraph{\textbf{Incoherence}} For matrix $Z^*$ with $n+T$ rows, let $\|Z^*\|_{2, \infty}$ denote the maximum of the $\ell_2$ norm of its rows and let $\|Z^*\|_{F}$ denote the Frobenius norm of $Z^*$. Define the \textit{incoherence parameter} of the matrix as 
\begin{align}\label{eq:def_mu}
    \mu \triangleq (n+T)(\|Z^*\|_{2, \infty}^2/\|Z^*\|_{F}^2).
\end{align}
In principle, $\mu$ can take values from $1$ to $n+T$. However, the sample complexity worsens with $\mu$.

\paragraph{\textbf{Link Function}}  
The unknown link function $\phi^*(\cdot)$ is assumed to be monotone and to belong to a reproducing kernel Hilbert space, defined below, associated with the kernel $K(\cdot,\cdot)$. For details see \cite{scholkopf2018learning}.

\begin{definition}\label{def:rkhs}
Let $\mathcal X$ be a nonempty set and let $K\!:\!\mathcal X\!\times\!\mathcal X\!\to\!\mathbb R$ be a symmetric positive semidefinite kernel. A Hilbert space $(\mathcal H,\langle\cdot,\cdot\rangle_{\mathcal H})$ of real-valued
functions on $\mathcal X$ is called a \emph{reproducing kernel Hilbert space (RKHS)} with kernel $K$ if:
i) for every $x\!\in\!\mathcal X$, the function $K(x,\cdot)\in\mathcal H$ and ii) for every $f\in\mathcal H$ and every $x\in\mathcal X$,
$f(x)=\langle f, K(x,\cdot)\rangle_{\mathcal H}$. The induced norm in this space is defined by $\|f\|_\mathcal{H}:=\langle f, f\rangle_{\mathcal H}$.
\end{definition}

Given constants $\xi, \Xi>0$, we also define $\mathcal{H}_{\xi,\Xi}\subset\mathcal{H}$ as the subset of functions in 
$\mathcal{H}$ with bounded derivatives, i.e.,
\begin{align}\label{eq:subspace_H}
    \mathcal{H}_{\xi,\Xi}:=\{&f\in\mathcal{H}: \ 0< \xi\leq f'(x)\leq \Xi<\infty\}.
\end{align}
Since this set is convex, the projection onto it is well-defined $\mathcal{P}_{\mathcal{H}_{\xi,\Xi}}(\phi):=\arg\min_{f\in\mathcal{H}_{\xi,\Xi}}\|f - \phi\|_\mathcal{H}$. As there is no closed form solution to this projection, a practical approach is to discretize the support of $\phi$ and impose inequality constraints to ensure that its derivative remains between $\xi$ and $\Xi$. This reduces the problem to a quadratic program with box constraints. Nevertheless, our empirical study show that Algorithm \ref{alg:pbcd} converges even without any projections. 

\begin{assumption}\label{ass:rkhs_1}
We assume that the true link function $\phi^*$ belongs to $\mathcal{H}_{\xi,\Xi}$ with bounded kernels, i.e., there exists a constant $B_K>0$ such that $\sup_{x\in \mathcal X} K(x,x) \le B_K$.
\end{assumption}

Under this assumption, the infinite-dimensional problem of estimating the link function can be reduced to a finite-dimensional one via the representer theorem \cite{scholkopf2018learning}. This approach is standard in RKHS-based methods and is widely used in the machine learning literature when the parameter of interest is a continuous function.

\subsection{The Loss Function}
Recall that our goal is to estimate the pair $(\phi^*,Z^*)$. 
We do so by maximizing the log-likelihood. 
Equivalently, we formulate a loss function in terms of the negative log-likelihood, and minimize this function using a gradient-based method. 
Under the assumption that the errors are independent and identically distributed according to a sub-Gaussian distribution, the negative log-likelihood function is given by
\begin{align*}
    {\mathcal{L}}(\phi, Z)&:=\frac{1}{M}\sum_{k=1}^M \big(y_{k} - \phi(\langle A_k, ZZ^\top\rangle )\big)^2.
\end{align*}

\paragraph{\textbf{Tikhonov Regularization}}
Function $\phi$ belongs to $\mathcal H$, which can have many degrees of freedom, possibly infinite. To favor smoother and simpler representations within $\mathcal H$, it is common to regularize the objective function using $\|\phi\|_\mathcal{H}$, also known as Tikhonov regularization \cite{scholkopf2018learning}.
This is because $\|\phi\|_\mathcal{H}$ measures the “roughness” of $\phi$, and the regularization penalizes large coefficients in the RKHS expansion of the solution, thereby encouraging smoother solutions. Another important consequence is that this regularizer enables the use of the representer theorem, which guarantees that the optimal solution is unique and can be expressed as a finite linear combination of kernels evaluated at the observed data points.

\paragraph{\textbf{Scale Invariance}}
The generative model, and consequently the log-likelihood function, is invariant to a certain transformation in the $Z$ matrix. Below, we explore such symmetry and introduce a regularizer to favor solutions that have two factors with comparable second-order magnitudes. 

Note that the mapping from a matrix $X$ to its factor representation $Z=[U; V]$ is generally non-identifiable. In particular, for any invertible $P\in\mathbb{R}^{r\times r}$, the reparameterization
$\tilde Z=[UP^\top ; VP^{-1}]$ yields the same matrix $X$, and therefore the same likelihood, i.e., ${\mathcal{L}}(\phi, Z)={\mathcal{L}}(\phi, \tilde Z)$.
To distinguish ``imbalanced'' factorizations from ``balanced'' ones, we augment the loss with the regularizer $\|U^\top U - V^\top V\|_F^2$, which encourages solutions in which the two factors have comparable second-order magnitudes. 
In compact form, this regularizer can be written as
\begin{equation}
\mathcal{R}(Z):= \|Z^\top D Z\|_F^2,\qquad
D :=
\begin{bmatrix}
{I_{n}} & 0\\
0 & -{I_{T}}
\end{bmatrix}.
\end{equation}

\begin{remark}
Similar assumption has also been imposed in the context of linear factor models \cite{fan2013large}. In particular, it is commonly assumed that the covariance matrix of the factors is diagonal and that the covariance matrix of the loading factors is the identity matrix, or vice versa. These assumptions ensure the identifiability and tractability of recovering the factors and loading factors.
\end{remark}
Overall, the loss function that we use in this work to recover the unknown parameters is given by 
\begin{align}\label{eq:loss}
        \tilde{\mathcal{L}}(\phi, Z)&:={\mathcal{L}}(\phi, Z) + \frac{\lambda}{4}\mathcal R(Z) + \frac{\alpha}{2}\|\phi\|_{\mathcal{H}}^2,
\end{align}
where $\lambda$ and $\alpha$ are the regularization coefficients, which will be specified later.

\paragraph{\textbf{Equivalent Solution Set}}
Beyond scale invariance, the problem exhibits rotational symmetry in the latent dimension.
Let $R \in \mathbb{R}^{r \times r}$ be any orthogonal matrix, i.e., $RR^\top= I_r$.
Then, $ZR=[UR; VR]$ produces exactly the same loss function, i.e., $\tilde{\mathcal{L}}(\phi, Z)=\tilde{\mathcal{L}}(\phi,  ZR)$. 
Hence, the true features cannot be uniquely recovered; they are determined only up to an orthogonal rotation.
Accordingly, we define the equivalence class of ground-truth feature matrices as
\begin{equation*}
\mathcal{E} := \left\{ Z^\star R:\ R\in\mathbb{R}^{r\times r},\ RR^\top=I \right\}.
\end{equation*}
Using this equivalent solution set, we can measure the "goodness" of an arbitrary $Z$ via $\|\Delta(Z)\|$, where
\begin{align*}
    \Delta(Z)&:=Z-\Phi(Z),\qquad \Phi(Z):=Z^*R(Z),\\
    R(Z)&:=\arg\min_{R: RR^\top=I}\|Z-Z^*R\|.
\end{align*}

\subsection{Algorithm}
As mentioned in the previous section, we apply a gradient-based method to optimize the loss function in \eqref{eq:loss} which is presented in Algorithm \ref{alg:pbcd}.
However, for proving theoretical guarantees, we need to use \textit{projected gradient descent}. 
Notably, two projection steps are designed: first projecting $Z$ onto a set of "incoherent matrices" $\mathcal{C}$, defined below and then projecting $\phi$ onto $\mathcal{H}_{\xi,\Xi}$ (defined in \eqref{eq:subspace_H}). 
\begin{align}\label{eq:set_c}
    \!\!\mathcal{C}\! :=\!\Big\{\!Z\!\! \in \! \mathbb{R}^{(n+T) \times r}\! :  \|Z\|_{2, \infty}\! \leq\! \frac{4}{3}\sqrt{\frac{\mu}{n+T} }\|Z_0\|_F\! \Big\}\!.
\end{align}
Note that this is a convex and closed set containing matrices that are nearly as incoherent as $Z^*$ when $\|{Z_0}\|_F \approx \|{Z^*}\|_F$. 
For any $Z$, the projection of $Z$ onto $\mathcal{C}$ is  obtained by clipping the rows of $Z$ as follows:
\begin{align*}
    \forall \ j \in [n+T], \ \mathcal{P}_{\mathcal{C}}(Z)_j &= 
    \begin{cases}
        Z_j & \text{if } \|{Z_j}\|_{2} \leq \beta,\\
         \frac{\beta}{\|{Z_j}\|_{2}}Z_j & \text{otherwise},
    \end{cases}
\end{align*}
where $\beta = (4/3)\sqrt{({\mu}/{n+T})}\|{Z_0}\|_F$. 
As noted in \cite{sankagiri2025recommendations}, the projection onto $\mathcal C$ is often unnecessary in practice; however, it is essential for deriving the theoretical guarantees.
The second projection ensures that the learned link function preserve monotonicity in $\mathcal H$.

\begin{algorithm}[h]
\caption{Projected BCD in RKHS}\label{alg:pbcd}
\begin{algorithmic}
\STATE {\bfseries Input:} Initial link function $\phi_0$, initial matrix $Z_0 \in \mathbb{R}^{(n+T) \times r}$, stepsizes $\zeta$ and $\eta$
\STATE $Z_0 \gets \mathcal{P}_{\mathcal{C}}\left(Z_0\right)$
\FOR{$t=0,...,\bar T-1$}
    \STATE $Z_{t+1} \gets \mathcal{P}_{\mathcal{C}}\left(Z_t - \zeta \nabla_Z \tilde{\mathcal{L}}(\phi_t,Z_t) \right)$
    \STATE $\phi_{t+1}(\cdot) \gets \mathcal{P}_{\mathcal{H}_{\xi,\Xi}}\left(\phi_t(\cdot) - \eta\nabla_{\phi}\tilde{\mathcal{L}}(\phi_t,Z_{t+1})(\cdot)\right)$
\ENDFOR
\STATE {\bfseries Output:} $Z_{\bar T}, \phi_{\bar T}(\cdot)$
\end{algorithmic}
\end{algorithm}
In this algorithm, the gradients are given by
\begin{align*} 
    &\nabla_Z \tilde{\mathcal{L}}(\phi, Z)\! :=\! -\frac{2}{M}\!\!\sum_{k=1}^{M}g_k\phi'(x_k)
      (A_k\!+\! A_k^\top) Z \!+\! \lambda DZZ^\top\! DZ, \\
     &\nabla_{\phi} \tilde{\mathcal{L}}(\phi, Z)(\cdot) :=   -\frac{2}{M}\sum_{k=1}^M
    g_k K\big(x_k,\cdot\big)\!+\!\alpha \phi(\cdot).
\end{align*}
where $x_k:=\langle A_{k}, ZZ^\top\rangle$ and $g_k:=y_{k} - \phi\big(x_k\big)$.

\paragraph{\textbf{Initialization}} 
Let $\mathcal{B}(\epsilon):= \{Z: \|{\Delta(Z)}\|_{F}^2 \leq \epsilon\sigma^*_r\}$ denote a neighborhood around the true solution. We assume that $Z_0\in\mathcal{B}(\epsilon)$ for some $\epsilon>0$. Following the initialization schemes in \cite{zheng2016convergence, sankagiri2025recommendations}, we initialize $Z_0$ via a rank-$r$ SVD of the data matrix, with missing entries filled with zeros.

\section{Theoretical Results}\label{sec:results}
In this section, we present our theoretical results under different settings. We begin with the noiseless observation setting, which is further divided into complete and random cases. We then present our results for the noisy setting.

\subsection{Noiseless Observation}
Recall that in this setting, the idiosyncratic components have zero variance, i.e., $y_{i,t}=\phi^*((b_i^*)^\top f_t^*)$. 
We begin by showing the properties of the loss function.

\begin{lemma}\label{lem:strong_convexity}
   For any fixed $Z$, $\tilde{\mathcal{L}}(\phi,Z)$ is $\frac{\alpha}{2}$-strongly convex with respect to $\phi$, i.e., for any $\phi_1, \phi_2\in \mathcal{H}$,
\begin{align*}
\tilde{\mathcal{L}}(\phi_1,Z)&\geq \tilde{\mathcal{L}}(\phi_2, Z) + \langle \nabla_\phi\tilde{\mathcal{L}}(\phi_2, Z), \phi_1-\phi_2\rangle_\mathcal{H} \\
&+ \frac{\alpha}{2} \|\phi_1-\phi_2\|^2_\mathcal{H}.
\end{align*} 
\end{lemma}
This follows directly from the definition of the loss function.
\begin{lemma}\label{lem:smoothness_phi}
    For any fixed $Z$,  $\tilde{\mathcal{L}}(\phi, Z)$ is $L_{\phi,\alpha}$-smooth with respect to $\phi$, i.e., for any $\phi_1, \phi_2\in \mathcal{H}$,
    \begin{align*}
\tilde{\mathcal{L}}(\phi_1, Z)&\leq \tilde{\mathcal{L}}(\phi_2, Z) + \langle \nabla_\phi\tilde{\mathcal{L}}(\phi_2, Z), \phi_1-\phi_2\rangle_\mathcal{H}\\
&+ L_{\phi,\alpha} \|\phi_1-\phi_2\|^2_\mathcal{H}.
\end{align*} 
where $L_{\phi,\alpha}:=\sqrt{8B_K^2+2\alpha^2}$.
\end{lemma}
Proof is in Appendix \ref{pr:lem:smoothness_phi}. 
Next, we show that both $\nabla_\phi \tilde{\mathcal{L}}$ and $\nabla_Z \tilde{\mathcal{L}}$ are Lipschitz continuous with respect to $Z$ for any fixed $\phi$. To establish this result, we assume that the kernel function itself is Lipschitz continuous and smooth. This assumption is satisfied, for example, by RKHSs induced by Gaussian kernels.
\begin{assumption}\label{ass:kernel_smootheness}\label{ass:phi_smoothness}
    We assume that for any $Z_1,Z_2\in\mathcal{C}$ and any $A_k$,  there exist constants $L_K,L_{K'}>0$, such that 
    \begin{align*}
        &\|K(x_{1k} ,\cdot)\!-\!K(x_{2k},\cdot)\|_\mathcal{H}\leq L_{K}|x_{1k}-x_{2k}|,\\
        &\|\nabla_x K(x ,\cdot)|_{x=x_{1k}}\!-\!\nabla_x K(x ,\cdot)|_{x=x_{2k}}\|_\mathcal{H}\!\leq\! L_{K'}|x_{1k}\!-\!x_{2k}|,
    \end{align*}
    where $x_{1k}:=\langle A_k, Z_1Z_1^\top\rangle$ and $x_{2k}:=\langle A_k, Z_2Z_2^\top\rangle$.
\end{assumption}

\begin{lemma}\label{lem:smoothness_z_to_phi}
    Under Assumption \ref{ass:kernel_smootheness}, for any $Z_1,Z_2\in \mathcal{C}$ and any $\phi\in\mathcal{H}_{\xi,\Xi}$ such that $\tilde{\mathcal{L}}(\phi, Z_1),\tilde{\mathcal{L}}(\phi, Z_2)\leq\!\tilde L_{\max}$, where $\tilde L_{\max}$ is a constant, we have
    \begin{equation*}
        \|\nabla_\phi \tilde{\mathcal{L}}(\phi,Z_1)-\nabla_\phi\tilde{\mathcal{L}}(\phi,Z_2)\|_{\mathcal{H}}\leq L_{Z\to\phi}\|Z_1-Z_2\|_F,
    \end{equation*}
    where $L_{Z\to\phi}>0$ depends on $L_K$, $B_K$ and $L_{\max}$. 
\end{lemma}
Proof is in Appendix \ref{lem:smoothness_z_to_phi_proof}. In the paper, we can simply set $\tilde L_{\max}=\tilde{\mathcal{L}}(\phi_0,Z_0)$ when applying Algorithm \ref{alg:pbcd}.

\begin{lemma}\label{lem:Z-smoothness}
    Under Assumption \ref{ass:phi_smoothness}, for any $Z_1,Z_2\!\in\! \mathcal{C}$ and any $\phi\in\mathcal{H}_{\xi,\Xi}$ such that $\tilde{\mathcal{L}}(\phi, Z_1),\tilde{\mathcal{L}}(\phi, Z_2)\leq\!\tilde L_{\max}$, there is $L_Z\!\!>\!0$ such that
    \begin{align*}
        \|\nabla_Z \tilde{\mathcal L}(\phi,Z_1)-\nabla_Z \tilde{\mathcal L}(\phi,Z_2)\|_F
\le L_Z\|Z'-Z\|_F.
    \end{align*}
\end{lemma}
Proof is in Appendix \ref{pr:lem:Z-smoothness}.
As noted earlier, the loss function is non-convex in $Z$. However, when $\phi$ is fixed, the subsequent results characterize its local properties, which in turn facilitate the convergence analysis of the Algorithm \ref{alg:pbcd}.

\begin{lemma}\label{le:local_curvature}
Suppose $\phi$, $\phi^\star\in\mathcal{H}_{\xi,\Xi}$ and $Z\in\mathcal{B}(\epsilon)$, where $\epsilon=\xi^2/(20\Xi^2)$.
i) Complete observation: we have
\begin{align}\label{eq:PL1}
\!\!\langle \nabla_Z \tilde{\mathcal{L}}(\phi, Z), \Delta\rangle\!\! \geq\!\mu_Z\|\Delta\|_F^2\! -\!A_Z\|\Delta\|_F\|\phi\!-\!\phi^\star\|_\mathcal{H},
\end{align}
    where $A_Z:=8\Xi \sqrt{\frac{\sigma_1^*}{nT}}B_K$ and $\mu_Z:=\frac{\xi^2\sigma_r^*}{5nT}$.
ii) Random observation: when the number of observations is
\begin{align}\label{eq:number_observation}
    M\in\mathcal{O}\Big(\frac{(\mu r\kappa)^2\min\{n,T\}}{\epsilon^2}\log\big(\frac{n+T}{\delta}\big)\Big),
\end{align}
 with probability at least $1-\delta$, the above inequality holds. 
\end{lemma}
Proof is in Appendix \ref{pr:le:local_curvature}.
In the random observation setting, the required number of observations $M$ given in \eqref{eq:number_observation} matches the order established in matrix completion results \citep{candes2009exact, zheng2016convergence}.

\begin{lemma}\label{lemma:smoothness}
    Under the Assumptions of Lemma \ref{le:local_curvature}, for
    i) complete observation: we have
    \begin{align}\label{eq:PL2}
        \|\nabla_Z \tilde{\mathcal{L}}(\phi, Z)\|_F^2
    \leq B_Z\|\Delta\|_F^2+A'_Z\|\phi-\phi^*\|_\mathcal{H}^2.
    \end{align}
where $B_Z:=\frac{1093\Xi^4(\sigma_1^*)^2\mu r}{n^2T^2}$ and $A'_Z:=\frac{336\Xi^4\mu r \sigma_1^*B_K}{ nT\xi^2}$. For ii) random observation: when the number of observations is
as stated in Lemma \ref{le:local_curvature}, then with probability at least $1-\delta$, the above inequality holds. 
\end{lemma}
Proof is in Appendix \ref{pr:lemma:smoothness}. 
Note that both Lemmas \ref{le:local_curvature} and \ref{lemma:smoothness} assume that $Z\in\mathcal B(\epsilon)$. This is ensured by the initialization assumption, i.e., $Z_0\in\mathcal B(\epsilon)$ together with the result of Proposition \ref{lemma:d_bounded} that $\|\Delta_t\|_F$ remains bounded.
 Using the above results and by introducing a Lyapunov potential function, see below, we establish the convergence of Algorithm~\ref{alg:pbcd}. 
To define the Lyapunov function, we first introduce the $\phi$-subproblem minimizer. Specifically, for any fixed $Z$, let 
$
\phi^\sharp(Z):= \arg\min_{\phi\in {\mathcal{H}}} \tilde{\mathcal{L}}(\phi, Z).   
$
Since the objective function is strongly convex with respect to $\phi$, (Lemma \ref{lem:strong_convexity}), then $\phi^\sharp(Z)$ is the unique stationary point of $\tilde{\mathcal{L}}$ for a given $Z$. Define the bias radius as 
\begin{equation}\label{eq:epsalpha}
\chi(Z) := \|\phi^\sharp(Z)-\phi^\star\|_{\mathcal{H}}\ \in [0,\infty).
\end{equation}
We \emph{do not} require a closed-form expression for $\chi$, only its existence and finiteness.
But if a model-specific bound is available for $\chi$ (e.g., in terms of kernel eigenvalues and $\alpha$), it can be used to tighten the final error bounds.

\paragraph{\textbf{Lyapunov-based Analysis}}
Define the Lyapunov potential function as
\begin{equation}\label{eq:Lyap-def}
\mathcal V_t:=E_t+\gamma D_t,
\end{equation}
where $E_t:=\|\phi_t-\phi^\sharp(Z_t)\|_\mathcal{H}^2$ and $D_t:=\|\Delta(Z_t)\|_F^2$. The constant
$\gamma>0$ will be chosen explicitly later. 
 The term $D_t$ quantifies the closeness of $Z_t$, at iteration $t$, to the equivalent solution set $\mathcal E$. The term $E_t$ measures the closeness of the $t$-th estimate of the link function to the subproblem minimizer corresponding to $Z_t$. 
 Next result shows that $D_t$ remains bounded and $E_t$ converges to zero. 

 \begin{proposition}\label{lemma:d_bounded}\label{prop:E_converge}
The sequence $D_t$ generated by Algorithm \ref{alg:pbcd}, with bounded $\chi(Z_0)$ and $\xi_t\leq \min\{1/A_Z, \mu_Z/B_Z, 1/\mu_Z\}$, remains bounded with high probability. Furthermore, the sequence $E_t$ generated by Algorithm \ref{alg:pbcd} converges to zero, i.e.,  $\lim_{t\to\infty}E_t=0$.
\end{proposition}
Proof is in Appendix \ref{lemma:d_bounded_p}.
The following convergence result establishes a recursion for the potential function, implying that when $\phi_t\to\phi^*$, the potential function, and consequently $D_t$, converges to zero.

 \begin{theorem}\label{thm:main}
Under Assumption \ref{ass:rkhs_1} and the sample size given in Lemma \ref{le:local_curvature}, if the step sizes in Algorithm \ref{alg:pbcd} are selected such that $\eta\in\mathcal{O}\big(\frac{\alpha}{\alpha+L_{\phi,\alpha}}\big)$ and $\zeta\in\mathcal{O}\big(\min\{\frac{\alpha^3\sqrt{nT}}{(\alpha+L_{\phi,\alpha})^2},1\}\big)
$, then with probability at least $1-\delta$, the results of Proposition \ref{prop:E_converge} holds and for all $t\ge 0$,
\begin{equation*}\label{eq:main}
\mathcal V_t\ \le\ \rho^{\,t}\,\mathcal V_0\ +\ C_\phi(t)\,
\end{equation*}
where $C_\phi(t)\in\mathcal{O}(\sum_{i=0}^{t-1} \rho^{t-i-1}\chi^2(Z_i))$ and $0<\rho<1$.
\end{theorem}
A proof including the constant terms are in Appendix \ref{pr:lem:one-step-re}.

\paragraph{\textbf{Regret Analysis}}
Herein, we define and bound the corresponding regret of Algorithm \ref{alg:pbcd}. Let $\tilde R_{\bar T}:=\tfrac{1}{\bar T}\sum_{t=1}^{\bar T} \big(\tilde{\mathcal{L}}(\phi_t,Z_t)\!-\!\min_{\phi\in \mathcal{H}_{\xi,\Xi},Z\in\mathcal{C}}\tilde{\mathcal{L}}(\phi,Z)\big)$. 
This quantity measures the cumulative regret of Algorithm~\ref{alg:pbcd} relative to the global minimum of the loss function. However, since the loss function is non-convex with respect to $Z$ for any $\phi$, convergence of the algorithm’s iterates to the global minimum cannot be guaranteed. 
In particular, when the initialization is closer to a local minimizer than to the global one, the algorithm may converge to such a local solution. As a result, $\tilde R_T$ grows linearly with $T$. Therefore, we introduce the following modified notion of regret with respect to the global minimum of the loss function when $Z$ is fixed,

\begin{align}
    R_{\bar T}:= \frac{1}{\bar T}\sum_{t=1}^{\bar T} \Bigl(\tilde{\mathcal{L}}(\phi_t,Z_t)\!-\!\!\min_{\phi\in \mathcal{H}_{\xi,\Xi}}\tilde{\mathcal{L}}(\phi,Z_t)\Bigr).
\end{align}
Using this regret definition, we derive a sublinear regret bound for Algorithm~\ref{alg:pbcd} (proof in Appendix~\ref{thm:avg_phi_gap_boundary_p}).
\begin{theorem}
\label{thm:avg_phi_gap_boundary}
Under the step size scheme of Theorem \ref{thm:main} and the sample size given in Lemma \ref{le:local_curvature}, we have $R_{\bar T}\in
\mathcal{O}(1/\sqrt{\bar T})$ with probability at least $1-\delta$.
\end{theorem}

\paragraph{\textbf{Analyzing the Stationary Points.}}
Let $(\phi, Z)$ be a stationary point of the $\tilde{\mathcal{L}}(\phi, Z)$, i.e., the gradients of $\tilde{\mathcal{L}}$ with respect to both $\phi$ and $Z$ vanish at $(\phi, Z)$. The following result shows that, at such a point, the average residuals remain bounded and tend to zero as $\alpha$ decreases. Proof is in Appendix \ref{thm:phi_error_proof}.
\begin{theorem}\label{thm:phi_error}
For a given regularization coefficient $\alpha$, define the residual vector $e_\alpha \in \mathbb{R}^M$ at the stationary point $(\phi, Z)$ by $(e_\alpha)_k := y_k - \phi(\langle A_{k}, ZZ^\top\rangle)$, then
\[
\frac{\|e_\alpha\|_2}{\sqrt{M}}\leq\frac{\alpha\sqrt{B_K}}{\alpha+2\lambda_{\min}(K)}\|\phi^\star\|_{\mathcal{H}}.
\]
where $K$ is the kernel Gram matrix.
\end{theorem}

\subsection{Noisy Observation Setting}
Recall that in this setting, the idiosyncratic components have non-zero but bounded variance.
\begin{assumption}\label{ass:noise_subgaussi}
    We assume that the idiosyncratic components $\{u_k\}_{k=1}^M$ are zero mean independent $\sigma$-sub-Gaussian random variables. 
\end{assumption}

As the results from the noiseless setting carry over to this case with minor adjustments to account for the noise, we provide the auxiliary lemmas in Appendix \ref{app:noise_extension} and present below only the main result for the potential function. 

\begin{theorem}\label{thm:main_noisy}
Under the same assumptions as in Lemma \ref{lem:local-curvature-noisy}, there exist $\rho\in(0,1)$, $C_\phi\ge 0$, and $C_{\sigma}\ge 0$ such that with probability $1-2\delta$, for all $t\ge 0$,  
\[
\mathcal{V}_t
\le \rho^t  \mathcal{V}_0
+ \frac{C_\phi(t)}{1-\rho}+ \frac{C_{\sigma}}{1-\rho},
\]
where $C_\phi(t)\in\mathcal{O}(\sum_{i=0}^{t-1} \rho^{t-i-1}\chi^2(Z_i))$ and $0<\rho<1$.
\end{theorem}
This result implies that when $\alpha$ is small enough, $\lim\limits_{t \to \infty}\!\|\Delta_t\|$ $\!\in\mathcal{O}(\chi(Z_\infty)+\sigma)$, where $Z_\infty$ is the limit point of $Z_t$. See Appendix \ref{app:discussion_alpha} for details.

\section{Experiments}\label{sec:experiments}

In this section, we evaluate the performance of Algorithm~\ref{alg:pbcd} in both synthetic and real-world settings, study the effect of key parameters, and empirically validate our theoretical results. 

\subsection{Synthetic Setting.}
We consider two settings: a known link function and an unknown link function. The former isolates the impact of mismatch in the link function and clarifies the role of the remaining parameters. 

\paragraph{\textbf{Learning the latent factors and loadings with a known link function}}
In this setting, no learning of the link function is required, consequently, Algorithm~\ref{alg:pbcd} does not update $\phi_t$.
\begin{figure}[h]
        \centering
        \includegraphics[width=\linewidth]{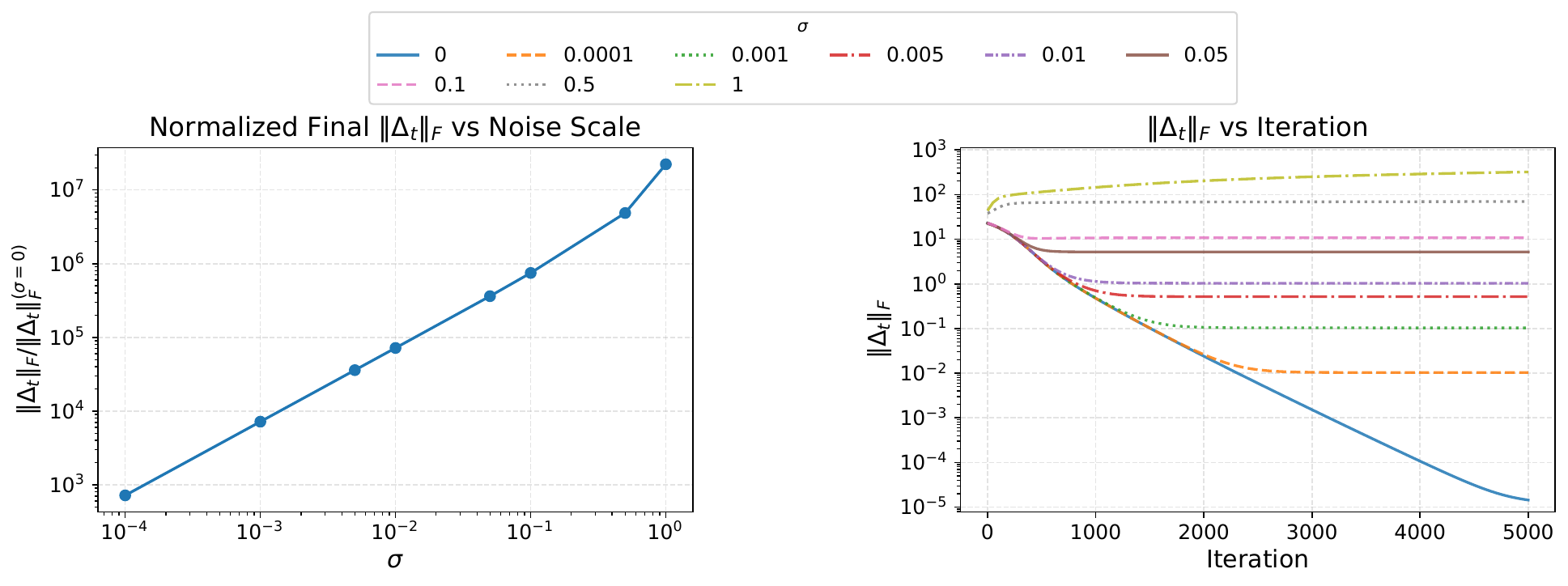}
    \caption{Left: Normalized estimation error versus the noise level after 5000 iterations. Right: Estimation error over iterations under different noise levels. The sigmoid link function is known. 
   } 
    \label{fig:only_z:sigmoid_noise}
\end{figure}

\begin{figure*}
    \centering
    \begin{subfigure}[b]{0.32\linewidth}
        \includegraphics[width=\linewidth,height=0.7\linewidth]{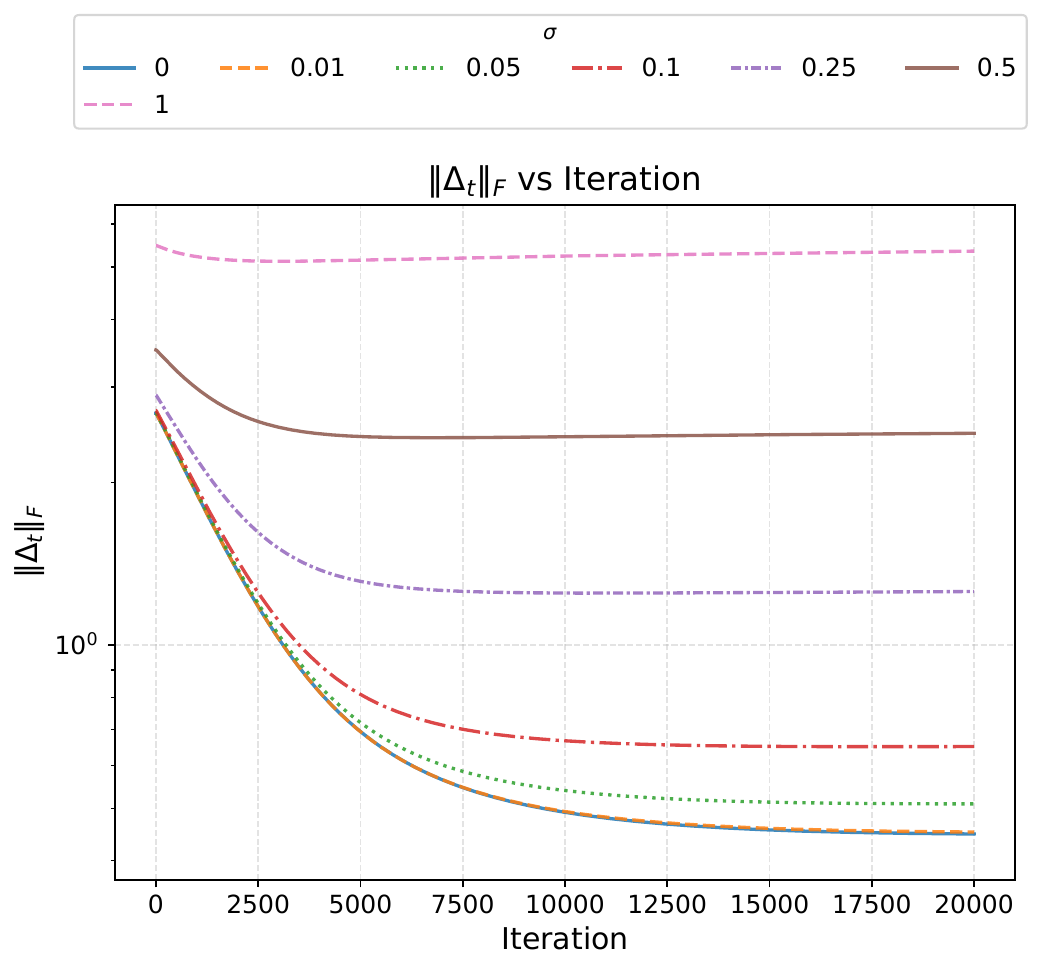}
        \caption{}
    \end{subfigure}
    \hfill
    \begin{subfigure}[b]{0.32\linewidth}
        \includegraphics[width=\linewidth,height=0.7\linewidth]{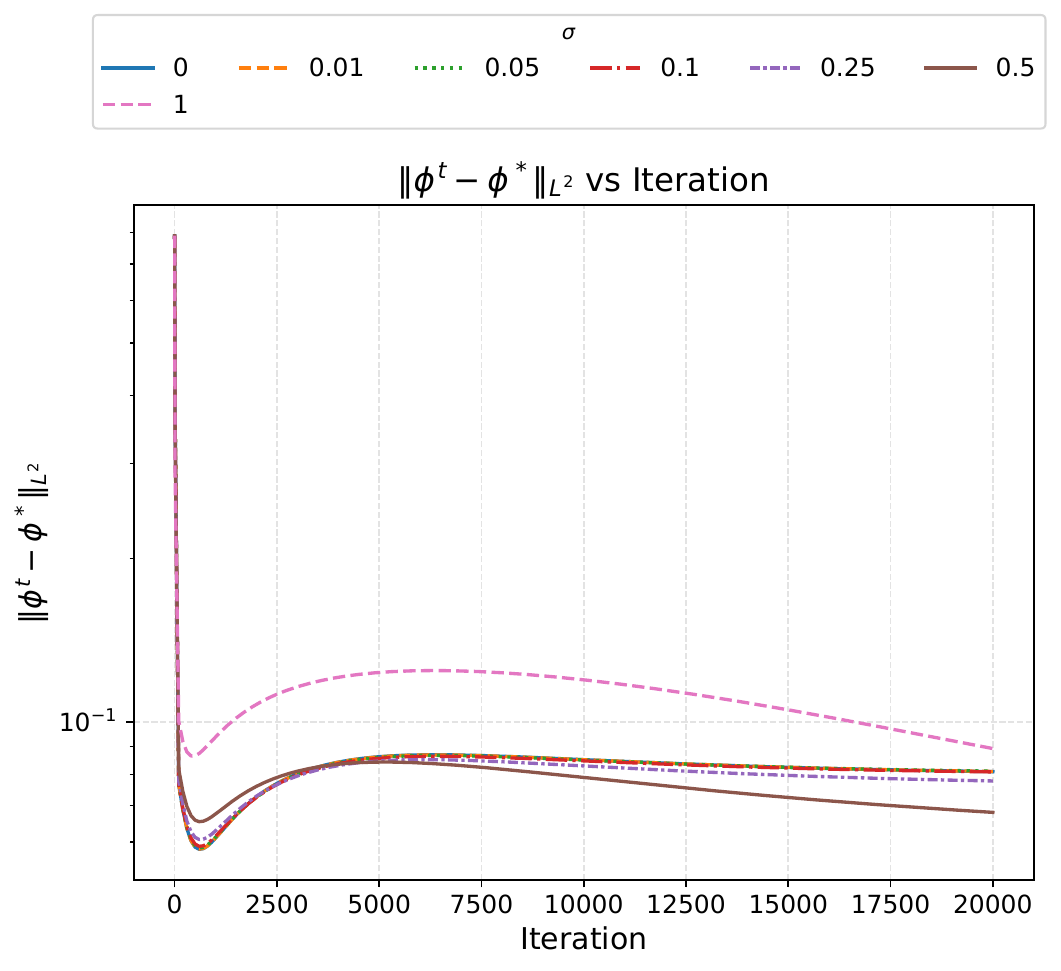}
        \caption{}
    \end{subfigure}
    \hfill
    \begin{subfigure}[b]{0.32\linewidth}
        \includegraphics[width=\linewidth,height=0.7\linewidth]{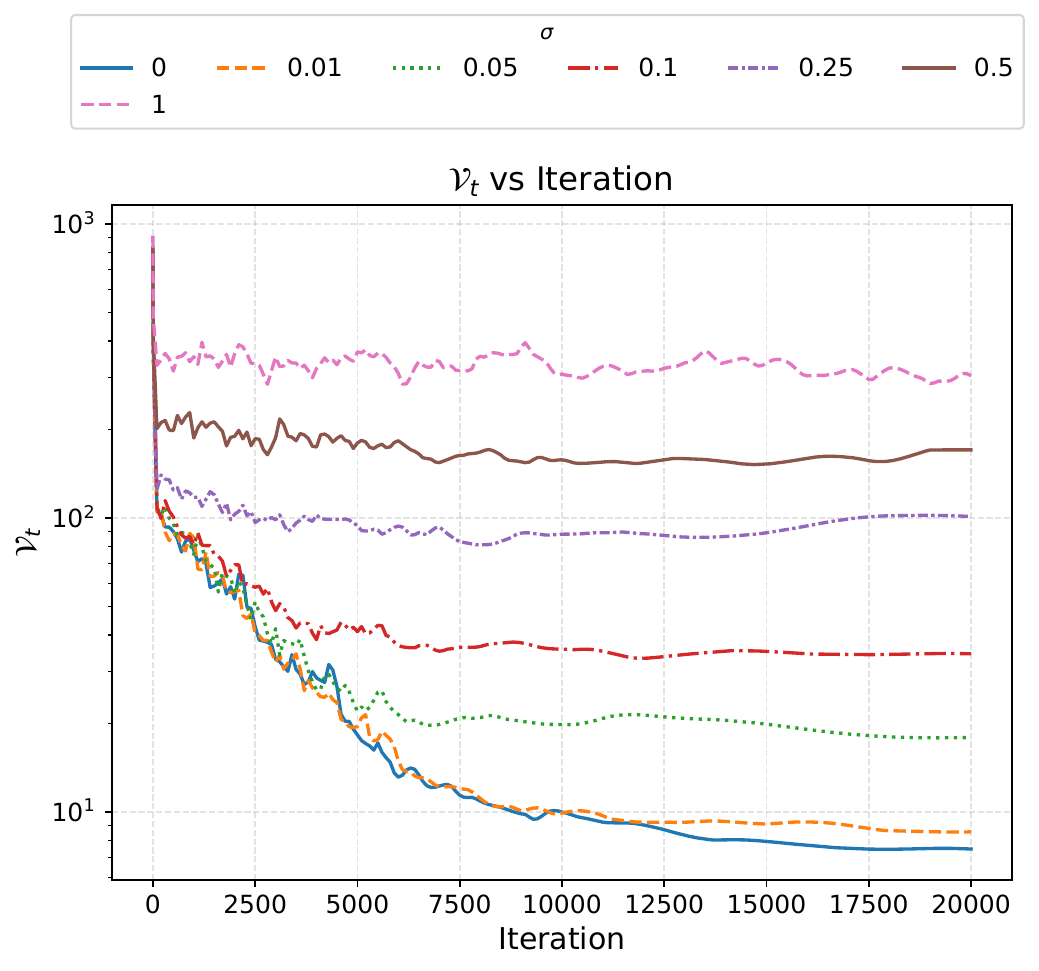}
        \caption{}
    \end{subfigure}
    
    \caption{Performance of Algorithm~\ref{alg:pbcd} under varying noise levels: (a) latent factors estimation error $\|\Delta_t\|_F$, (b) link function estimation error, and (c) Lyapunov function convergence.}\label{fig:both:sigma_delta}
    
\end{figure*}

\textit{Effect of the noise level:}
We applied the algorithm only on the $Z$-block under varying noise levels, where the idiosyncratic terms are $u_{i,t}\stackrel{\text{i.i.d.}}{\sim} \mathcal{N}(0, \sigma^2)$.
For different $\sigma$, Figure \ref{fig:only_z:sigmoid_noise}(right) depicts the convergence behavior for varying $\sigma$. In this experiment, the link function is known and given by the sigmoid function $\sigma(x):=1/(1+\exp(-x))$, with $M=500000$ samples.
Moreover, to illustrate the dependency between the convergence point and the noise level, the algorithm is run for a fixed number of iterations, and the resulting normalized estimation error $\|\Delta_t\|_F/(\|\Delta_t\|_F|_{\sigma=0})$, where the denominator is the final error for  noiseless observation ($\sigma=0$), is shown in Figure \ref{fig:only_z:sigmoid_noise}(left).  
This is consistent with the prediction of Theorem~\ref{thm:main_noisy}, which implies a linear dependence of the latent factor estimation error on the noise variance when the link function is known. 
These results validate the tightness of the theoretical bound in the known link function setting.

\textit{Effect of the sample size:} 
To investigate this effect, we set  $n\!=\!T\!=\!1000$, $r\!=\!10$ and assume the link function is known and equal to the identity, thereby isolating any mismatch in the link.  Figure \ref{fig:identity:sample_size} shows the convergence of $\|\Delta_t\|_F$ for varying sample sizes $M\in\{1,2,..,8\}\times rn$. This results shows that as $M\geq 3nr$, the iterates converges.

\begin{figure}[H]
        \centering
        \includegraphics[width=.7\linewidth,height=.45\linewidth]{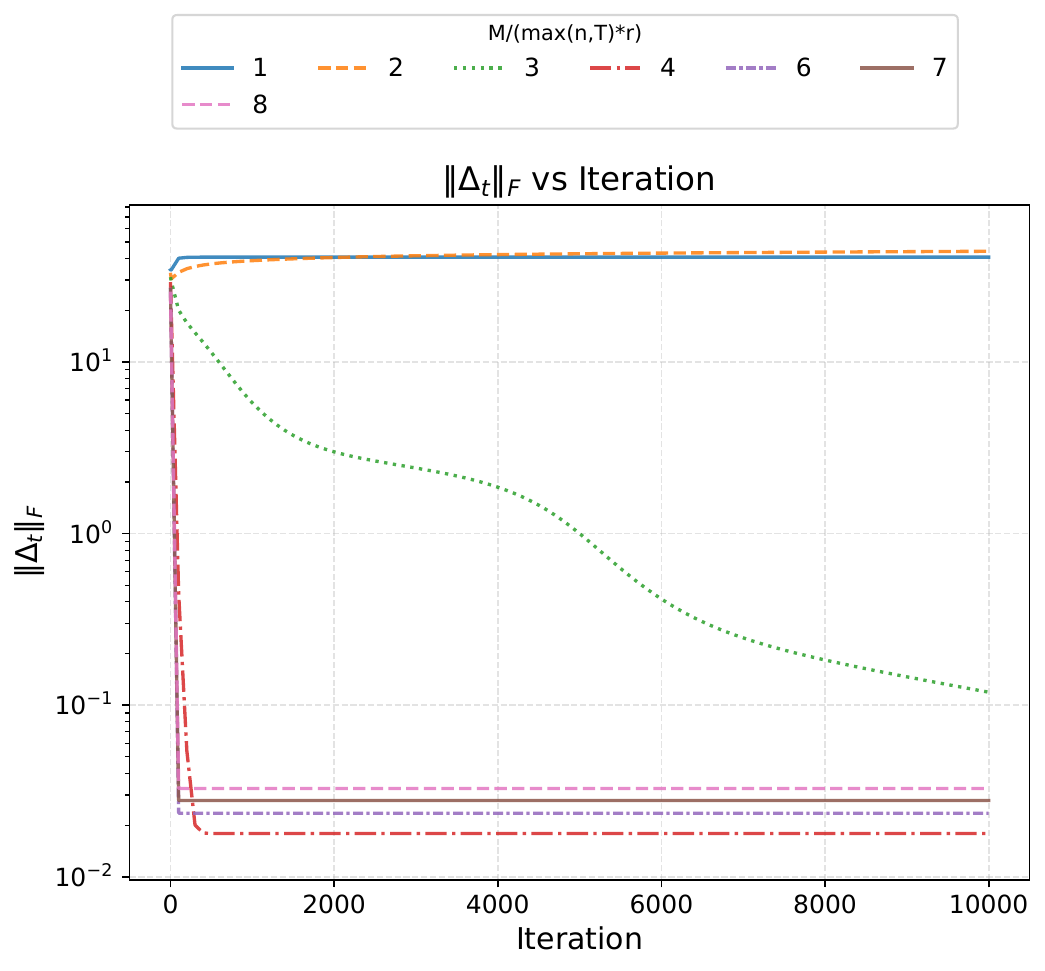}
    \caption{Estimation error $\|\Delta_t\|_F$ versus the number of iterations for different sample sizes $M$.}
    \label{fig:identity:sample_size}
\end{figure}

\paragraph{\textbf{Joint learning of latent factors, loadings, and the link function}}
Herein, we considered the full problem and applied Algorithm~\ref{alg:pbcd}.
We examined the convergence behavior of the factors estimation error $\|\Delta_t\|_F$, the link function estimation error $\|\phi_t-\phi^\star\|_{\mathcal H}$, and the Lyapunov $\mathcal V_t$ across different noise levels $\sigma$ and regularization parameters $\alpha$. In these experiments, the shared parameters are $n=T=100,r=3, M=5000,\lambda=0.5,\zeta=10^{-5},\eta=10^{-4}$.

\textit{Effect of the noise level:}
Figure \ref{fig:both:sigma_delta} illustrates the estimation errors and potential function over iterations for various noise levels with $\alpha=10^{-3}$.
Our results show that, for a fixed $\alpha$, increasing the noise level $\sigma$ leads to a degradation in both $\|\Delta_t\|$ and $\|\phi_t-\phi^\star\|_{\mathcal H}$, while the Lyapunov potential $\mathcal V_t$ converges to a neighborhood whose radius grows with $\sigma$.
In the noiseless case ($\sigma=0$), the limiting value of $\mathcal V_t$ provides an empirical estimate of the constant $C_\phi(t)$ appearing in Theorem \ref{thm:main}. For additional plots see Appendix \ref{app:additional_exp}

\begin{figure}[H]
        \centering
        \includegraphics[width=.8\linewidth,height=.4\linewidth]{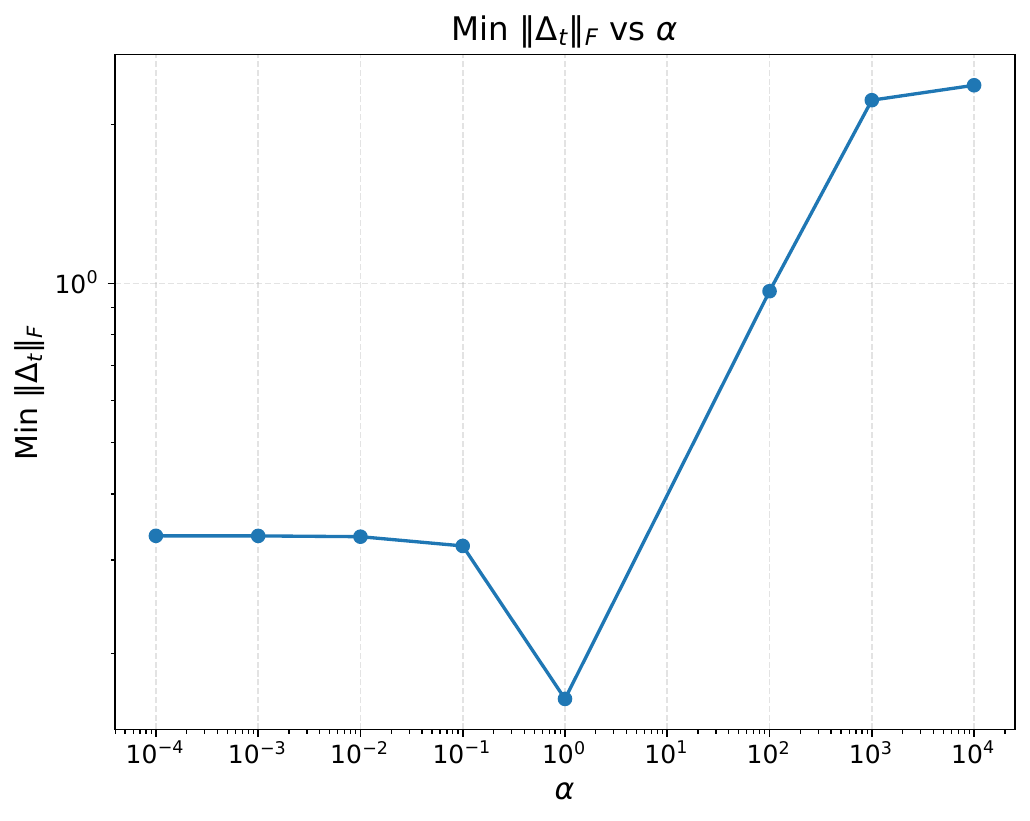}
    \caption{The final estimation error $\|\Delta_t\|_F$ under different $\alpha$.}
    \label{fig:sigmoid:alpha}
\end{figure}

\textit{Effect of regularization coefficient $\alpha$:}
To further investigate the role of the $\alpha$, we fix the noise level at $\sigma=0.1$ and vary $\alpha$. Figure \ref{fig:sigmoid:alpha} shows the final estimation error of Algorithm~\ref{alg:pbcd} as function of $\alpha$. The result aligns with our discussion in Appendix \ref{app:discussion_alpha}
and highlights the practical impact of the regularization parameter in nonlinear factor recovery.

\subsection{Empirical Settings.} 
We consider two real-world datasets in this section to demonstrate the applicability and scalability of our algorithm. Moreover, our experimental results show that the projection steps in the algorithm do not adversely affect its scalability.

\textit{MovieLens}\footnote{\href{https://grouplens.org/datasets/movielens/}{MovieLens dataset link}.} uses a $610\times9724$ user-movie matrix with 5-star ratings, yielding a total of 5,931,640 entries. Among these, 100,836 ratings are observed. We use a row-stratified 10\% validation split, with 91,104 ratings for training and the remainder for validation. $\beta$ is used for projecting the $Z$ matrix. We report results at the algorithms' best checkpoints in Table \ref{table_movielens}; for example, for Id. Link, this occurs at iteration 1000. The Val RMSE and Train RMSE correspond to the validation and training Root Mean Squared Errors (RMSE), respectively.

\begin{table}[h]
\centering
\small
\setlength{\tabcolsep}{4pt}
\begin{tabular}{l c c c c c c c}
\hline
Method & Mono. $\phi$ & Proj. $Z$ & $\beta$ & Iter. & Time(s) & Val RMSE & Train RMSE\\
\hline
Id. Link          & N & N & -     & 1000 & 2.46  & 0.909 & 0.723 \\
Id. Z proj.       & N & Y & 2.441 & 1000 & 2.77  & 0.910 & 0.766 \\
NL Z proj.        & N & Y & 0.496 & 1500 & 61.04 & 0.847 & 0.710 \\
NL monotone       & Y & Y & 0.496 & 1525 & 61.06 & 0.847 & 0.718 \\
\hline
\end{tabular}
\caption{Results of different learning methods on the MovieLens dataset.}
\label{table_movielens}
\end{table}

Id. Link and NL. Link denote the linear factor model and our model, respectively. Z proj. method performs only projection of matrix $Z$ while NL monotone performs both projections, as also indicated in the second and third columns. 
Clearly, the monotone method (which performs both projection) and Z proj. method (which only does $Z$ projection) have similar runtimes. This further supports our analysis that projecting the link function does not impact scalability.
An important observation is the advantage of link function learning that decreases the validation RMSE compared to the identity link function. 

Table \ref{table_movielens_2} shows the range of the inferred factor-loading products, i.e., the matrix $X$. For the linear model (Id. Link), the range is roughly 0 to 5, corresponding to the scale of the ratings. In contrast, for the nonlinear model, the range is much smaller, approximately [-0.2, 0.2], and the link function maps this range to the 5-star ratings through a nonlinear transformation. The inferred link functions are presented in Figure \ref{fig:movielens}.

\begin{table}[h]
\centering
\small
\setlength{\tabcolsep}{6pt}
\begin{tabular}{lcc}
\hline
Method & $\min(X)$ & $\max(X)$ \\
\hline
Id. Link         & -0.83 & 6.06 \\
Id. Z proj.      &  0.18 & 5.45 \\
NL Z proj.       & -0.23 & 0.24 \\
NL monotone      & -0.23 & 0.24 \\
\hline
\end{tabular}
\caption{Range of $X$, i.e., the input of the link function for various type of link functions.}
\label{table_movielens_2}
\end{table}

\begin{figure}
    \centering
    \includegraphics[width=1\linewidth]{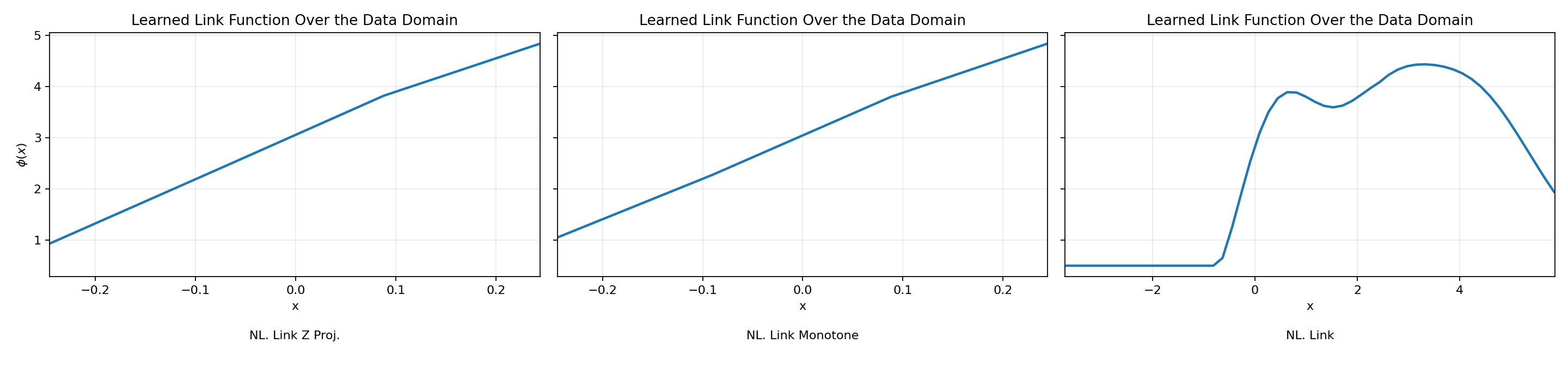}
    \caption{Inferred link functions for different methods in Table \ref{table_movielens_2}.}
    \label{fig:movielens}
\end{figure}

\textit{Jester}\footnote{\href{https://www.kaggle.com/datasets/vikashrajluhaniwal/jester-17m-jokes-ratings-dataset}{Jester dataset link}.} dataset dataset has 24,983 users and 100 jokes, with continuous ratings from -10 to 10. After cleaning the data, the matrix contains 2,498,300 possible user-joke entries, of which 1,810,455 are observed. We use a 10\% hold-out for validation, with the remaining entries used for training. Table \ref{table_jester} presents the results. 

\begin{table}[h]
\centering
\small
\setlength{\tabcolsep}{5pt}
\begin{tabular}{l c c c c c c c}
\hline
Method & M.$\phi$ & P.$Z$ & $\beta$ & Iter & Time & Val & Train \\
\hline
Id. Link      & N & N & -     & 1000 & 79.84   & 4.192 & 3.997 \\
NL monotone   & Y & Y & 3.537 & 2000 & 1385.82 & 4.190 & 3.903 \\
\hline
\end{tabular}
\caption{Results of different learning methods on the Jester dataset.}
\label{table_jester}
\end{table}

\section{Conclusion}\label{sec:conclusion}
We studied the problem of jointly learning the latent factors, loading factors, and the link function under several observation settings, including complete, random, noisy, and noiseless. 
Our proposed learning method is a projected block coordinate descent. By assuming a nonparametric link function that lies in an RKHS with smooth and bounded kernels, and imposing mild incoherence conditions on the factors and loadings, we provided a theoretical analysis of the proposed algorithm. The analysis in this work relies on several assumptions that may be relaxed, and thus represent limitations of the present study. Moreover, we conjecture that the projection steps in our algorithm are unnecessary.

\nocite{langley00}

\bibliography{example_paper}

\appendix
\begin{center}
{\Large\textbf{Appendix}}    
\end{center}

\section{Helper Lemmas}

\begin{lemma}
For matrix $A$ with rank $r$, we have
    \begin{align*}
    \|AA^\top \|_F^2=\sum_{r}\sigma_r^4\leq(\sum_r\sigma_r^2)^2=\|A\|_F^4.    
    \end{align*}
\end{lemma}

\begin{lemma}
    For any matrices $A,B,C$, we have
    \begin{align*}
       &\|ABC\|_F\leq\|A\|_2\|B\|_2\|C\|_F,\\
    &\|AB\|_F\leq \|A\|_2\|B\|_F\leq \|A\|_F\|B\|_F
    \end{align*}
\end{lemma}

\begin{definition}
We define the following empirical operator:
    \begin{align}\label{eq:D_operator}
    \mathcal{D}(Y) := \frac{1}{M}\sum_{k = 1}^M \langle A_k, Y \rangle^2, \quad
    \mathcal{G}(Y) := \frac{1}{M}\sum_{k = 1}^M |\langle A_k, Y \rangle|. 
\end{align}
\end{definition}

\begin{lemma}\label{lemma:usi2}
For matrices $A$ and $B$, we have
    \begin{align*}
        \sqrt{\mathcal{D}(AB^\top )/M}\leq\mathcal{G}(AB^\top )\leq \sqrt{\mathcal{D}(AB^\top )}
    \end{align*}
\end{lemma}

\begin{lemma}\label{lemma:usi3}
In the complete observation setting, we have

\begin{itemize}
\item   for any matrices $W,Z$,
\begin{align*}
\mathcal{D}(WZ^\top )&=\frac{1}{M}\sum_{i,j}\langle W_i, Z_j\rangle^2\leq \frac{1}{M}(\sum_{i\in[n]}\|W_i\|_2^2)(\sum_{j\in[T]}\|Z_j\|^2_2 )\leq  \frac{1}{M}\|w_b\|_\infty \|z_f\|_1.
\end{align*}
where $w:=[\|W_1\|^2_2,\|W_2\|^2_2,...]$ and $z:=[\|Z_1\|^2_2,\|Z_2\|^2_2,...]$ and $w_b$ are the indices corresponding to the loading factors $i\in[n]$ and $z_f$ are the ones corresponding to the factors $j\in[T]$. Thus,
\begin{align}
\mathcal{D}(WZ^\top +ZW^\top )&\leq \frac{2}{M}\|W\|^2_{2,\infty}\|Z\|^2_F.
\end{align}

\item Moreover, we have
\begin{align*}
\mathcal{D}(WW^\top )&\leq   \frac{1}{M}w_b\Big(\sum_{i,j}e_{i}\tilde{e}_j\Big)w^\top_f \leq \frac{1}{M}\|W\|_F^4.
\end{align*} 

\end{itemize}
\end{lemma}

\begin{lemma}\label{lemma:usi}
    We have
\begin{align*}
&\|Z^*\|_F^2\leq 2r\sigma_1^*,\\
&\|\Phi\|_{2,\infty}^2\leq \frac{2r\sigma^*_1\mu}{n+T},\\
&|\langle A, ZZ^\top  \rangle|=\sum_{i,j} a_{i,j}Z_i^\top Z_j\leq \|{Z}\|_{2, \infty}^2, \quad \|Z\|_{2,\infty}^2\leq \frac{3.5\mu \|Z^*\|_F^2}{n+T},\\
& \|\Delta\|_{2,\infty}^2\leq \frac{10\mu}{n+T}\|Z^*\|_F^2\leq \frac{20\mu r \sigma_1^*}{n+T}.
\end{align*}
\end{lemma}
\begin{proof}
    Recall the definition of $Z^*$. It is clear that $\sqrt{2\sigma_1^*},...,\sqrt{2\sigma_r^*}$ are the singular-values of $Z^*$, thus 
    $\|Z^*\|_F^2=2\sum_{i=1}^r\sigma_i^*\leq 2r\sigma_1^*$. 

    From the solution of the equivalent solution set, we have $\|Z^*\|_{2,\infty}^2=\|Z^*R(Z)\|_{2,\infty}^2=\|\Phi\|_{2,\infty}^2$ and thus
    $$
    \|\Phi\|_{2,\infty}^2\leq\frac{\mu}{n+T}\|Z^*\|_F^2\leq \frac{2r\sigma^*_1\mu}{n+T}.
    $$
    \begin{align*}
        \|\Delta\|_{2,\infty}^2\leq 2(\|Z\|_{2,\infty}^2+\|\Phi\|_{2,\infty}^2).
    \end{align*}
    According to the initialization assumption, we have $\|\Delta(Z_0)\|^2\leq \epsilon\sigma_r^*$ and thus for small enough $\epsilon$, we have
    $$
    \|Z_0\|_F=\|\Delta(Z_0)+\Phi(Z_0)\|_F\leq \|\Delta(Z_0)\|_F+\|Z^*\|_F\leq \sqrt{\epsilon\sigma_r^*}+\|Z^*\|_F\leq \sqrt{2}\|Z^*\|_F.
    $$
    On the other hand, since $Z\in\mathcal{C}$, we get 
    $$
    \|Z\|_{2,\infty}\leq \frac{4}{3}\sqrt{\frac{\mu}{n+T}}\|Z_0\|_F\leq \frac{4}{3}\sqrt{\frac{2\mu}{n+T}}\|Z^*\|_F.
    $$
    Putting back the above inequalities, we conclude the result.  
    
\end{proof}

\subsection{Concentration Results:}
Herein, we present the useful lemmas for the random observation setting. 
\begin{lemma}[Chernoff bound]\label{lem:chernoff_bound}
    Suppose $x_1, x_2, \ldots, x_m$ are i.i.d. Bernoulli random variables with parameter $p$ and let $\epsilon > 0$ be given. Then:
    \begin{align*}
        P\left(\frac{1}{m}\sum_{k = 1}^m x_k \geq p + \epsilon \right) \leq \exp\left(-\frac{m\epsilon^2}{2p(1-p)}\right)
    \end{align*}
\end{lemma}

\begin{lemma}[Matrix Bernstein Inequality]\label{lem:matrix_bernstein}
    Consider a random matrix $X$ of shape $n_1 \times n_2$ that satisfies:
    \begin{align*}
        \mathbb{E}[X] = \bar{X} \quad \text{ and } \quad \|{X}\|_2 \leq L \text{ almost surely}.
    \end{align*}
    Let $b$ be an upper bound on the second moment of $X$:
    \begin{align*}
        \|{\mathbb{E}[XX^\top ]}\|_2 \leq b \quad \text{ and } \quad \|{\mathbb{E}[X^\top X]}\|_2 \leq b.
    \end{align*}
    Let $X_{\mathcal{D}} = \frac{1}{m}\sum_{k = 1}^m X_k$, where each $X_k$ is an i.i.d. copy of $X$. Then, for all $t \geq 0$,
    \begin{align*}
        P( \|{X_{\mathcal{D}} - \bar{X}}\|_2 \geq t) &\leq (n_1 + n_2) \exp\left(\frac{-mt^2/2}{b + 2Lt/3}\right)
    \end{align*}
\end{lemma}

\begin{lemma}[Lemma 10 \cite{zheng2016convergence}]\label{lem:con1}
Let 
\begin{align*}
    \mathcal{T}:=\{H\in\mathbb R^{n\times T}: H=U^* X^\top +Y(V^*)^\top \ \text{for some}\ X, Y\}.
\end{align*}
When sampling probability $p=M/(nT)$ is at least $\frac{c}{\delta^2}\frac{\mu r\log(n+T)}{\min\{n,T\}}$, with probability at least $1-3(n+T)^{-3}$, uniformly for all $Y\in \mathcal{T}\subset\mathbb R^{n\times T}$, we have
\begin{align*}
    &p(1-\delta)\|Y\|_F^2\leq \sum_{(i,j)\in \Omega}y_{i,j}^2 \leq p(1+\delta)\|Y\|_F^2,\\
    & \Big|p^{-1}\sum_{i,j\in\Omega}a_i^\top b_j-\sum_{i,j}a_i^\top b_j\Big|\leq \delta\|A\|_F\|B\|_F.
\end{align*}
\end{lemma}

\begin{lemma}\label{lemma:con1}
    For $M\in\mathcal{O}(\frac{(\mu r\kappa)^2\min\{n,T\}}{\omega^2}\log(\frac{n+T}{\delta}))$, we have with probability at least $1-\delta$,
\begin{align}\label{eq:con_2}
\mathcal{D}(WW^\top )&\leq \Big(\frac{\|W\|^2_F}{4nT} +\frac{1}{2}\frac{\omega\|W\|^2_{2,\infty}}{\mu r \kappa\min\{n,T\}}\Big)\|W\|^2_F.
\end{align} 
and for $M\in\mathcal{O}(\min\{n,T\}\log(\frac{n+T}{\delta}))$, we have with probability at least $1-\delta$,
\begin{align}\label{eq:con_1}
        \mathcal{D}(WZ^\top +ZW^\top )\leq \frac{4 }{\min\{n,T\}} \|W\|^2_{2,\infty}\|Z\|^2_F.
\end{align}   
\end{lemma}
\begin{proof}
    For any matrices $W,Z$, we get
    \begin{align*}
        \mathcal{D}(WZ^\top )&=\frac{1}{M}\sum_{i,j}\langle W_i, Z_j\rangle^2\leq \frac{1}{M}\sum_{i,j}\|W_i\|_2^2\|Z_j\|^2_2 = w_b^\top \Big(\frac{1}{M}\sum_{i,j}e_{i}\tilde{e}_j^\top \Big)z_f\\
        &\leq  \|w_b\|_\infty\Big\|\frac{1}{M}\sum_{i,j}e_{i}\tilde{e}_j^\top \Big\|_1\|z_f\|_1
    \end{align*}
    where $w=[\|W_i\|^2_2]_i$ and $z=[\|Z_j\|^2_2]_j$ and $w_b$ are the indices corresponding to the loading factors $i\in[n]$ and $z_f$ are the ones corresponding to the factors $j\in[T]$. Hence, $\|w\|_1=\|W\|^2_F$.
Note that $\|A\|_1=\max_j \sum_i | a_{i,j}|$. 
Using the Chernoff bound \ref{lem:chernoff_bound} we can write 
\begin{align*}
    P(\|\frac{1}{M}\sum_{i}e_{i}\tilde{e}_1^\top \|_1\geq\frac{2}{n})\leq e^{-\frac{M}{2n}}\implies  P(\max_j\|\frac{1}{M}\sum_{i}e_{i}\tilde{e}_j^\top \|_1\geq\frac{2}{n})\leq Te^{-\frac{M}{2n}}\leq\frac{\delta}{2}
\end{align*}
Thus, for $M\geq 2(n+T)\log(2(n+T)/\delta)$, we have
$$
\Big\|\frac{1}{M}\sum_{i,j}e_{i}\tilde{e}_j^\top \Big\|_1\leq \frac{2}{n}\leq\frac{2}{\min\{n,T\}}, \quad \Big\|\frac{1}{M}\sum_{i,j}(e_{i}\tilde{e}_j^\top )^\top \Big\|_1\leq \frac{2}{T}\leq \frac{2}{\min\{n,T\}}
$$
with probability at least $1-\delta$.
Overall, we have
\begin{align}\notag
        \mathcal{D}(WZ^\top +ZW^\top )&\leq 2\|w_b\|_\infty\max\Big\{\Big\|\frac{1}{M}\sum_{i,j}e_{i}\tilde{e}_j^\top \Big\|_1, \Big\|\frac{1}{M}\sum_{i,j}(e_{i}\tilde{e}_j^\top )^\top \Big\|_1\Big\}\|z_f\|_1\\ \notag
        &\leq \frac{4 }{\min\{n,T\}} \|W\|^2_{2,\infty}\|Z\|^2_F.
\end{align}    
    On the other hand, we have
\begin{align*}
    \mathcal{D}(WW^\top )&\leq   w_b^\top \Big(\frac{1}{M}\sum_{i,j}e_{i}\tilde{e}_j\Big)w_f =  w_b^\top Aw_f + w_b^\top \Big(\frac{1}{M}\sum_{i,j}e_{i}\tilde{e}_j -A\Big)w_f \\
    &= w_b^\top \frac{11^\top }{nT}w_f + w_b^\top \Big(\frac{1}{M}\sum_{i,j}e_{i}\tilde{e}_j -A\Big)w_f\\
    &\leq w_b^\top \frac{11^\top }{nT}w_f + \|w_b\|_2\Big\|\frac{1}{M}\sum_{i,j}e_{i}\tilde{e}_j -A\Big\|_2\|w_f\|_2.
\end{align*}

where $A:=\mathbb E[\frac{1}{M}\sum_{i,j}e_{i}\tilde{e}_j^\top ]=\frac{11^\top }{nT}$. 
Note that $\|w_b\|_1\|w_f\|_1\leq\|w\|_1^2/4$ and $\|w_b\|_2\|w_f\|_2\leq\|w\|_2^2/2$. Thus, we get
\begin{align*}
\mathcal{D}(WW^\top )&\leq \frac{\|w\|^2_1}{4nT} +\Big\|\frac{1}{M}\sum_{i,j}e_{i}\tilde{e}_j -A\Big\|_2\frac{\|w\|^2_2}{2}\leq \Big(\frac{\|w\|_1}{4nT} +\Big\|\frac{1}{M}\sum_{i,j}e_{i}\tilde{e}_j -A\Big\|_2\frac{\|w\|_\infty}{2}\Big)\|w\|_1\\
&=\Big(\frac{\|W\|^2_F}{4nT} +\Big\|\frac{1}{M}\sum_{i,j}e_{i}\tilde{e}_j -A\Big\|_2\frac{\|W\|^2_{2,\infty}}{2}\Big)\|W\|^2_F.
\end{align*}
From the matrix Bernstein inequality \ref{lem:matrix_bernstein}, we have
\begin{align*}
    P\Big(\Big\|\frac{1}{M}\sum_{i,j}e_{i}\tilde{e}_j -A\Big\|_2\geq a \Big)\leq (n+ T)\exp\big({-\frac{Ma^2}{\frac{2}{\min\{n,T\}}+\frac{4a}{3}}}\big).
\end{align*}
By setting $a=\omega/\min\{n,T\}$ and $M\in\mathcal{O}(\frac{(\mu r\kappa)^2\min\{n,T\}}{\omega^2}\log(\frac{n+T}{\delta}))$, we get with probability at least $1-\delta$,
\begin{align*}
\mathcal{D}(WW^\top )&\leq \Big(\frac{\|W\|^2_F}{4nT} +\omega\frac{\|W\|^2_{2,\infty}}{2\mu r\kappa\min\{n,T\}}\Big)\|W\|^2_F.
\end{align*}

\end{proof}


\subsection{Proof of Lemma \ref{lem:smoothness_phi}:}\label{pr:lem:smoothness_phi}
\begin{proof}
    Given the equations for the gradient of $\tilde{\mathcal{L}}$ with respect to $\phi$ and defining $z_k:=\langle A_{k}, ZZ^\top\rangle $, for any two functions in $\mathcal{H}$, we have
    \begin{align*}
    &\|\nabla_{\phi} \tilde{\mathcal{L}}(Z,\phi_1)-\nabla_{\phi} \tilde{\mathcal{L}}(Z,\phi_2)\|_\mathcal{H}^2\\
    &\leq \frac{8}{M^2}\|\sum_{k=1}^M
    \big(\phi_1(z_k)-\phi_2(z_k )\big)K\big(z_k,\cdot\big)\|_\mathcal{H}^2+2\alpha^2\|\phi_1-\phi_2\|_\mathcal{H}^2\\
    &=\frac{8}{M^2}\sum_{k=1}^M\sum_{l=1}^M\big(\phi_1(z_k)-\phi_2(z_k )\big)K\big(z_k,z_l\big)\big(\phi_1(z_l)-\phi_2(z_l)\big)+2\alpha^2\|\phi_1-\phi_2\|_\mathcal{H}^2\\
    &\leq \frac{8B_K}{M^2}\sum_{k=1}^M\sum_{l=1}^M |\phi_1(z_k )-\phi_2(z_k)|\cdot|\phi_1(z_l )-\phi_2(z_l)|+2\alpha^2\|\phi_1-\phi_2\|_\mathcal{H}^2\\
    &=\frac{8B_K}{M^2}\big(\sum_{k=1}^M|\phi_1(z_k)-\phi_2(z_k)|\big)^2+2\alpha^2\|\phi_1-\phi_2\|_\mathcal{H}^2\\
    &\leq \frac{8B_K}{M^2}(\sum_{k=1}^M\sqrt{K(z_k,z_k)}\|\phi_1-\phi_2\|_\mathcal{H})^2+2\alpha^2\|\phi_1-\phi_2\|_\mathcal{H}^2\leq (8B_K^2+2\alpha^2)\|\phi_1-\phi_2\|_\mathcal{H}^2.
    \end{align*}
    In the above, we used the fact that $|\phi_1(z_k)-\phi_2(z_k)|\leq\sqrt{K(z_k,z_k)}\|\phi_1-\phi_2\|_\mathcal{H}$. This can be seen by the Cauchy–Schwarz inequality in $\mathcal H$.
\end{proof}

\subsection{Proof of Lemma \ref{lem:smoothness_z_to_phi}}\label{lem:smoothness_z_to_phi_proof}

\begin{proof}
Recall that the gradient of $\tilde{\mathcal{L}}$ with respect to $\phi$ is given by
\[
\nabla_\phi \tilde{\mathcal{L}}(\phi,Z)(\cdot)
=
-\frac{2}{M}\sum_{k=1}^M
\bigl(y_k-\phi(z_k(Z))\bigr)\,K(z_k(Z),\cdot)
+\alpha\,\phi(\cdot),
\]
where $z_k(Z):=\langle A_k,ZZ^\top\rangle$.
For fixed $\phi$,
\begin{align*}
&\nabla_\phi \tilde{\mathcal{L}}(\phi,Z_1)-\nabla_\phi \tilde{\mathcal{L}}(\phi,Z_2) \\
&\quad=
-\frac{2}{M}\sum_{k=1}^M
\Big[
(y_k-\phi(x_{1k}))K(x_{1k},\cdot)
-
(y_k-\phi(x_{2k}))K(x_{2k},\cdot)
\Big],
\end{align*}
where $x_{ik}:=x_k(Z_i)$.
For each $k$, we decompose the difference as
\begin{align*}
&(y_k-\phi(x_{1k}))K(x_{1k},\cdot)
-
(y_k-\phi(x_{2k}))K(x_{2k},\cdot) \\
&\quad=
\bigl(\phi(x_{2k})-\phi(x_{1k})\bigr)K(x_{1k},\cdot)
+
\bigl(y_k-\phi(x_{2k})\bigr)
\bigl(K(x_{1k},\cdot)-K(x_{2k},\cdot)\bigr).
\end{align*}

We bound the two terms separately. Using $|\phi'(x)|\le \Xi$ and
$\|K(x,\cdot)\|_{\mathcal H}^2=K(x,x)\le B_K$, we obtain
\[
\bigl\|
(\phi(x_{2k})-\phi(x_{1k}))K(x_{1k},\cdot)
\bigr\|_{\mathcal H}
\le
\Xi\sqrt{B_K}\,|x_{1k}-x_{2k}|.
\]
To bound the second term, recall the objective function:
\[
\tilde{\mathcal{L}}(\phi,Z)
=
\frac{1}{M}\sum_{k=1}^M
\bigl(y_k-\phi(\langle A_k,ZZ^\top\rangle)\bigr)^2
+ \frac{\lambda}{4}R(Z)
+ \frac{\alpha}{2}\|\phi\|_{\mathcal H}^2
\]
implies
\[
\frac{\alpha}{2}\|\phi\|_{\mathcal H}^2 \le \tilde{\mathcal{L}}(\phi,Z),
\qquad
\text{and hence}
\qquad
\|\phi\|_{\mathcal H}
\le
\sqrt{\frac{2}{\alpha}\,\tilde{\mathcal{L}}(\phi,Z)}\leq\sqrt{\frac{2}{\alpha}\,\tilde{\mathcal{L}}(\phi_0,Z_0)},
\]
where the last inequality is due to the monotone decreasing of $\tilde{\mathcal L}$ by applying Algorithm 1. 
On the other hand, for any $x\in\mathbb R$, the bounded derivative assumption $|\phi'(x)|\le \Xi$ implies
\[
|\phi(x)| \le |\phi(0)| + \Xi |x|.
\]
Hence, for any $k\in[M]$ and $Z\in\mathcal C$,
\[
|y_k-\phi(\langle A_k,ZZ^\top\rangle)|
\le
|y_k| + |\phi(0)| + \Xi |\langle A_k,ZZ^\top\rangle|\leq \max_k |y_k|+|\phi(0)|+\Xi\max_{Z\in\mathcal{C}}\|Z\|_F^2.
\]
Note that for $Z\in \mathcal{C}$, $\|Z\|_F^2\leq (n+T)\|Z\|^2_{2,\infty}\leq16\mu/9\|Z_0\|_F^2$, thus the above maximum is bounded. By the reproducing property of the RKHS,
\begin{align*}
|\phi(0)|
=
|\langle \phi, K(0,\cdot)\rangle_{\mathcal H}|
\le
\|\phi\|_{\mathcal H}\,\|K(0,\cdot)\|_{\mathcal H}
=
\|\phi\|_{\mathcal H}\sqrt{K(0,0)}
\\
\le
\sqrt{B_K}\,\|\phi\|_{\mathcal H}\leq\sqrt{\frac{2B_K}{\alpha}\,\tilde{\mathcal{L}}(\phi_0,Z_0)}.
\end{align*}

Hence, 
\begin{equation}\label{def:G}
|y_k-\phi(\langle A_k,ZZ^\top\rangle)|
\le \max_k |y_k|+\sqrt{\frac{2B_K}{\alpha}\,\tilde{\mathcal{L}}(\phi_0,Z_0)}+\Xi\max_{z\in\mathcal{C}}\|Z\|_F^2:=G    
\end{equation}

By the Lipschitz continuity of the kernels, Assumption \ref{ass:kernel_smootheness}, and the residual bound,
\[
\bigl\|
(y_k-\phi(x_{2k}))
\bigl(K(x_{1k},\cdot)-K(x_{2k},\cdot)\bigr)
\bigr\|_{\mathcal H}
\le
GL_K\,|x_{1k}-x_{2k}|.
\]
Combining the above bounds and summing over $k$ yields
\[
\|\nabla_\phi \tilde{\mathcal{L}}(\phi,Z_1)-\nabla_\phi \tilde{\mathcal{L}}(\phi,Z_2)\|_{\mathcal H}
\le
\frac{2}{M}\sum_{k=1}^M
\bigl(\Xi\sqrt{B_K}+GL_K\bigr)\,|x_{1k}-x_{2k}|.
\]

Finally, note that
\[
|x_{1k}-x_{2k}|
=
|\langle A_k,Z_1Z_1^\top-Z_2Z_2^\top\rangle|
\le
\|A_k\|_F\,\|Z_1Z_1^\top-Z_2Z_2^\top\|_F.
\]
Since
\[
\|Z_1Z_1^\top-Z_2Z_2^\top\|_F
\le
(\|Z_1\|_F+\|Z_2\|_F)\,\|Z_1-Z_2\|_F
\]
and $\|A_k\|_F=1$ for the sampling matrices defined in \eqref{eq:sample_matrix},
we conclude that
\begin{align*}
\|\nabla_\phi \tilde{\mathcal{L}}(\phi,Z_1)-\nabla_\phi \tilde{\mathcal{L}}(\phi,Z_2)\|_{\mathcal H}
\le
2\,(\|Z_1\|_F+\|Z_2\|_F)
\bigl(\Xi\sqrt{B_K}+GL_K\bigr)
\|Z_1-Z_2\|_F,\\
\le 4\,\max_{Z\in\mathcal{C}}\|Z\|_F
\bigl(\Xi\sqrt{B_K}+GL_K\bigr)
\|Z_1-Z_2\|_F
\end{align*}
which completes the proof.
\end{proof}

\subsection{Proof of Lemma \ref{lem:Z-smoothness}}\label{pr:lem:Z-smoothness}

\begin{proof}
It suffices to prove that $\nabla_Z \tilde{\mathcal L}(\phi,\cdot)$ is Lipschitz on $\mathcal C$:
\begin{equation}
\label{eq:grad_lip_implies_smooth}
\|\nabla_Z \tilde{\mathcal L}(\phi,Z')-\nabla_Z \tilde{\mathcal L}(\phi,Z)\|_F
\le L_Z\|Z'-Z\|_F,\qquad \forall Z,Z'\in\mathcal C,
\end{equation}
since \eqref{eq:grad_lip_implies_smooth} implies the stated smoothness inequality
by the standard descent lemma.

Write the data-fit term as $F(Z):=\frac{1}{M}\sum_{k=1}^M f_k(Z)$ with
\[
f_k(Z) := \bigl(y_k-\phi(z_k(Z))\bigr)^2,
\qquad z_k(Z)=\langle A_k,ZZ^\top\rangle.
\]
A direct differentiation gives
\[
\nabla f_k(Z)
=
-2\,g_k(Z)\,\phi'(z_k(Z))\,\nabla z_k(Z),
\qquad
g_k(Z):=y_k-\phi(z_k(Z)),
\]
and
\[
\nabla z_k(Z)=(A_k+A_k^\top)Z.
\]
Hence,
\[
\nabla F(Z)
=
-\frac{2}{M}\sum_{k=1}^M g_k(Z)\phi'(z_k(Z))(A_k+A_k^\top)Z.
\]

Fix $Z,Z'\in\mathcal C$. For each $k$, let $s_k(Z):=g_k(Z)\phi'(z_k(Z))$.
Then
\[
\nabla f_k(Z')-\nabla f_k(Z)
=
-2\Bigl[s_k(Z')(A_k+A_k^\top)Z' - s_k(Z)(A_k+A_k^\top)Z\Bigr].
\]
Add and subtract $s_k(Z')(A_k+A_k^\top)Z$ to get
\[
\nabla f_k(Z')-\nabla f_k(Z)
=
-2\Bigl[s_k(Z')(A_k+A_k^\top)(Z'-Z) + (s_k(Z')-s_k(Z))(A_k+A_k^\top)Z\Bigr].
\]
Taking Frobenius norms and using $\|A_k+A_k^\top\|_F\le 2\|A_k\|_F\le 2$ yields
\begin{equation}
\label{eq:term_split}
\|\nabla f_k(Z')-\nabla f_k(Z)\|_F
\le
4|s_k(Z')|\|Z'-Z\|_F
+
4|s_k(Z')-s_k(Z)|\|Z\|_F.
\end{equation}

We bound $|s_k(Z')|$ and $|s_k(Z')-s_k(Z)|$.
By the residual bound and $|\phi'|\le \Xi$,
\[
|s_k(Z')| = |g_k(Z')\phi'(z_k(Z'))|\le G\,\Xi.
\]

where $G$ follows the definition at equation \ref{def:G}.
Next,
\[
s_k(Z')-s_k(Z)
=
g_k(Z')\bigl(\phi'(z_k(Z'))-\phi'(z_k(Z))\bigr)
+
\phi'(z_k(Z))\bigl(g_k(Z')-g_k(Z)\bigr).
\]
RKHS functions satisfy

\[
\nabla \phi(x)-\nabla\phi(x')=\langle\phi,\nabla_x K(x,\cdot)-\nabla_x K(x',\cdot)\rangle_{\mathcal{H}}.
\]
for any $x,x'$.
By applying the Cauchy-Schwarz inequality and setting $x=z_k(Z)$ and $x'=z_k(Z')$, we have

\begin{align*}
|\nabla\phi(z_k(Z))-\nabla\phi(z_k(Z'))|&\leq \|\phi\|_{\mathcal{H}}\|\nabla_x K(z_k(Z),\cdot)-\nabla_x K(z_k(Z'),\cdot)\|,\\
&\leq\sqrt{\frac{2L_{max}}{\alpha}}\|\nabla_x K(z_k(Z),\cdot)-\nabla_x K(z_k(Z'),\cdot)\|.
\end{align*}

Thus, from the assumption, we get

\[
|\nabla\phi(z_k(Z))-\nabla\phi(z_k(Z'))|\leq \sqrt{\frac{2L_{max}}{\alpha}}L_{K'}|z_k(Z')-z_k(Z)|.
\]

Using this Lipschitzness of $\phi'$ and $|\phi'|\le \Xi$,
together with $|g_k(Z')|\le G$ and
\[
|g_k(Z')-g_k(Z)|
=
|\phi(x_k(Z'))-\phi(z_k(Z))|
\le \Xi\,|z_k(Z')-z_k(Z)|,
\]
we obtain
\begin{align*}
|s_k(Z')-s_k(Z)|
\le
G_\phi \|\phi\|_{max}L_{K'}|z_k(Z')-z_k(Z)|
+
\Xi^2|z_k(Z')-z_k(Z)|
\\
=
\left(\Xi^2+\sqrt{\frac{2L_{max}}{\alpha}}G_\phi L_{K'}\right)|z_k(Z')-z_k(Z)|.
\end{align*}
Finally, since
\[
|z_k(Z')-z_k(Z)|
=
|\langle A_k, Z'Z'^\top - ZZ^\top\rangle|
\le
\|A_k\|_F\,\|Z'Z'^\top - ZZ^\top\|_F,
\]
and
\[
\|Z'Z'^\top - ZZ^\top\|_F
\le (\|Z'\|_F+\|Z\|_F)\|Z'-Z\|_F
\le 2\max_{Z\in\mathcal{C}}\|Z\|_F\|Z'-Z\|_F,
\]
we conclude
\[
|x_k(Z')-z_k(Z)|
\le 2R_Z\|Z'-Z\|_F,
\]
and hence
\[
|s_k(Z')-s_k(Z)|
\le
2\max_{Z\in\mathcal{C}}\|Z\|_F\left(\Xi^2+\sqrt{\frac{2L_{max}}{\alpha}}G L_{K'}\right)\|Z'-Z\|_F.
\]

Plugging these bounds into \eqref{eq:term_split} and using $\|Z\|_F\le R_Z$ gives
\[
\|\nabla f_k(Z')-\nabla f_k(Z)\|_F
\le
4G\Xi\|Z'-Z\|_F
+
8\max_{Z\in\mathcal{C}}\|Z\|_F^2\left(\Xi^2+\sqrt{\frac{2L_{max}}{\alpha}}G L_{K'}\right)\|Z'-Z\|_F.
\]
Averaging over $k$ yields the same bound for $\|\nabla F(Z')-\nabla F(Z)\|_F$.

Finally, the regularizer term $\frac{\lambda}{4}R(Z)$ has Lipschitz gradient on $\mathcal C$
with some constant $L_R$ depending on $\max_{Z\in\mathcal{C}}\|Z\|_F$ and $\lambda$, hence
\[
\|\nabla_Z \tilde{\mathcal L}(\phi,Z')-\nabla_Z \tilde{\mathcal L}(\phi,Z)\|_F
\le \left(4G\Xi+8\max_{Z\in\mathcal{C}}\|Z\|_F^2\left(\Xi^2+\sqrt{\frac{2L_{max}}{\alpha}}G L_{K'}\right)+L_R\right)\|Z'-Z\|_F,
\]
which is exactly \eqref{eq:grad_lip_implies_smooth} with $L_Z=\left(4G\Xi+8\max_{Z\in\mathcal{C}}\|Z\|_F^2\left(\Xi^2+\sqrt{\frac{2L_{max}}{\alpha}}G L_{K'}\right)+L_R\right)$.
This completes the proof.
\end{proof}

\subsection{Proof of Lemma \ref{le:local_curvature}:}\label{pr:le:local_curvature}
\begin{proof}
   From the expression of $\nabla_Z \tilde{\mathcal{L}}$, i.e., 
       $$
    \nabla_Z \tilde{\mathcal{L}}(Z,\phi) = \frac{2}{M}\sum_{k=1}^{M}
    \big(\phi\big(\langle A_{k}, ZZ^\top\rangle \big)-y_{k}\big) \phi'\big(\langle A_{k}, ZZ^\top\rangle \big)  (A_k+ A_k^\top) Z + \lambda DZZ^\top DZ.
    $$ 
   This leads to
    \begin{align*}
        \langle \nabla_Z \tilde{\mathcal{L}}(Z,\phi), \Delta(Z)\rangle =  & \frac{2}{M}\sum_{k=1}^{M}
    h_k \langle (A_k+ A_k^\top ) Z, \Delta(Z)\rangle +\lambda\langle DZZ^\top DZ, \Delta(Z)\rangle,
    \end{align*}
    where $ h_k =  \phi'(z_k)\left( \phi(z_k)- y_k \right), z_k = \langle A_k,ZZ^\top  \rangle$. 
    With slight abuse of notation, we use $\Phi$ and $\Delta$ to denote $\Phi(Z)$ and $\Delta(Z)$, respectively. Then, we have $Z = \Phi + \Delta$. 
    Therefore, the term $\langle (A_k+ A_k^\top ) Z, \Delta \rangle$ can be expanded as follows:
    \begin{align*}
        \langle (A_k+ A_k^\top ) Z, \Delta \rangle &= \langle (A_k+ A_k^\top ) \Phi, \Delta \rangle + \langle (A_k+ A_k^\top ) \Delta, \Delta \rangle\\
        &= \langle A_k+ A_k^\top , \Delta \Phi^\top  \rangle + \langle A_k+ A_k^\top , \Delta \Delta^\top  \rangle\\
        &= \langle A_k, \Delta \Phi^\top  \rangle + \langle A_k, \Phi\Delta^\top  \rangle+2\langle A_k, \Delta \Delta^\top  \rangle
    \end{align*}
    Assuming noiseless observations, i.e., $y_k = \phi^*(\langle A_k,Z^*Z^{*T} \rangle)$ yields
    \begin{align*}
        h_k &= \phi'(z_k)\big(\phi(z_k)- \phi(z^*_k)+\phi(z^*_k) - \phi^*(z^*_k)\big),  \text{where }\quad z^*_k =  \langle A_k,Z^*Z^{*T} \rangle = \langle A_k, \Phi \Phi^\top  \rangle
    \end{align*}
By the mean value theorem, we get
    \begin{align*}
        \phi(z_k) - \phi(z^*_k) &= \phi'(s_k) (z_k - z^*_k  ) \quad \text{for some } s_k \text{ in the interval between } z_k \text{ and } z^*_k\\
        &= \phi'(s_k) \left(\langle A_k,ZZ^\top  \rangle - \langle A_k,\Phi \Phi^\top  \rangle \right) \\
        &= \phi'(s_k) \left(\langle A_k,\Phi \Delta^\top  + \Delta \Phi^\top \rangle + \langle A_k,\Delta \Delta^\top \rangle\right)  \quad (\text{because }Z = \Phi + \Delta)\\
        &= \phi'(s_k) \left(\langle A_k+ A_k^\top ,\Delta \Phi^\top \rangle + \frac{1}{2}\langle A_k+ A_k^\top ,\Delta \Delta^\top \rangle\right)\\
        &= \phi'(s_k) \left(\langle A_k,\Delta \Phi^\top \rangle +\langle A_k,\Phi\Delta^\top \rangle  + \langle A_k,\Delta \Delta^\top \rangle\right).
    \end{align*}
Let $e_k^*(\phi):= \phi(z^*_k) - \phi^*(z^*_k)$ and assume that $|e^*_k(\phi)|\leq\varepsilon$ for all $k$.
 Putting the above equations together, we get:
    \begin{align*}
        &h_k \langle (A_k+ A_k^\top ) Z, \Delta \rangle \\
        &\ = \phi'(z_k) \left(e^*_k(\phi) + \phi'(s_k)\big(\langle A_k , \Delta \Phi^\top \rangle + \langle A_k , \Phi\Delta^\top \rangle + \langle A_k,\Delta \Delta^\top \rangle\big)\right)\\
        &\qquad \left(\langle A_k,\Delta \Phi^\top \rangle+ \langle A_k,\Phi\Delta^\top \rangle + 2\langle A_k,\Delta \Delta^\top \rangle\right) \\
        &\ \geq \phi'(z_k) e^*_k(\phi) \left(\langle A_k,\Delta \Phi^\top \rangle+ \langle A_k,\Phi\Delta^\top \rangle + 2\langle A_k,\Delta \Delta^\top \rangle\right)\\
        &+\phi'(z_k)\phi'(s_k)\left(\frac{1}{2}\langle A_k ,\Delta \Phi^\top +\Phi\Delta^\top \rangle^2 - \frac{5}{2}\langle A_k ,\Delta \Delta^\top \rangle^2\right).
    \end{align*}
The last step uses the inequality $a^2 + 3ab + 2b^2$ $\geq \frac{a^2}{2} - \frac{5b^2}{2}$. Note that the coefficient $\phi'(z_k)\phi'(s_k)$ is positive.
This implies
    \begin{align*}
    \langle \nabla_Z \tilde{\mathcal{L}}(Z,\phi), \Delta\rangle &\geq   \frac{2}{M}\sum_{k=1}^{M} \phi'(z_k)e^*_k(\phi) \left(\langle A_k,\Delta \Phi^\top + \Phi\Delta^\top \rangle + 2\langle A_k,\Delta \Delta^\top \rangle\right)\\
    &+\xi^2\mathcal{D}(\Delta \Phi^\top +\Phi\Delta^\top ) - {5\Xi^2}\mathcal{D}(\Delta \Delta^\top )+\lambda\langle DZZ^\top DZ, \Delta\rangle,
    \end{align*}

i) In the complete observation setting, $M=nT$ and the following equation satisfies, 
\begin{align}\label{eq:complete:1}
  &\mathcal{D}(\Delta \Phi^\top +\Phi\Delta^\top )=  \frac{1}{nT}\|\Delta_U \Phi_V^\top +\Phi_U\Delta_V^\top \|_F^2\\
  &=\frac{1}{nT}(\|\Delta_U \Phi_V^\top \|_F^2+\|\Phi_U\Delta_V^\top \|_F^2 + 2\langle \Delta_U \Phi_V^\top , \Phi_U\Delta_V^\top \rangle)
\end{align}
where $\Phi=(\Phi_U,\Phi_V)$ and $\Delta=(\Delta_U,\Delta_V)$ denote the split of $\Phi$ and $\Delta$ into the first $n$ and last $T$ rows.

Note that 
$$
2\langle \Delta_U \Phi_V^\top , \Phi_U\Delta_V^\top \rangle = \frac{1}{2}\langle (\Phi\Delta^\top -D\Phi\Delta^\top  D)\Phi, \Delta\rangle.
$$
Now using Lemma \ref{lemm:12}, we get
\begin{align*}
\mathcal{D}(\Delta \Phi^\top +\Phi\Delta^\top ) &+\frac{\lambda}{\xi^2}\langle DZZ^\top DZ, \Delta\rangle = \frac{1}{nT}\Big(\|\Delta_U \Phi_V^\top \|_F^2+\|\Phi_U\Delta_V^\top \|_F^2\Big) \\
&+ \frac{2}{nT}\langle \Delta_U \Phi_V^\top , \Phi_U\Delta_V^\top \rangle 
+\frac{\lambda}{\xi^2}\langle DZZ^\top DZ, \Delta\rangle\\
&\geq \frac{1}{nT}\Big(\|\Delta_U \Phi_V^\top \|_F^2+\|\Phi_U\Delta_V^\top \|_F^2\Big) + \frac{\lambda}{2\xi^2}\|{\Phi^\top D\Delta}\|_F^2 - \frac{7\lambda}{2\xi^2}\|{\Delta}\|_F^4 
\\
&+ \left(\frac{\lambda}{\xi^2} - \frac{1}{2nT}\right)\text{tr}(\Phi^\top D\Delta \Phi^\top D\Delta)
\end{align*}
Setting $\frac{\lambda}{\xi^2}=\frac{1}{2nT}$ yields
\begin{align*}
&\mathcal{D}(\Delta \Phi^\top +\Phi\Delta^\top ) +\frac{\lambda}{\xi^2}\langle DZZ^\top DZ, \Delta\rangle \\
&\geq  \frac{1}{nT}\Big(\|\Delta_U \Phi_V^\top \|_F^2+\|\Phi_U\Delta_V^\top \|_F^2\Big) + \frac{1}{4nT}\|{\Phi^\top D\Delta}\|_F^2 - \frac{7}{4nT}\|{\Delta}\|_F^4.
\end{align*}

On the other hand, we have $\|\Delta_U \Phi_V^\top \|_F^2\geq \sigma^*_r\|\Delta_U\|_F^2$ and $\|\Phi_U\Delta_V^\top  \|_F^2\geq \sigma^*_r\|\Delta_V\|_F^2$, which in combination with the above inequality gives us
\begin{align*}
    &\mathcal{D}(\Delta \Phi^\top +\Phi\Delta^\top ) +\frac{\lambda}{\xi^2}\langle DZZ^\top DZ, \Delta\rangle \geq  \frac{\sigma^*_r}{nT}\|\Delta\|_F^2 + \frac{1}{4nT}\|{\Phi^\top D\Delta}\|_F^2 - \frac{7}{4nT}\|{\Delta}\|_F^4\\
 &\mathcal{D}(\Delta \Delta^\top )\leq \frac{1}{nT}\|\Delta\|^4_F
\end{align*}
Putting the above results together give us
\begin{align*}
    &{\xi^2}\mathcal{D}(\Delta \Phi^\top +\Phi\Delta^\top ) - {5\Xi^2}\mathcal{D}(\Delta \Delta^\top )+\lambda\langle DZZ^\top DZ, \Delta\rangle\\
    &\geq \frac{\sigma^*_r\xi^2}{nT} \|\Delta\|_F^2+ \frac{\xi^2}{4nT}\|{\Phi^\top D\Delta}\|_F^2 -\Big(\frac{7\xi^2}{4nT}+\frac{5\Xi^2}{nT}\Big)\|\Delta\|_F^4\\
    &\geq \frac{\sigma^*_r\xi^2}{nT} \|\Delta\|_F^2+ \frac{\xi^2}{4nT}\|{\Phi^\top D\Delta}\|_F^2 -\epsilon\sigma_r^*\Big(\frac{7\xi^2}{4nT}+\frac{5\Xi^2}{nT}\Big)\|\Delta\|_F^2\\
    &\geq \frac{3}{10}\frac{\sigma^*_r\xi^2}{nT} \|\Delta\|_F^2+ \frac{\xi^2}{4nT}\|{\Phi^\top D\Delta}\|_F^2 .
\end{align*}
The last inequality is obtained by setting $\epsilon=\xi^2/(10\Xi^2)$.
It is straightforward to see that $\varepsilon\leq B_k\|\phi-\phi^\star\|_\mathcal{H}$ for all $k$, $\sqrt{2\sigma_1^*}\geq\|\Phi\|_2\geq\max\{\|\Phi_U\|_2, \|\Phi_V\|_2\}$ and via Lemma \ref{lemma:usi2}, we get
\begin{align*}
       &\frac{2}{nT}\sum_{k=1}^{nT} \phi'(z_k)e^*_k(\phi) \left(\langle A_k,\Delta \Phi^\top + \Phi\Delta^\top \rangle + 2\langle A_k,\Delta \Delta^\top \rangle\right)\\
       &\geq -\frac{2}{nT}\Xi \varepsilon\sum_{k=1}^{nT}  \left|\langle A_k,\Delta \Phi^\top + \Phi\Delta^\top \rangle + 2\langle A_k,\Delta \Delta^\top \rangle\right|\\
       &\geq -{2\Xi \varepsilon} \Big(\mathcal{G}(\Delta \Phi^\top + \Phi\Delta^\top ) + 2\mathcal{G}(\Delta \Delta^\top )\Big)\\
       &\geq -{2\Xi \varepsilon} \Big(\sqrt{\mathcal{D}(\Delta \Phi^\top + \Phi\Delta^\top )} + 2\sqrt{\mathcal{D}(\Delta \Delta^\top )}\Big)\\
       &\geq -\frac{2\Xi \varepsilon}{\sqrt{nT}} \Big(\|\Delta_U \Phi_V^\top + \Phi_U\Delta_V^\top \|_F + 2\|\Delta\|_F^2\Big)\\
       &\geq-\frac{2\Xi \varepsilon}{\sqrt{nT}} \Big(\|\Delta_U \Phi_V^\top \|_F+ \|\Phi_U\Delta_V^\top \|_F +  2\|\Delta\|_F^2\Big)\\
       &\geq-\frac{2\Xi \varepsilon}{\sqrt{nT}} \Big(\|\Delta_U\|_F \|\Phi_V\|_2+ \|\Phi_U\|_2\|\Delta_V\|_F +  2\|\Delta\|_F^2\Big)\\
       &\geq-\frac{2\Xi \varepsilon}{\sqrt{nT}}\Big(\sqrt{2\sigma_1^*}\|\Delta_U\|_F+ \sqrt{2\sigma_1^*}\|\Delta_V\|_F +  2\|\Delta\|_F^2\Big)\\
       &\geq-\frac{4\Xi \varepsilon}{\sqrt{nT}} \Big(\sqrt{\sigma_1^*}+  \|\Delta\|_F\Big)\|\Delta\|_F.
    \end{align*}
    Recall that $Z\in\mathcal{B}(\epsilon)$, and thus $\|\Delta\|_F\leq \sqrt{\epsilon\sigma_r^*}$. Therefore, we get
    \begin{align*}
    \langle \nabla_Z \tilde{\mathcal{L}}(Z,\phi), \Delta\rangle &\geq -\frac{4\Xi \varepsilon}{\sqrt{nT}} \Big(\sqrt{\sigma_1^*}+  \|\Delta\|_F\Big)\|\Delta\|_F +\frac{3\xi^2}{10nT}\sigma_r^*\|\Delta\|_F^2 +\frac{\xi^2}{4nT}\|{\Phi^\top D\Delta}\|_F^2,\\
    & \geq -\frac{4\Xi \varepsilon}{\sqrt{nT}}  \Big(\sqrt{\sigma_1^*}+  \|\Delta\|_F\Big)\|\Delta\|_F +\frac{3\xi^2}{10nT}\sigma_r^*\|\Delta\|_F^2,\\
    & \geq -\frac{4\Xi }{\sqrt{nT}} \Big(\sqrt{\sigma_1^*}+  \|\Delta\|_F\Big)B_K\|\Delta\|_F\|\phi-\phi^\star\|_\mathcal{H} +\frac{3\xi^2}{10nT}\sigma_r^*\|\Delta\|_F^2,\\
    &\geq -\frac{8\Xi}{\sqrt{nT}} \sqrt{\sigma_1^*}B_K\|\Delta\|_F\|\phi-\phi^\star\|_\mathcal{H} +\frac{3\xi^2}{10nT}\sigma_r^*\|\Delta\|_F^2.
    \end{align*}
 
{
ii) In the random observation setting, using Lemma \ref{lem:con1}, we get
\begin{align*}
    &\mathcal{D}(\Delta \Phi^\top +\Phi\Delta^\top )= \frac{1}{M}\sum_{i,j\in \Omega}([\Delta_U]_i [\Phi_V]_j^\top +[\Phi_U]_i[\Delta_V]_j^\top )^2\\
    &= \frac{1}{M}\sum_{i,j\in \Omega}([\Delta_U]_i [\Phi_V]_j^\top )^2+\frac{1}{M}\sum_{i,j\in \Omega}([\Phi_U]_i[\Delta_V]_j^\top )^2+\frac{2}{M}\sum_{i,j\in \Omega}([\Delta_U]_i [\Phi_V]_j^\top )([\Phi_U]_i[\Delta_V]_j^\top )\\
    &\geq\frac{(1-\delta)p}{M}\sum_{i,j}([\Delta_U]_i [\Phi_V]_j^\top )^2+\frac{(1-\delta)p}{M}\sum_{i,j}([\Phi_U]_i[\Delta_V]_j^\top )^2\\
    &+\frac{2p}{M}\sum_{i,j}([\Delta_U]_i [\Phi_V]_j^\top )([\Phi_U]_i[\Delta_V]_j^\top )-\frac{2p\delta}{M}\|\Delta_U \Phi_V^\top \|_F\|\Phi_U\Delta_V^\top \|_F^2\\
    &\geq \frac{(1-\delta)}{nT}\|\Delta_U\Phi_V^\top \|_F^2+\frac{(1-\delta)}{nT}\|\Delta_V\Phi_U^\top \|_F^2-\frac{2\delta}{nT}\|\Delta_U \Phi_V^\top \|_F\|\Phi_U\Delta_V^\top \|_F^2+\frac{2}{nT}\langle \Delta_U \Phi_V^\top ,\Phi_U\Delta_V^\top \rangle\\
&\geq \frac{(1-2\delta)}{nT} \Big(\|\Delta_U \Phi_V^\top \|_F^2+\|\Phi_U\Delta_V^\top \|_F^2 \Big)+ \frac{2}{nT}\langle \Delta_U \Phi_V^\top , \Phi_U\Delta_V^\top \rangle
\end{align*}
In this setting, using Lemma \ref{lemm:12}, we get
\begin{align*}
&\mathcal{D}(\Delta \Phi^\top +\Phi\Delta^\top ) +\frac{\lambda}{\xi^2}\langle DZZ^\top DZ, \Delta\rangle \geq \frac{1-2\delta}{nT}\Big(\|\Delta_U \Phi_V^\top \|_F^2+\|\Phi_U\Delta_V^\top \|_F^2\Big)\\
&+ \frac{2}{nT}\langle \Delta_U \Phi_V^\top , \Phi_U\Delta_V^\top \rangle +\frac{\lambda}{\xi^2}\langle DZZ^\top DZ, \Delta\rangle\\
&\geq \frac{1-2\delta}{nT}\Big(\|\Delta_U \Phi_V^\top \|_F^2+\|\Phi_U\Delta_V^\top \|_F^2\Big) + \frac{\lambda}{2\xi^2}\|{\Phi^\top D\Delta}\|_F^2 - \frac{7\lambda}{2\xi^2}\|{\Delta}\|_F^4 \\
&+ \left(\frac{\lambda}{\xi^2} - \frac{1}{2nT}\right)\text{tr}(\Phi^\top D\Delta \Phi^\top D\Delta)
\end{align*}
Setting $\frac{\lambda}{\xi^2}=\frac{1}{2nT}$ yields
\begin{align*}
&\mathcal{D}(\Delta \Phi^\top +\Phi\Delta^\top ) +\frac{\lambda}{\xi^2}\langle DZZ^\top DZ, \Delta\rangle \geq  \frac{1-2\delta}{nT}\Big(\|\Delta_U \Phi_V^\top \|_F^2+\|\Phi_U\Delta_V^\top \|_F^2\Big)\\
&\qquad+ \frac{1}{4nT}\|{\Phi^\top D\Delta}\|_F^2 - \frac{7}{4nT}\|{\Delta}\|_F^4.
\end{align*}
On the other hand, we have $\|\Delta_U \Phi_V^\top \|_F^2\geq \sigma^*_r\|\Delta_U\|_F^2$ and $\|\Phi_U\Delta_V^\top  \|_F^2\geq \sigma^*_r\|\Delta_V\|_F^2$, which in combination with the above inequality gives us
\begin{align*}
    &\mathcal{D}(\Delta \Phi^\top +\Phi\Delta^\top ) +\frac{\lambda}{\xi^2}\langle DZZ^\top DZ, \Delta\rangle \geq  (1-2\delta)\frac{\sigma^*_r}{nT}\|\Delta\|_F^2 + \frac{1}{4nT}\|{\Phi^\top D\Delta}\|_F^2 - \frac{7}{4nT}\|{\Delta}\|_F^4.
\end{align*}
From the concentration results in Lemma \ref{lemma:con1} and Lemma \ref{lemma:usi}, and for $M\in\mathcal{O}(\frac{(\mu r\kappa)^2\min\{n,T\}}{\omega^2}\log(\frac{n+T}{\delta}))$, we obtain
\begin{align*}
 \mathcal{D}(\Delta \Delta^\top )\leq \Big(\frac{\|\Delta\|^2_F}{4nT} +\omega\frac{\|\Delta\|^2_{2,\infty}}{2\mu r \kappa \min\{n,T\}}\Big)\|\Delta\|^2_F\leq \frac{\epsilon\sigma_r^*}{nT}\|\Delta\|_F^2
\end{align*}
The last inequality is coming from the fact that $Z\in \mathcal{B}(\epsilon)$ and by letting $\omega=\epsilon/20$. 
Putting the above results together give us
\begin{align*}
    &\xi^2\mathcal{D}(\Delta \Phi^\top +\Phi\Delta^\top ) - 5\Xi^2\mathcal{D}(\Delta \Delta^\top )+\lambda\langle DZZ^\top DZ, \Delta\rangle\\
    &\geq {\Big((1-2\delta)\frac{\sigma^*_r\xi^2}{nT} - \frac{7\xi^2}{4nT}\epsilon\sigma_r^*-5\frac{\Xi^2\epsilon\sigma_r^*}{nT}\Big)}\|\Delta\|_F^2+ \frac{\xi^2}{4nT}\|{\Phi^\top D\Delta}\|_F^2\\
    &= \underbrace{\Big((1-2\delta) - \frac{7}{4}\epsilon-5\frac{\Xi^2}{\xi^2}\epsilon\Big)}_{\geq 1/5}\frac{\sigma^*_r\xi^2}{nT}\|\Delta\|_F^2+ \frac{\xi^2}{4nT}\|{\Phi^\top D\Delta}\|_F^2.
\end{align*}
By setting $\delta = 1/8$ and $\epsilon=\xi^2/(20\Xi^2)$, we get result. On the other hand, using Lemma \ref{lemma:usi2}, we have
    \begin{align*}
        &\frac{2}{M}\sum_{k=1}^{M} \phi'(z_k)e^*_k(\phi) \left(\langle A_k,\Delta \Phi^\top + \Phi\Delta^\top \rangle+2\langle A_k,\Delta\Delta^\top \rangle\right)\\
        &\geq-{2\Xi \varepsilon} \Big(\mathcal{G}(\Delta \Phi^\top + \Phi\Delta^\top ) + 2\mathcal{G}(\Delta \Delta^\top )\Big)\geq -{2\Xi \varepsilon} \Big(\sqrt{\mathcal{D}(\Delta \Phi^\top + \Phi\Delta^\top )} + 2\sqrt{\mathcal{D}(\Delta \Delta^\top )}\Big)\\
        &\geq-\frac{2\Xi \varepsilon}{\sqrt{nT}}(\sqrt{1+\delta}\|\Delta_U \Phi_V^\top + \Phi_U\Delta_V^\top \|_F+2\sqrt{\epsilon\sigma_r^*}\|\Delta\|_F)\\
        &\geq-\frac{2\Xi \varepsilon}{\sqrt{nT}}(\sqrt{1+\delta}(\|\Delta_U \Phi_V^\top \|_F+\|\Phi_U\Delta_V^\top \|_F)+2\sqrt{\epsilon\sigma_r^*}\|\Delta\|_F)\\
        &\geq-\frac{2\Xi \varepsilon}{\sqrt{nT}}(\sqrt{1+\delta}\sqrt{2\sigma_1^*}(\|\Delta_U\|_F+\|\Delta_V\|_F)+2\sqrt{\epsilon\sigma_r^*}\|\Delta\|_F)\\
        &\geq-\frac{4\Xi \varepsilon}{\sqrt{nT}}(\sqrt{1+\delta}\sqrt{\sigma_1^*}+2\sqrt{\epsilon\sigma_r^*})\|\Delta\|_F\geq -\frac{8\Xi \varepsilon}{\sqrt{nT}}\sqrt{\sigma_1^*}\|\Delta\|_F,
    \end{align*}
    where we set $\delta=1/8$.  Recall that $Z\in\mathcal{B}(\epsilon)$, and thus $\|\Delta\|_F\leq \sqrt{\epsilon\sigma_r^*}$. Therefore, with probability at least $1-\delta$, we have
    \begin{align*}
    \langle \nabla_Z \tilde{\mathcal{L}}(Z,\phi), \Delta\rangle &\geq -\frac{8\Xi \varepsilon}{\sqrt{nT}} \sqrt{\sigma_1^*}\|\Delta\|_F +\frac{\xi^2}{5nT}\sigma_r^*\|\Delta\|_F^2 +\frac{\xi^2}{4nT}\|{\Phi^\top D\Delta}\|_F^2,\\
    & \geq -\frac{8\Xi \varepsilon}{\sqrt{nT}}  \sqrt{\sigma_1^*}\|\Delta\|_F +\frac{\xi^2}{5nT}\sigma_r^*\|\Delta\|_F^2,\\
    & \geq -\frac{8\Xi }{\sqrt{nT}} \sqrt{\sigma_1^*}B_K\|\Delta\|_F\|\phi-\phi^\star\|_\mathcal{H} +\frac{\xi^2}{5nT}\sigma_r^*\|\Delta\|_F^2.
    \end{align*}
}
\end{proof}

\begin{lemma}\label{lemm:12}
 For any $\lambda'\geq0$, we have
 \begin{align*}
         &\frac{1}{2}\langle (\Phi\Delta^\top -D\Phi\Delta^\top  D)\Phi, \Delta\rangle+ \lambda' \langle DZZ^\top DZ ,\Delta\rangle \geq\frac{\lambda'}{2}\|{\Phi^\top D\Delta}\|_F^2 - \frac{7\lambda'}{2}\|{\Delta}\|_F^4 \\
         &\qquad + \left(\lambda' - \frac{1}{2}\right)\text{tr}(\Phi^\top D\Delta \Phi^\top D\Delta)
    \end{align*}
 \end{lemma}

 \begin{proof}
        Furthermore, we have 
    \begin{align*}
        \langle DZZ^\top DZ, \Delta\rangle &= \langle D(ZZ^\top - \Phi\Phi^\top ) DZ, \Delta\rangle +\langle D\Phi\Phi^\top  DZ, \Delta\rangle\\
        &= \langle D(ZZ^\top - \Phi\Phi^\top ) DZ, \Delta\rangle +\langle D\Phi\Phi^\top  D\Phi, \Delta\rangle+\langle D\Phi\Phi^\top  D\Delta, \Delta\rangle.
    \end{align*}
    Note that $Z=\Delta+\Phi$ and $\Phi^\top D\Phi=0$, thus
    \begin{align*}
        \langle DZZ^\top DZ, \Delta\rangle &= \langle D(ZZ^\top - \Phi\Phi^\top ) DZ, \Delta\rangle +\| \Phi^\top  D\Delta\|_F^2\\
        &= \langle D(\Delta\Delta^\top +\Delta\Phi^\top +\Phi\Delta^\top ) D(\Delta+\Phi), \Delta\rangle +\| \Phi^\top  D\Delta\|_F^2\\
        &= \| \Delta^\top  D\Delta\|_F^2 +\| \Phi^\top  D\Delta\|_F^2 + 3\text{tr}(\Delta^\top  D\Phi\Delta^\top  D\Delta) + \text{tr}(\Delta^\top  D\Phi\Delta^\top  D\Phi).
    \end{align*}
    Using the above equation and knowing that $\Delta^\top  \Phi$ is symmetric, we get
    \begin{align*}
         &\frac{1}{2}\langle (\Phi\Delta^\top -D\Phi\Delta^\top  D)\Phi, \Delta\rangle+ \lambda' \langle DZZ^\top DZ ,\Delta\rangle \\
         &= \lambda'\| \Delta^\top  D\Delta\|_F^2 +\lambda'\| \Phi^\top  D\Delta\|_F^2 + 3\lambda'\text{tr}(\Delta^\top  D\Phi\Delta^\top  D\Delta)\\
         &+ (\lambda'-\frac{1}{2})\text{tr}(\Delta^\top  D\Phi\Delta^\top  D\Phi)+\|\Delta^\top  \Phi\|_F^2\\
         &=\frac{\lambda'}{2}\|{\Phi^\top D\Delta}\|_F^2 +\|\Phi^\top D\Delta + 3\Delta^\top D\Delta\|_F^2- \frac{7\lambda'}{2}\|{\Delta}\|_F^4 + \left(\lambda' - \frac{1}{2}\right)\text{tr}(\Phi^\top D\Delta \Phi^\top D\Delta).
    \end{align*}
 \end{proof}

\subsection{Proof of Lemma \ref{lemma:smoothness}}\label{pr:lemma:smoothness}
\begin{proof}
From the proof of the previous lemma, we have 
    \begin{align*}\notag
        &\langle \nabla_Z \tilde{\mathcal{L}}(Z,\phi)-\lambda DZZ^\top DZ, H\rangle =   \frac{2}{M}\sum_{k=1}^{M} h_k \langle A_k+ A_k^\top , HZ^\top  \rangle \\
        &= \frac{2}{M}\sum_{k=1}^{M} \phi'(z_k)\phi'(s_k)\left(\frac{e^*_k(\phi)}{\phi'(s_k)} + \langle A_k + A_k^\top , \Delta \Phi^\top \rangle + \frac{1}{2}\langle A_k + A_k^\top ,\Delta \Delta^\top \rangle\right)
        \left( \langle A_k+ A_k^\top , HZ^\top  \rangle \right),
    \end{align*}
    where
        \begin{align*}
        h_k = \phi'(z_k)\phi'(s_k) \left(\frac{e^*_k(\phi)}{\phi'(s_k)} + \langle A_k + A_k^\top , \Delta \Phi^\top \rangle + \frac{1}{2}\langle A_k + A_k^\top ,\Delta \Delta^\top \rangle\right).
    \end{align*}
Next, we invoke two inequalities which apply to any sequence of scalars $(a_k)_{k \in [M]}, (b_k)_{k \in [M]},$ and  $(c_k)_{k \in [M]}$ with $a_k \geq 0$:
    \begin{align*}
        \left(\sum_{k = 1}^M a_k b_k c_k\right)^2 \leq \left(\sum_{k = 1}^M a_k b_k^2 \right) \left(\sum_{k = 1}^M a_k c_k^2 \right), \quad
        \left(\sum_{k = 1}^M a_k b_k^2 \right) \leq \left(\max_{k \in [M]} a_k\right) \left(\sum_{k = 1}^M b_k^2 \right).
    \end{align*}
The first inequality can be viewed as a form of the Cauchy-Schwarz inequality and the second, a form of Hölder's inequality.
Squaring both sides of the previous equation and applying these inequalities with 
\begin{align*}
    &a_k =  \phi'(z_k) \phi'(s_k), \ b_k = \frac{e^*_k(\phi)}{\phi'(s_k)}+\langle A_k + A_k^\top ,\Delta \Phi^\top \rangle + \frac{1}{2}\langle A_k + A_k^\top ,\Delta \Delta^\top  \rangle, \\
    &c_k = \langle A_k+ A_k^\top , HZ^\top  \rangle,
\end{align*}
and observing that $\max_{k \in [M]} a_k \leq \Xi^2$, we get
\begin{align*}
    &\langle \nabla_Z \tilde{\mathcal{L}}(Z,\phi)-\lambda DZZ^\top DZ, H\rangle^2 \\
    &\leq \frac{4\Xi^4}{M^2} \left( \sum_{k = 1}^M(\frac{e^*_k(\phi)}{\phi'(s_k)}+\langle A_k + A_k^\top ,\Delta \Phi^\top \rangle + \frac{1}{2}\langle A_k + A_k^\top ,\Delta \Delta^\top  \rangle)^2 \right) \left(\sum_{k = 1}^M \langle A_k+ A_k^\top , HZ^\top  \rangle^2 \right)\\
    &=\frac{4\Xi^4}{M^2}\left( \sum_{k = 1}^M(\frac{e^*_k(\phi)}{\phi'(s_k)}+\langle A_k ,\Delta \Phi^\top +\Phi\Delta^\top \rangle + \langle A_k ,\Delta \Delta^\top  \rangle)^2 \right) \left(\sum_{k = 1}^M \langle A_k, HZ^\top +ZH^\top  \rangle^2 \right) \\
    &\leq 12\Xi^4 \left( \frac{1}{M}\sum_{k = 1}^M\left(\frac{e^*_k(\phi)}{\phi'(s_k)}\right)^2+\mathcal{D}(\Delta \Phi^\top +\Phi\Delta^\top ) +\mathcal{D}(\Delta \Delta^\top )\right)\mathcal{D}(HZ^\top +ZH^\top ) .
\end{align*}
This implies 
\begin{align*}
    &\|\nabla_Z \tilde{\mathcal{L}}(Z,\phi)-\lambda DZZ^\top DZ \|_F^2=\sup_{H, \|H\|_F^2=1}\langle \nabla_Z \tilde{\mathcal{L}}(Z,\phi)-\lambda DZZ^\top DZ, H\rangle^2\\
    &\leq 12\Xi^4  \left( \frac{1}{M}\sum_{k = 1}^M\left(\frac{e^*_k(\phi)}{\phi'(s_k)}\right)^2+\mathcal{D}(\Delta \Phi^\top +\Phi\Delta^\top ) +\mathcal{D}(\Delta \Delta^\top )\right)\sup_{H, \|H\|_F^2=1}\mathcal{D}(HZ^\top +ZH^\top ).
\end{align*}

i) In the complete observation setting, by assuming $|e^*_k(\phi)|\leq\varepsilon$ for all $k$, using \eqref{eq:complete:1} and Lemma \ref{lemma:usi3}, we can bound the last terms as follows
\begin{align*}
&\leq 12\Xi^4  \left( \left(\frac{\varepsilon}{\xi}\right)^2+\frac{1}{nT}\|\Delta_U \Phi_V^\top +\Phi_U\Delta_V^\top \|_F^2 +\frac{1}{nT}\|\Delta\|_F^4\right)\sup_{H, \|H\|_F^2=1}\mathcal{D}(HZ^\top +ZH^\top )\\
&\leq 12\Xi^4  \left( \left(\frac{\varepsilon}{\xi}\right)^2+\frac{1}{nT}\|\Delta_U \Phi_V^\top +\Phi_U\Delta_V^\top \|_F^2 +\frac{1}{nT}\|\Delta\|_F^4\right)\sup_{H, \|H\|_F^2=1}\mathcal{D}(HZ^\top +ZH^\top )\\
&\leq 12\Xi^4  \left( \left(\frac{\varepsilon}{\xi}\right)^2+\frac{2}{nT}\big(\|\Delta_U \Phi_V^\top \|^2_F+\|\Phi_U\Delta_V^\top \|_F^2 \big)+\frac{1}{nT}\|\Delta\|_F^4\right)\sup_{H, \|H\|_F^2=1}\mathcal{D}(HZ^\top +ZH^\top )\\
&\leq 12\Xi^4 \left( \left(\frac{\varepsilon}{\xi}\right)^2+\frac{2\sigma_1^*}{nT}\|\Delta\|^2_F+\frac{1}{nT}\|\Delta\|_F^4\right)\sup_{H, \|H\|_F^2=1}\mathcal{D}(HZ^\top +ZH^\top ).
\end{align*}
Using Lemma \ref{lemma:usi}, we also have
\begin{align*}
    \mathcal{D}(HZ^\top +ZH^\top )=\frac{1}{nT}\|H_bZ_f^\top +Z_bH_f^\top \|^2_F\leq \frac{2}{nT}\|Z\|^2_{2,\infty}\|H\|^2_F\leq  \frac{1}{nT}\frac{7\mu r \sigma_1^*}{n+T}\|H\|^2_F.
\end{align*}
On the other hand, we have
\begin{align*}
    \|\nabla_Z \tilde{\mathcal{L}}(Z,\phi)\|_F^2&=\|\nabla_Z \tilde{\mathcal{L}}(Z,\phi)-\lambda DZZ^\top DZ + \lambda DZZ^\top DZ\|_F^2\\
    &\leq 2\|\nabla_Z \tilde{\mathcal{L}}(Z,\phi)-\lambda DZZ^\top DZ \|_F^2+2\lambda^2\| DZZ^\top DZ\|_F^2.
\end{align*}
Since $Z \in \mathcal{B}(\epsilon)$, $\epsilon\leq1$, we have $\|{\Delta}\|_F^2 \leq \sigma^*_r \leq \sigma^*_1$. 
Using this bound along with the analysis in \citet{zheng2016convergence} (Appendix C.2), we get:
\begin{align}\label{eq:lem2_2}
    \|{DZZ^\top DZ}\|_F^2 &\leq 6(\|{\Delta}\|_F^2 + 4\sigma^*_1)\|{\Delta}\|_F^2 \|{Z}\|_2^2 + 4\sigma^*_1\|{\Phi^\top D\Delta}\|_F^2 \nonumber \\
    &\leq 30\sigma^*_1\|{\Delta}\|_F^2 \|{Z}\|_2^2 + 4\sigma^*_1\|{\Phi^\top D\Delta}\|_F^2 \  \nonumber \\
    &\leq 180(\sigma^*_1)^2\|{\Delta}\|_F^2 + 4\sigma^*_1\|{\Phi^\top D\Delta}\|_F^2, \ \quad (\|{Z}\|_2^2 \leq 6\sigma^*_1)
\end{align}
The last bound can be derived using the fact that  $\|\Phi\|^2_2\leq 2\sigma^*_1$ and
\begin{align*}
    \|{Z}\|_2^2 = \|{\Phi + \Delta}\|_2^2 \leq (\|{\Phi}\|_2 + \|{\Delta}\|_2)^2 \leq 2(\|{\Phi}\|_2^2 + \|{\Delta}\|_2^2)  \leq 2(2\sigma^*_1 + \sigma^*_1) = 6\sigma^*_1.
\end{align*}
Putting all together and recalling that $\frac{\lambda}{\xi^2}=\frac{1}{2nT}$ and $\|\Delta\|^2\leq\epsilon\sigma_r^*$ yield
\begin{align*}
    \|\nabla_Z \tilde{\mathcal{L}}(Z,\phi)\|_F^2&\leq 24\Xi^4  \left( \left(\frac{\varepsilon}{\xi}\right)^2+\frac{2\sigma_1^*}{nT}\|\Delta\|^2_F+\frac{1}{nT}\|\Delta\|_F^4\right) \frac{7\mu r \sigma_1^*}{nT(n+T)}\\
    &+\frac{\xi^4}{n^2T^2}\Big(90(\sigma^*_1)^2\|{\Delta}\|_F^2 + 2\sigma^*_1\|{\Phi^\top D\Delta}\|_F^2\Big)\\
    &\leq  \Big(\frac{336\Xi^4\mu r(\sigma_1^*)^2}{(nT)^2(n+T)}+ \frac{90(\sigma^*_1)^2\xi^4}{(nT)^2}+ \frac{168\Xi^4  \mu r \sigma_1^*}{(nT)^2(n+T)}\|\Delta\|_F^2\Big)\|\Delta\|_F^2\\
    &+  \frac{168 \Xi^4  \mu r \sigma_1^*}{nT(n+T)\xi^2}\varepsilon^2 + \frac{2\xi^4\sigma^*_1}{(nT)^2}\|{\Phi^\top D\Delta}\|_F^2 \\
    &\leq  \Big(\frac{504\Xi^4\mu r(\sigma_1^*)^2}{(nT)^2(n+T)}+ \frac{90(\sigma^*_1)^2\xi^4}{(nT)^2}\Big)\|\Delta\|_F^2 +  \frac{168 \Xi^4  \mu r \sigma_1^*}{nT(n+T)\xi^2}\varepsilon^2\\
    &+ \frac{2\xi^4\sigma^*_1}{(nT)^2}\|\Phi\|_2^2\|D\|_2^2\Delta\|_F^2 \\
    &\leq  \frac{90\Xi^4\mu r(\sigma_1^*)^2}{(nT)^2}\Big(\frac{5.6}{(n+T)}+ \frac{\xi^4}{\Xi^4\mu r}+\frac{4\xi^4}{90\Xi^4\mu r}\|D\|_2^2\Big)\|\Delta\|_F^2 +  \frac{168 \Xi^4  \mu r \sigma_1^*}{nT(n+T)\xi^2}\varepsilon^2 \\
    &\leq  \frac{90\Xi^4\mu r(\sigma_1^*)^2}{(nT)^2}\|\Delta\|_F^2 +  \frac{168 \Xi^4  \mu r \sigma_1^*}{nT(n+T)\xi^2}\varepsilon^2\\
    &\leq \frac{1093\Xi^4(\sigma_1^*)^2\mu r}{n^2T^2}\|\Delta\|_F^2+\frac{336\Xi^4\mu r \sigma_1^*}{ nT\xi^2}\varepsilon^2.
\end{align*}
The last inequality holds when $\frac{5.6}{(n+T)}+ \frac{\xi^4}{\Xi^4\mu r}+\frac{4\xi^4}{90\Xi^4\mu r}\|D\|_2^2\leq1$ when $n+T$ is large enough. Note that $r,\mu\geq1$.

{
ii) In the random observation setting, we have:
Assuming $|e^*_k(\phi)|\leq\varepsilon$. Using the concentration results, the last terms can be bounded as follows
\begin{align*}
&\leq 12\Xi^4  \left( \left(\frac{\varepsilon}{\xi}\right)^2+2\frac{(1+\delta)}{nT}\big(\|\Delta_U \Phi_V^\top \|^2_F+\|\Phi_U\Delta_V^\top \|_F^2 \big) +\mathcal{D}(\Delta \Delta^\top )\right)\sup_{H, \|H\|_F^2=1}\mathcal{D}(HZ^\top +ZH^\top )\\
&\leq 12\Xi^4  \left( \left(\frac{\varepsilon}{\xi}\right)^2+2\frac{(1+\delta)}{nT}\big(\|\Delta_U \Phi_V^\top \|^2_F+\|\Phi_U\Delta_V^\top \|_F^2 \big) +\frac{\epsilon\sigma_r^*}{nT}\|\Delta\|_F^2\right)\sup_{H, \|H\|_F^2=1}\mathcal{D}(HZ^\top +ZH^\top )\\
&\leq 12\Xi^4 \left( \left(\frac{\varepsilon}{\xi}\right)^2+2\frac{(1+\delta)}{nT}\big(\|\Delta_U \Phi_V^\top \|^2_F+\|\Phi_U\Delta_V^\top \|_F^2 \big)+\frac{\epsilon\sigma_r^*}{nT}\|\Delta\|_F^2\right)\sup_{H, \|H\|_F^2=1}\mathcal{D}(HZ^\top +ZH^\top )\\
&\leq 12\Xi^4  \left( \left(\frac{\varepsilon}{\xi}\right)^2+2\frac{(1+\delta)\sigma_1^*}{nT}\|\Delta\|^2_F+\frac{\epsilon\sigma_r^*}{nT}\|\Delta\|_F^2\right)\sup_{H, \|H\|_F^2=1}\mathcal{D}(HZ^\top +ZH^\top ).
\end{align*}
From the concentration results in \eqref{eq:con_1}, we have
\begin{align*}
    \mathcal{D}(HZ^\top +ZH^\top )\leq \frac{4 }{\min\{n,T\}} \|Z\|^2_{2,\infty}\|H\|^2_F\leq \frac{4}{\min\{n,T\}} \frac{3.5\mu r \sigma_1^*}{n+T}\|H\|^2_F.
\end{align*}
On the other hand, we have
\begin{align*}
    \|\nabla_Z \tilde{\mathcal{L}}(Z,\phi)\|_F^2&=\|\nabla_Z \tilde{\mathcal{L}}(Z,\phi)-\lambda DZZ^\top DZ + \lambda DZZ^\top DZ\|_F^2\\
    &\leq 2\|\nabla_Z \tilde{\mathcal{L}}(Z,\phi)-\lambda DZZ^\top DZ \|_F^2+2\lambda^2\| DZZ^\top DZ\|_F^2.
\end{align*}
The inequality \eqref{eq:lem2_2} gives
    \begin{align*}
        \|{DZZ^\top DZ}\|_F^2 \leq 180(\sigma^*_1)^2\|{\Delta}\|_F^2 + 4\sigma^*_1\|{\Phi^\top D\Delta}\|_F^2.
    \end{align*}
Similar to the complete observation setting, and having $\delta=1/8$, we obtain
\begin{align*}
    \|\nabla_Z \tilde{\mathcal{L}}(Z,\phi)\|_F^2&\leq 24\Xi^4  \left( \left(\frac{\varepsilon}{\xi}\right)^2+\frac{9\sigma_1^*}{4nT}\|\Delta\|^2_F+\frac{\epsilon\sigma_r^*}{nT}\|\Delta\|_F^2\right) \frac{14\mu r \sigma_1^*}{(n+T)\min\{n,T\}}\\
    &+\frac{\xi^4}{n^2T^2}\Big(90(\sigma^*_1)^2\|{\Delta}\|_F^2 + 2\sigma^*_1\|{\Phi^\top D\Delta}\|_F^2\Big)\\
    &\leq \Big(\frac{1092\Xi^4(\sigma_1^*)^2\mu r}{nT(n+T)\min\{n,T\}}+\frac{90\xi^4(\sigma^*_1)^2}{n^2T^2}+\frac{4\xi^4(\sigma^*_1)^2\|D\|_2^2}{n^2T^2}\Big)\|\Delta\|_F^2\\
    &+\frac{336\Xi^4\mu r \sigma_1^*}{\xi^2(n+T)\min\{n,T\}}\varepsilon^2\\
    &\leq \frac{1093\Xi^4(\sigma_1^*)^2\mu r}{n^2T^2}\|\Delta\|_F^2+\frac{336\Xi^4\mu r \sigma_1^*}{ nT\xi^2}\varepsilon^2.
\end{align*}
}
\end{proof}

\subsection{Proof of Proposition \ref{lemma:d_bounded}}\label{lemma:d_bounded_p}
\begin{proof}
\textbf{Boundedness of $D_t$:} we have
\begin{align*}
    \|{\Delta_{t+1}}\|_F^2 
    &= \|{Z_{t+1} - \Phi_{t+1}}\|_F^2 \leq \|{Z_{t+1} - \Phi_{t}}\|_F^2 \\
    &= \|{\mathcal{P}_{\mathcal{C}}\left(Z_t - \zeta_t \nabla_Z \tilde{\mathcal{L}}(Z_t, \phi_t) \right) - \Phi_{t}}\|_F^2 \\
    &\leq \|{Z_t - \zeta_t \nabla_Z \tilde{\mathcal{L}}(Z_t, \phi_t) - \Phi_{t}}\|_F^2 \\
    &= \|{\Delta_{t} - \zeta_t \nabla_Z \tilde{\mathcal{L}}(Z_t, \phi_t)}\|_F^2\\    
    &= \|{\Delta_{t}}\|_F^2 + \zeta_t^2 \|\nabla_Z \tilde{\mathcal{L}}(Z_t, \phi_t)\|_F^2 - 2 \zeta_t \langle \nabla_Z \tilde{\mathcal{L}}(Z_t, \phi_t), \Delta_{t} \rangle.  
\end{align*}
The first inequality is by the definition of $\Phi(Z_{t+1})$. The second inequality is due to the projection property. 
From the results of Lemmas \ref{le:local_curvature} and \ref{lemma:smoothness}, with high probability, we have
\begin{align*}
    &\!\!\langle \nabla_Z \tilde{\mathcal{L}}(Z_t,\phi_t), \Delta_t\rangle\!\! \geq\!\mu_Z\|\Delta_t\|_F^2 -A_Z\|\Delta_t\|_F\!\|\phi_t\!-\!\phi^\star\|_\mathcal{H},\\
    &        \|\nabla_Z \tilde{\mathcal{L}}(Z_t,\phi_t)\|_F^2
    \leq B_Z\|\Delta_t\|_F^2+A'_Z\|\phi_t-\phi^*\|_\mathcal{H}^2.
\end{align*}
It follows that
\begin{align*}
    &\|{\Delta_{t+1}}\|_F^2 \leq \|{\Delta_{t}}\|_F^2 + \zeta_t^2 \|{\nabla_Z \tilde{\mathcal{L}}(Z_t, \phi_t)}\|_F^2 - 2 \zeta_t \langle \nabla_Z \tilde{\mathcal{L}}(Z_t, \phi_t), \Delta_{t} \rangle  \\
    &\leq (1 - 2\zeta_t \mu_Z + \zeta_t^2 B_Z) \|{\Delta_t}\|_F^2  + \zeta_t\Big(\zeta_t A'_Z \|\phi_t-\phi^*\|_\mathcal{H} + 2 A_Z\|{\Delta_t}\|_F \Big)
    \|\phi_t-\phi^*\|_\mathcal{H} \\ 
    &\leq (1 - \zeta_t \mu_Z) \|{\Delta_t}\|_F^2+ \zeta_t\Big(\zeta_tA'_Z \|\phi_t-\phi^*\|_\mathcal{H} + 2 A_Z \|{\Delta_t}\|_F\Big)\|\phi_t-\phi^*\|_\mathcal{H},
\end{align*}
provided $\zeta_t \leq \min\{1/A_Z, \mu_Z/B_Z, 1/\mu_Z\}$. 
Now, consider the following recursion, which captures the above dynamic of $\|\Delta_t\|_F$.
    \begin{align*}
        x_{t+1}^2&\leq \big(1-b\zeta_t\big)x_t^2+2a\zeta_t\varepsilon x_t + c\varepsilon^2\zeta_t^2,
    \end{align*}
where $b=\mu_Z$, $a=A_Z$, $c=A'_Z$, and $\|\phi_t-\phi^*\|_\mathcal{H}\leq \varepsilon$. 
This can be further bounded as follows
    \begin{align*}
        x_{t+1}^2&\leq \big(1-b\zeta_t\big)x_t^2+(\frac{a^2\zeta^2_t\varepsilon^2}{\gamma} +  x_t^2\gamma) + c\varepsilon^2\zeta_t^2\\
        &=\big(1+ \gamma-b\zeta_t\big)x_t^2+(\frac{a^2}{\gamma} +  c)\varepsilon^2\zeta_t^2.
    \end{align*}
Let $\gamma = b\zeta_t/2$. This leads to
    \begin{align*}
        x_{t+1}^2\leq\big(1 -\frac{b\zeta_t}{2}\big)x_t^2+(2\frac{a^2}{b\zeta_t} +  c)\varepsilon^2\zeta_t^2=\big(1 -\frac{b\zeta_t}{2}\big)x_t^2+(2\frac{a^2}{b} +  c\zeta_t)\varepsilon^2\zeta_t.
    \end{align*}
    For $\zeta_t=\zeta $ and $R:=(2\frac{a^2}{b} +  c\zeta_t)\zeta_t\varepsilon^2$, we have
    \begin{align*}
        x_{t+1}^2\leq (1-\frac{b\zeta}{2})^t x^2_0 + R\sum_{i=0}^{t-1}(1-\frac{b\zeta}{2})^i\leq (1-\frac{b\zeta}{2})^t x^2_0 + \frac{2R}{b\zeta}
    \end{align*}
    For $b\zeta=1$, we have $x_{t}$ is bounded by $\sqrt{2R}$.
\end{proof}

\textbf{Convergence of $E_t$:}
To this end, we first prove the following helper lemmas. 
\begin{lemma}[Bounding $\|Z_{t+1}-Z_t\|$]\label{lem:Zstep2}
Let $c_D:=\zeta^2B_Z$ and $c_E:=2\zeta^2A'_Z$. Then
\begin{equation}\label{eq:Zstep2}
\|Z_{t+1}-Z_t\|_F^2 \ \le\ c_D D_t + c_E E_t + c_E\,\chi^2(Z_t).
\end{equation}
\end{lemma}
Proof is in Appendix \ref{lem:Zstep2_p}.
\begin{lemma}\label{lem:phi-PG}
If $0<\eta\le 2\alpha/L_{\phi,\alpha}(\alpha+L_{\phi,\alpha})$, then 
\begin{equation}\label{eq:phi-contraction}
\|\phi_{t+1}-\phi^\sharp(Z_{t+1})\|_\mathcal{H}^2
\ \le\ q_\phi\,\|\phi_t-\phi^\sharp(Z_{t+1})\|_\mathcal{H}^2,
\end{equation}
where $q_\phi:=1-\frac{\eta\alpha L_{\phi,\alpha}}{\alpha+L_{\phi,\alpha}}\in(0,1)$.
\end{lemma}
Proof is in Appendix \ref{lem:phi-PG_p}.
\begin{lemma}\label{lem:phi-sharp-Lip}
Recall that $\phi^\sharp(Z)\ \in\ \arg\min_{\phi\in\mathcal K} \tilde{\mathcal{L}}(\phi,Z)$, then for all $Z,Z'$,
\begin{equation}\label{eq:phi-sharp-Lip-constrained}
\|\phi^\sharp(Z')-\phi^\sharp(Z)\|_\mathcal{H}
\ \le\ \frac{L_{Z\to\phi}}{\alpha}\,\|Z'-Z\|_F.
\end{equation}
\end{lemma}
Proof is in Appendix \ref{lem:phi-sharp-Lip_p}.
\begin{proposition}\label{prop:E-rec}
For any $\delta>0$, letting $A:=(L_{Z\to\phi}/\alpha)^2$,
\begin{equation}\label{eq:E-rec}
E_{t+1}\ \le\ q_\phi(1+\delta)E_t\ +\ q_\phi\Big(1+\tfrac{1}{\delta}\Big)A\,\|Z_{t+1}-Z_t\|_F^2.
\end{equation}
\end{proposition}
Proof is in Appendix \ref{prop:E-rec_p}.
\begin{proposition}\label{lem:D-rec}
If $0<\zeta\le \mu_Z/(2B_Z)$, then 
\begin{equation}\label{eq:D-rec}
D_{t+1}\ \le\ \rho_Z\,D_t\ +\ C_{Z\leftarrow\phi}\,E_t\ +\ C_{Z\leftarrow\phi}\,\chi^2(Z_t),
\end{equation}
where $\rho_Z:=\ 1-\tfrac{\zeta\mu_Z}{2}\in(0,1)$, $C_{Z\leftarrow\phi}:=2\zeta^2A'_Z+2\frac{\zeta A_Z^2}{\mu_Z}$.
\end{proposition}
Proof is in Appendix \ref{lem:D-rec_p}.
We are now ready to present the prove for the convergence of $E_t$.
From \eqref{eq:diff_Z}, we have
\[
\langle \nabla_Z \tilde{\mathcal{L}}(\phi_t,Z_t), Z_{t+1}-Z_t\rangle\leq-\frac{1}{\zeta}\|Z_{t+1}-Z_{t}\|^2,
\]
By applying the smoothness of $\tilde{\mathcal L}$, Lemma \ref{lem:Z-smoothness}, we obtain
\begin{align*}
\tilde{\mathcal L}(Z_t,\phi_t)-\tilde{\mathcal L}(Z_{t+1},\phi_t)&\geq  -\langle \nabla_Z \tilde{\mathcal{L}}(\phi_t,Z_t), Z_{t+1}-Z_t\rangle-\frac{1}{2L_{Z}}\|Z_{t}-Z_{t+1}\|^2\\
&\geq(\frac{1}{\zeta}-\frac{1}{2L_{Z}})\|Z_{t}-Z_{t+1}\|^2\geq\frac{1}{2\zeta}\|Z_{t}-Z_{t+1}\|^2.    
\end{align*}
Similarly, we have
\begin{align*}
    \tilde{\mathcal L}(Z_{t+1},\phi_t)-\tilde{\mathcal L}(Z_{t+1},\phi_{t+1})\geq\frac{1}{2\eta}\|\phi_{t}-\phi_{t+1}\|_\mathcal{H}^2.
\end{align*}
By summing these two inequalities over $t$, we have
\[
\tilde{\mathcal L}(Z_0,\phi_0)-\tilde{\mathcal L}(Z_{T+1},\phi_{T+1})\geq\sum_{t=0}^{T}[\frac{1}{2\zeta}\|Z_{t}-Z_{t+1}\|^2+\frac{1}{2\eta}\|\phi_{t}-\phi_{t+1}\|_\mathcal{H}^2]\geq\sum_{t=0}^{T}\frac{1}{2\zeta}\|Z_{t}-Z_{t+1}\|^2
\]
As a result, $\lim_{t\to\infty}\|Z_{t+1}-Z_{t}\|=0$. Then, from Proposition \ref{prop:E-rec}, we have $\lim_{t\to\infty}E_{t}=0$.

\qed

\subsection{Proof of Lemma \ref{lem:Zstep2}}\label{lem:Zstep2_p}
\begin{proof}
By first-order optimality of $Z_{t+1}=\mathcal P_{\mathcal C}(Z_t-\zeta \nabla_Z \tilde{\mathcal{L}}(\phi_t,Z_t))$,
\[
\langle Z_t-Z_{t+1}-\zeta \nabla_Z \tilde{\mathcal{L}}(\phi_t,Z_t),\ Z_{t+1}-Y\rangle\ge 0,\quad \forall\,Y\in\mathcal C.
\]
When $Y=Z_t$, we get
\begin{equation}\label{eq:diff_Z}
    \|Z_{t+1}-Z_t\|_F^2 \le \zeta\langle \nabla_Z \tilde{\mathcal{L}}(\phi_t,Z_t), Z_t-Z_{t+1}\rangle \le \zeta \|\nabla_Z \tilde{\mathcal{L}}(\phi_t,Z_t)\|_F \|Z_{t+1}-Z_t\|_F,
\end{equation}
so either $\|Z_{t+1}-Z_t\|_F=0$ (trivial) or
\[
\|Z_{t+1}-Z_t\|_F \le \zeta \|\nabla_Z \tilde{\mathcal{L}}(\phi_t,Z_t)\|_F \ \Rightarrow\ \|Z_{t+1}-Z_t\|_F^2 \le \zeta^2 \|\nabla_Z \tilde{\mathcal{L}}(\phi_t,Z_t)\|_F^2.
\]
By \eqref{eq:PL2}, i.e., $\|\nabla_Z \tilde{\mathcal{L}}(\phi_t,Z_t)\|_F^2\le B_Z D_t + A'_Z \|\phi_t-\phi^\star\|_\mathcal{H}^2$. 
Decompose $\|\phi_t-\phi^\star\|_\mathcal{H}\le \|\phi_t-\phi^\sharp(Z_t)\|_\mathcal{H}+\|\phi^\sharp(Z_t)-\phi^\star\|_\mathcal{H}\le \sqrt{E_t}+\chi$, so
$\|\phi_t-\phi^\star\|_\mathcal{H}^2\le 2E_t+2\chi^2$. 
Now, we get
\[
\|Z_{t+1}-Z_t\|_F^2 \le \zeta^2 B_Z D_t + \zeta^2A'_Z(2E_t+2\chi^2),
\]
which is \eqref{eq:Zstep2}.
\end{proof}

\subsection{Proof of Lemma \ref{lem:phi-PG}}\label{lem:phi-PG_p}
\begin{proof}
Fix $t$ and define
\[
g(\phi) := \tilde{\mathcal{L}}(\phi, Z_{t+1}).
\]
Let the feasible set be the cloIf sed convex set $K := \mathcal H_{\xi,\Xi}$ and let $\mathcal P_K$ denote the projection onto $K$. The projected gradient update reads
\begin{equation}
\phi_{t+1} \;=\; \mathcal P_K\!\big(\phi_t-\eta \nabla g(\phi_t)\big).
\label{eq:PG}
\end{equation}
Let $\phi_\sharp := \phi_\sharp(Z_{t+1}) \in \arg\min_{\phi\in K} g(\phi)$ be the constrained minimizer.
Since $\phi_\sharp$ minimizes $g$ over $K$, the first-order optimality condition is
\begin{equation}
0 \in \nabla g(\phi_\sharp) + N_K(\phi_\sharp),
\label{eq:KKT}
\end{equation}
where $N_K(\phi_\sharp)$ denotes the normal cone of $K$ at $\phi_\sharp$.
We claim that \eqref{eq:KKT} implies the fixed-point identity
\begin{equation}
\phi_\sharp \;=\; \mathcal P_K\!\big(\phi_\sharp-\eta \nabla g(\phi_\sharp)\big)
\qquad\text{for any }\eta>0.
\label{eq:FP}
\end{equation}
To prove \eqref{eq:FP}, recall the characterization of projection onto a closed convex set:
for any $y\in \mathcal H$ and any $x\in K$,
\begin{equation}
x=\mathcal P_K(y)
\quad\Longleftrightarrow\quad
\langle y-x,\; u-x\rangle_\mathcal{H} \le 0,\ \ \forall u\in K.
\label{eq:proj-char}
\end{equation}
Set $y=\phi_\sharp-\eta\nabla g(\phi_\sharp)$ and $x=\phi_\sharp$.
Then $y-x=-\eta\nabla g(\phi_\sharp)$, and \eqref{eq:proj-char} becomes
\[
\langle -\eta\nabla g(\phi_\sharp),\; u-\phi_\sharp\rangle_\mathcal{H} \le 0,\ \ \forall u\in K,
\]
equivalently,
\begin{equation}
\langle \nabla g(\phi_\sharp),\; u-\phi_\sharp\rangle_\mathcal{H} \ge 0,\ \ \forall u\in K.
\label{eq:VI}
\end{equation}
Condition \eqref{eq:VI} is equivalent to $-\nabla g(\phi_\sharp)\in N_K(\phi_\sharp)$, i.e.,
\eqref{eq:KKT}. Hence \eqref{eq:FP} holds, and in particular we do \emph{not} require
$\nabla g(\phi_\sharp)=0$ (therefore boundary minimizers are covered).
Combining \eqref{eq:PG} and \eqref{eq:FP},
\[
\phi_{t+1}-\phi_\sharp
=
\mathcal P_K(\phi_t-\eta\nabla g(\phi_t))-\mathcal P_K(\phi_\sharp-\eta\nabla g(\phi_\sharp)).
\]
By non-expansiveness of $\mathcal P_K$,
\begin{equation}
\|\phi_{t+1}-\phi_\sharp\|_\mathcal{H}
\le
\|(\phi_t-\eta\nabla g(\phi_t))-(\phi_\sharp-\eta\nabla g(\phi_\sharp))\|_\mathcal{H}.
\label{eq:nonexp}
\end{equation}
Squaring \eqref{eq:nonexp} and expanding yields
\begin{align}
\|\phi_{t+1}-\phi_\sharp\|_\mathcal{H}^2
&\le
\|(\phi_t-\phi_\sharp)-\eta(\nabla g(\phi_t)-\nabla g(\phi_\sharp))\|_\mathcal{H}^2 \notag\\
&=
\|\phi_t-\phi_\sharp\|_\mathcal{H}^2
-2\eta\left\langle \nabla g(\phi_t)-\nabla g(\phi_\sharp),\; \phi_t-\phi_\sharp\right\rangle_\mathcal{H}
+\eta^2\|\nabla g(\phi_t)-\nabla g(\phi_\sharp)\|_\mathcal{H}^2 .
\label{eq:expand-no-delta}
\end{align}
Since $g$ is $\alpha$-strongly convex, Lemma \ref{lem:strong_convexity}, $\nabla g$ is $\alpha$-strongly monotone:
\begin{equation}
\left\langle \nabla g(\phi_t)-\nabla g(\phi_\sharp),\; \phi_t-\phi_\sharp\right\rangle_\mathcal{H}
\ge
\alpha\|\phi_t-\phi_\sharp\|_\mathcal{H}^2.
\label{eq:strongmono}
\end{equation}
Since $g$ is $L_{\phi,\alpha}$-smooth and convex, Lemma \ref{lem:smoothness_phi}, $\nabla g$ is $1/L_{\phi,\alpha}$-cocoercive:
\begin{equation}
\left\langle \nabla g(\phi_t)-\nabla g(\phi_\sharp),\; \phi_t-\phi_\sharp\right\rangle_\mathcal{H}
\ge
\frac{1}{L_{\phi,\alpha}}\|\nabla g(\phi_t)-\nabla g(\phi_\sharp)\|_\mathcal{H}^2.
\label{eq:cocoercive}
\end{equation}
Let
\[
\theta := \frac{L_{\phi,\alpha}}{\alpha+L_{\phi,\alpha}},
\qquad
1-\theta := \frac{\alpha}{\alpha+L_{\phi,\alpha}}.
\]
Multiplying \eqref{eq:strongmono} by $\theta$ and \eqref{eq:cocoercive} by $(1-\theta)$, and adding, we obtain 
\begin{align}\label{eq:combo}
\left\langle \nabla g(\phi_t)-\nabla g(\phi_\sharp),\; \phi_t-\phi_\sharp\right\rangle_\mathcal{H}
&\ge
\frac{\alpha L_{\phi,\alpha}}{\alpha+L_{\phi,\alpha}}\|\phi_t-\phi_\sharp\|_\mathcal{H}^2
\\ \notag
&+
\frac{\alpha}{L_{\phi,\alpha}(\alpha+L_{\phi,\alpha})}\|\nabla g(\phi_t)-\nabla g(\phi_\sharp)\|_\mathcal{H}^2.
\end{align}
Substituting \eqref{eq:combo} into \eqref{eq:expand-no-delta} yields

\begin{align}\notag
\|\phi_{t+1}-\phi_\sharp\|_\mathcal{H}^2
&\le
\|\phi_t-\phi_\sharp\|_\mathcal{H}^2\\ \notag
&-2\eta\left(
\frac{\alpha L_{\phi,\alpha}}{\alpha+L_{\phi,\alpha}}\|\phi_t-\phi_\sharp\|_\mathcal{H}^2
+
\frac{\alpha}{L_{\phi,\alpha}(\alpha+L_{\phi,\alpha})}\|\nabla g(\phi_t)-\nabla g(\phi_\sharp)\|_\mathcal{H}^2
\right)
\\ \notag
&+\eta^2\|\nabla g(\phi_t)-\nabla g(\phi_\sharp)\|_\mathcal{H}^2\\ \label{eq:precontract}
&=\left(1-\frac{2\eta\alpha L_{\phi,\alpha}}{\alpha+L_{\phi,\alpha}}\right)\|\phi_t-\phi_\sharp\|_\mathcal{H}^2
\\ \notag
&+\left(\eta^2-\frac{2\eta\alpha}{L_{\phi,\alpha}(\alpha+L_{\phi,\alpha})}\right)\|\nabla g(\phi_t)-\nabla g(\phi_\sharp)\|_\mathcal{H}^2.
\end{align}
Assume $0<\eta\le \frac{2\alpha}{L_{\phi,\alpha}(\alpha+L_{\phi,\alpha})}$. Then
\[
\eta^2-\frac{2\eta\alpha}{L_{\phi,\alpha}(\alpha+L_{\phi,\alpha})} \le0,
\]
hence the second term in \eqref{eq:precontract} is non-positive and can be dropped:
\begin{equation}
\|\phi_{t+1}-\phi_\sharp\|_\mathcal{H}^2
\le
\left(1-\frac{2\eta\alpha L_{\phi,\alpha}}{\alpha+L_{\phi,\alpha}}\right)\|\phi_t-\phi_\sharp\|_\mathcal{H}^2 .
\label{eq:contract-2}
\end{equation}
Finally, since
\[
1-\frac{2\eta\alpha L_{\phi,\alpha}}{\alpha+L_{\phi,\alpha}}
\le
1-\frac{\eta\alpha L_{\phi,\alpha}}{\alpha+L_{\phi,\alpha}}
=: q_\phi,
\]
we obtain
\[
\|\phi_{t+1}-\phi_\sharp(Z_{t+1})\|_\mathcal{H}^2
\le
q_\phi \|\phi_t-\phi_\sharp(Z_{t+1})\|_\mathcal{H}^2,
\]
which is exactly \eqref{eq:phi-contraction} in Lemma~7. This completes the proof.
\end{proof}

\subsection{Proof of Lemma \ref{lem:phi-sharp-Lip}}\label{lem:phi-sharp-Lip_p}
\begin{proof}
Write the variational optimality conditions using the normal cone $N_{\mathcal K}(\cdot)$:
\[
0\ \in\ \nabla_\phi \tilde{\mathcal{L}}(\phi^\sharp(Z),Z)+N_{\mathcal K}(\phi^\sharp(Z)),
\qquad
0\ \in\ \nabla_\phi \tilde{\mathcal{L}}(\phi^\sharp(Z'),Z')+N_{\mathcal K}(\phi^\sharp(Z')).
\]
Thus there exist $v\in N_{\mathcal K}(\phi^\sharp(Z))$ and $v'\in N_{\mathcal K}(\phi^\sharp(Z'))$ such that
\[
\nabla_\phi \tilde{\mathcal{L}}(\phi^\sharp(Z),Z)+v=0,\qquad
\nabla_\phi \tilde{\mathcal{L}}(\phi^\sharp(Z'),Z')+v'=0.
\]
Subtract the two relations and add–subtract $\nabla_\phi \tilde{\mathcal{L}}(\phi^\sharp(Z'),Z)$:
\[
\underbrace{\big(\nabla_\phi \tilde{\mathcal{L}}(\phi^\sharp(Z'),Z)-\nabla_\phi \tilde{\mathcal{L}}(\phi^\sharp(Z),Z)\big)}_{=:~\mathrm{(A)}}
\ +\ 
\underbrace{\big(\nabla_\phi \tilde{\mathcal{L}}(\phi^\sharp(Z'),Z')-\nabla_\phi \tilde{\mathcal{L}}(\phi^\sharp(Z'),Z)\big)}_{=:~\mathrm{(B)}}
\ +\ (v'-v)
\ =\ 0.
\]
Take inner product with $d:=\phi^\sharp(Z')-\phi^\sharp(Z)$. By strong convexity in $\phi$ at any fixed $Z$, we get
\[
\langle \mathrm{(A)},d\rangle_\mathcal{H}\ \ge\ \alpha\|d\|_\mathcal{H}^2.
\]
By monotonicity of the normal cone,
\[
\langle v'-v,\, d\rangle_\mathcal{H}\ \ge\ 0.
\]
Hence,
\[
\alpha\|d\|_\mathcal{H}^2
\ \le\ -\,\langle \mathrm{(B)},d\rangle_\mathcal{H}
\ \le\ \|\mathrm{(B)}\|_\mathcal{H}\,\|d\|_\mathcal{H}.
\]
From the Lemma \ref{lem:smoothness_z_to_phi} at the point $\phi=\phi^\sharp(Z')$,
\[
\|\mathrm{(B)}\|_\mathcal{H}
=\big\|\nabla_\phi \tilde{\mathcal{L}}(\phi^\sharp(Z'),Z')-\nabla_\phi \tilde{\mathcal{L}}(\phi^\sharp(Z'),Z)\big\|_\mathcal{H}
\ \le\ L_{Z\to\phi}\,\|Z'-Z\|_F.
\]
If $\|d\|_\mathcal{H}=0$ we are done; otherwise divide both sides by $\|d\|_\mathcal{H}$ to obtain
\[
\|d\|_\mathcal{H}\ \le\ \frac{L_{Z\to\phi}}{\alpha}\,\|Z'-Z\|_F,
\]
which is \eqref{eq:phi-sharp-Lip-constrained}.
\end{proof}

\subsection{Proof of Lemma \ref{prop:E-rec}}\label{prop:E-rec_p}
\begin{proof}
From Lemma~\ref{lem:phi-PG},
\[
\|\phi_{t+1}-\phi^\sharp(Z_{t+1})\|_\mathcal{H}^2 \le q_\phi\,\|\phi_t-\phi^\sharp(Z_{t+1})\|_\mathcal{H}^2.
\]
Using $(a+b)^2\le (1+\delta)a^2+(1+1/\delta)b^2$ with 
$a=\sqrt{E_t}=\|\phi_t-\phi^\sharp(Z_t)\|_\mathcal{H}$ and $b=\|\phi^\sharp(Z_{t+1})-\phi^\sharp(Z_t)\|_\mathcal{H}$, then applying Lemma~\ref{lem:phi-sharp-Lip},
\[
\|\phi_t-\phi^\sharp(Z_{t+1})\|_\mathcal{H}^2
\le (1+\delta)E_t+(1+1/\delta)\Big(\tfrac{L_{Z\to\phi}}{\alpha}\Big)^2\|Z_{t+1}-Z_t\|_F^2,
\]
which yields \eqref{eq:E-rec}.
\end{proof}

\subsection{Proof of Lemma \ref{lem:D-rec}}\label{lem:D-rec_p}
\begin{proof}
By projection non-expansiveness and $\Phi_t\in\mathcal M$,
\begin{align*}
\|\Delta_{t+1}\|_F
&=\|Z_{t+1}-\Phi_{t+1}\|_F
\le \|Z_{t+1}-\Phi_t\|_F
\le \|Z_t-\zeta \nabla_Z \tilde{\mathcal{L}}(\phi_t,Z_t)-\Phi_t\|_F
\\
&=\|\Delta_t-\zeta \nabla_Z \tilde{\mathcal{L}}(\phi_t,Z_t)\|_F.
\end{align*}
Square and expand:
\[
\|\Delta_{t+1}\|_F^2\le \|\Delta_t\|_F^2+\zeta^2\|\nabla_Z  \tilde{\mathcal{L}}(\phi_t,Z_t)\|_F^2-2\zeta\langle \nabla_Z  \tilde{\mathcal{L}}(\phi_t,Z_t),\Delta_t\rangle.
\]
Use \eqref{eq:PL1}--\eqref{eq:PL2} and Young's inequality $uv\le \frac{\varepsilon}{2}u^2+\frac{1}{2\varepsilon}v^2$ with $\varepsilon=\mu_Z/A_Z$:
\begin{align*}
\langle \nabla_Z \tilde{\mathcal{L}}(\phi_t,Z_t),\Delta_t\rangle &\ge \mu_Z\|\Delta_t\|_F^2- A_Z\|\Delta_t\|_F\,\|\phi_t-\phi^\star\|_\mathcal{H},\\
\|\nabla_Z \tilde{\mathcal{L}}(\phi_t,Z_t)\|_F^2 &\le B_Z\|\Delta_t\|_F^2+A'_Z\|\phi_t-\phi^\star\|_\mathcal{H}^2,\\
2\zeta A_Z\|\Delta_t\|_F\,\|\phi_t-\phi^\star\|_\mathcal{H}
&\le \zeta\mu_Z\|\Delta_t\|_F^2+\frac{\zeta A_Z^2}{\mu_Z}\|\phi_t-\phi^\star\|_\mathcal{H}^2.
\end{align*}
Therefore, we get
\[
D_{t+1}\ \le\ \big(1-\zeta\mu_Z+\zeta^2B_Z\big)D_t
\ +\ \Big(\zeta^2A'_Z+\frac{\zeta A_Z^2}{\mu_Z}\Big)\,\|\phi_t-\phi^\star\|_\mathcal{H}^2.
\]
Decompose $\|\phi_t-\phi^\star\|_\mathcal{H}\le \|\phi_t-\phi^\sharp(Z_t)\|_\mathcal{H}+\|\phi^\sharp(Z_t)-\phi^\star\|_\mathcal{H}\le \sqrt{E_t}+\chi$, so
$\|\phi_t-\phi^\star\|_\mathcal{H}^2\le 2E_t+2\chi^2$, yielding \eqref{eq:D-rec}.
\end{proof}

\subsection{Proof of Theorem \ref{thm:main}}\label{pr:lem:one-step-re}

\begin{table}
\centering
\renewcommand{\arraystretch}{1.3}
\begin{tabular}{ll}
\hline
\textbf{Symbol} & \textbf{Definition} \\
\hline
$q_\phi$ & $1-\dfrac{\eta\alpha L_{\phi,\alpha}}{\alpha+L_{\phi,\alpha}}$ \\[4pt]
$\rho_Z$ & $1-{\zeta\mu_Z}/{2}$ \\[4pt]
$A$ & $(L_{Z\to\phi}/\alpha)^2$ \\[4pt]
$c_D$ & $\zeta^2B_Z$ \\[4pt]
$c_E$ & $2\zeta^2A'_Z$ \\[4pt]
$C_{Z\leftarrow\phi}$ & $2\zeta^2A'_Z+2\frac{\zeta A_Z^2}{\mu_Z}$ \\[4pt]
$\delta$ & ${(1/q_\phi - 1)}/{2}$ \\[4pt]
$a_{EE}$ & $q_\phi(1+\delta)+q_\phi\Big(1+\tfrac1\delta\Big)A c_E\ +\ \gamma\,C_{Z\leftarrow\phi}$ \\[4pt]
$a_{DD}$ & $q_\phi\Big(1+\tfrac1\delta\Big)A c_D + \gamma\,\rho_Z$ \\[6pt]
$\rho$ & $\max\{a_{EE},\,a_{DD}/\gamma\}$ \\[4pt]
$C_\phi$ & $q_\phi(1+1/\delta)A c_E + 2\gamma C_{Z\leftarrow\phi}$ \\
 $C'_\phi$ & ${5\eta\alpha L_{\phi,\alpha}}/(8(\alpha+L_{\phi,\alpha}))$\\
\hline
\end{tabular}
\caption{Notations and their definitions.}\label{table1}
\end{table}

\begin{proof}
From Proposition~\ref{prop:E-rec} and Lemma~\ref{lem:Zstep2},
\[
E_{t+1}\ \le\ q_\phi(1+\delta)E_t + q_\phi\Big(1+\tfrac{1}{\delta}\Big)A\,(c_D D_t + c_E E_t + c_E \chi^2(Z_t)).
\]
Proposition~\ref{lem:D-rec} gives us
\[
D_{t+1} \ \le\ \rho_Z D_t + C_{Z\leftarrow\phi}\, E_t + \,C_{Z\leftarrow\phi}\,\chi^2(Z_t).
\]
Multiplying the second inequality by $\gamma$ and adding to the first, we obtain
\begin{equation}
\begin{aligned}
\mathcal V_{t+1}
&=E_{t+1}+\gamma D_{t+1}\\
&\le \underbrace{\Big[q_\phi(1+\delta)+q_\phi\Big(1+\tfrac1\delta\Big)A c_E\ +\ \gamma\,C_{Z\leftarrow\phi}\Big]}_{a_{EE}}\,E_t\\
&\quad + \underbrace{\Big[q_\phi\Big(1+\tfrac1\delta\Big)A c_D + \gamma\,\rho_Z\Big]}_{a_{DD}}\,D_t\\
&\quad + \underbrace{\Big[q_\phi\Big(1+\tfrac1\delta\Big)A c_E + \gamma C_{Z\leftarrow\phi}\Big]}_{C_{\phi}}\,\chi^2(Z_t).
\end{aligned}
\end{equation}
We require both $a_{EE}<1$ and $a_{DD}/\gamma<1$.  Refer to Table \ref{table1} for the definitions. Below, we provide explicit recipe for having such $a_{EE}$ and $a_{DD}$.

We had $0<\eta\le \frac{2\alpha}{L_{\phi,\alpha}(\alpha+L_{\phi,\alpha})}$ and in order to have
$
q_\phi=1-\frac{\eta\alpha L_{\phi,\alpha}}{\alpha+L_{\phi,\alpha}}\in(0,1),
$
 we should select  
 $$
 \eta\leq\min\big\{\frac{\alpha+L_{\phi,\alpha}}{2\alpha L_{\phi,\alpha}}, \frac{2\alpha}{L_{\phi,\alpha}(\alpha+L_{\phi,\alpha})}\big\}=\frac{2\alpha}{L_{\phi,\alpha}(\alpha+L_{\phi,\alpha})}
.
 $$
 The above equality holds since $L_{\phi,\alpha}\geq \alpha$ as it is given in Lemma \ref{lem:smoothness_phi}.  
 which is guaranteed by choosing $\eta=\frac{C_1\alpha}{\alpha+C_2}$.

Let
\begin{equation}
\delta\ :=\ \frac{1/q_\phi -1}{2}\ (>0)
\quad\Rightarrow\quad
q_\phi(1+\delta)=\frac{1+q_\phi}{2}\ (<1),\qquad
q_\phi\Big(1+\frac{1}{\delta}\Big)=\frac{q_\phi(1+q_\phi)}{1-q_\phi}.
\end{equation}

Recall that
\begin{align*}
a_{EE}&=q_\phi(1+\delta)+q_\phi\Big(1+\tfrac1\delta\Big)A c_E\ +\ \gamma\,C_{Z\leftarrow\phi}\\
&=\frac{1+q_\phi}{2}+\frac{q_\phi(1+q_\phi)}{1-q_\phi}\Big(\frac{L_{Z\to\phi}}{\alpha}\Big)^2(2\zeta^2A_Z')\ +\ \gamma \Big(2\zeta^2A'_Z+\frac{2\zeta }{\mu_Z}A^2_Z\Big)\\
&\leq1-\frac{\eta\alpha L_{\phi,\alpha}}{2(\alpha+L_{\phi,\alpha})}+\frac{\alpha+L_{\phi,\alpha}}{\eta\alpha L_{\phi,\alpha}}\Big(\frac{L_{Z\to\phi}}{\alpha}\Big)^2(2\zeta^2A_Z')+\ \gamma \,\Big(2\zeta^2A_Z'+\frac{2\zeta }{\mu_Z}A_Z^2\Big).
\end{align*}
The last inequality is using the fact that $q_\phi\leq1/2$ and thus $q_\phi(1+q_\phi)\leq1$. 
By setting $\gamma \Big(2\zeta^2A_Z'+\frac{2\zeta }{\mu_Z}A_Z^2\Big)\leq\frac{\eta\alpha L_{\phi,\alpha}}{4(\alpha+L_{\phi,\alpha})}$ and $\zeta\leq\frac{\eta\alpha^2L_{\phi,\alpha}}{4\sqrt{A'_Z}L_{Z\to\phi}(\alpha+L_{\phi,\alpha})}$, we have
\begin{align*}
   a_{EE}&\leq 1-\frac{\eta\alpha L_{\phi,\alpha}}{4(\alpha+L_{\phi,\alpha})}+\frac{\alpha+L_{\phi,\alpha}}{\eta\alpha L_{\phi,\alpha}}\Big(\frac{L_{Z\to\phi}}{\alpha}\Big)^2\left(2\frac{\eta^2\alpha^4L_{\phi,\alpha}^2}{16L_{Z\to\phi}^2A'_Z(\alpha+L_{\phi,\alpha})^2}A'_Z\right)\\
   &=1-\frac{\eta\alpha L_{\phi,\alpha}}{8(\alpha+L_{\phi,\alpha})}<1.
\end{align*}

In order to have $\gamma \Big(2\zeta^2A_Z'+\frac{2\zeta }{\mu_Z}A_Z^2\Big)\leq\frac{\eta\alpha L_{\phi,\alpha}}{4(\alpha+L_{\phi,\alpha})}$ and $\zeta\leq\min\{\frac{\eta\alpha^2L_{\phi,\alpha}}{4\sqrt{A'_Z}L_{Z\to\phi}(\alpha+L_{\phi,\alpha})}, \frac{1}{2\mu_Z}\}$, we only require
\begin{align*}
    \gamma\leq \frac{\eta\alpha L_{\phi,\alpha}\mu_Z}{16(\alpha+L_{\phi,\alpha})\zeta \max\{A'_Z,A_Z^2\}}\leq \frac{\frac{\eta\alpha L_{\phi,\alpha}}{4(\alpha+L_{\phi,\alpha})}}{2\zeta^2A_Z'+\frac{2\zeta A_Z^2}{\mu_Z}}.
\end{align*}
Therefore, by choosing 
$$
\gamma= \frac{ L_{Z\to\phi}\mu_Z}{4\alpha \max\{\sqrt{A'_Z},A_Z\}
,}
$$
we ensure that 
$$
\gamma\leq \frac{\eta\alpha L_{\phi,\alpha}\mu_Z}{16(\alpha+L_{\phi,\alpha})\frac{\eta\alpha^2L_{\phi,\alpha}}{4\sqrt{A'_Z}L_{Z\to\phi}(\alpha+L_{\phi,\alpha})}\max\{A'_Z,A_Z^2\}}
$$ 
given that $A_Z\leq\sqrt{A_Z'}$ and consequently, we get $a_{EE}<1$. On the other hand, according to the definition of $a_{DD}$, we have
\[
\frac{a_{DD}}{\gamma}=\rho_Z+\frac{q_\phi(1+q_\phi)}{1-q_\phi}\cdot \frac{A\,c_D}{\gamma}
=1-\tfrac{\zeta\mu_Z}{2}+\frac{q_\phi(1+q_\phi)}{1-q_\phi}\cdot \frac{\zeta^2AB_Z}{\gamma}.
\]
To ensure $\frac{a_{DD}}{\gamma}\leq 1-\tfrac{\zeta\mu_Z}{4}$, it suffices to have
\begin{equation*}
\frac{q_\phi(1+q_\phi)}{1-q_\phi}\cdot \frac{\zeta^2AB_Z}{\gamma}\leq \frac{\zeta\mu_Z}{4}
\end{equation*}
or 
\begin{align*}
    \zeta\leq \frac{\mu_Z(1-q_\phi)\gamma}{4AB_Z} < \frac{\mu_Z(1-q_\phi)\gamma}{4q_\phi(1+q_\phi)AB_Z}
\end{align*}
While $A=(L_{Z\to\phi}/\alpha)^2$ and $\gamma=L_{Z\to\phi}\mu_Z/(4\alpha \max\{\sqrt{A'_Z},A_Z\})$, we obtain
\begin{align*}
\frac{\mu_Z(1-q_\phi)\,\gamma}{4\,A\,B_Z}
&=
\frac{\eta\,\alpha^2\,\mu_Z^2\,L_{\phi,\alpha}}{16\,\max\{\sqrt{A'_Z},A_Z\}\,L_{Z\to\phi}\,(\alpha+L_{\phi,\alpha})\,B_Z}.
\end{align*}
Therefore, by selecting
\[
\zeta \;\le\;
\frac{\eta\,\alpha^2\,\mu_Z^2\,L_{\phi,\alpha}}{16\,\max\{\sqrt{A'_Z},A_Z\}\,L_{Z\to\phi}\,(\alpha+L_{\phi,\alpha})\,B_Z}
,
\]
we ensure that $a_{DD}/\gamma<1$. The aforementioned conditions are satisfied for $\zeta$ when 
$$
\zeta\in\mathcal{O}\Big(\frac{\alpha^3}{(\alpha+L_{\phi,\alpha})^2(\mu r\kappa)^2}\big(\frac{\xi}{\Xi}\big)^5\frac{\sqrt{nT}}{B_K}\Big).
$$
Regarding the last term, i.e., $q_\phi\Big(1+\tfrac1\delta\Big)A c_E + 2\gamma C_{Z\leftarrow\phi}$, we have
\begin{align*}
    C_\phi&=\Big[q_\phi\Big(1+\tfrac1\delta\Big)A c_E + 2\gamma C_{Z\leftarrow\phi}\Big]\\
    &=\Big[\frac{\alpha+L_{\phi,\alpha}}{\eta\alpha L_{\phi,\alpha}}\Big(\frac{L_{Z\to\phi}}{\alpha}\Big)^2(2\zeta^2A_Z')+\ 4\gamma\,\Big(\zeta^2A_Z'+\frac{\zeta A_Z}{\mu_Z}\Big)\Big]\\
    &\leq\Big[\frac{\eta\alpha L_{\phi,\alpha}}{8(\alpha+L_{\phi,\alpha})}+\frac{\eta\alpha L_{\phi,\alpha}}{2(\alpha+L_{\phi,\alpha})}\Big]=\frac{5\eta\alpha L_{\phi,\alpha}}{8(\alpha+L_{\phi,\alpha})}:=C'_\phi.
\end{align*}
Overall, $a_{EE}\leq1-\frac{\eta\alpha L_{\phi,\alpha}}{8(\alpha+L_{\phi,\alpha})}$ and $\frac{a_{DD}}{\gamma}\leq 1-\frac{\zeta\mu_Z}{4}$, we have
\[
\mathcal{V}_{t+1}\leq \rho \mathcal{V}_t+C'_\phi\chi^2(Z_t), 
\]
where $\rho=\max\{1-\frac{\eta\alpha L_{\phi,\alpha}}{8(\alpha+L_{\phi,\alpha})},1-\frac{\zeta\mu_Z}{4}\}<1$. By telescoping over $t$, we can get the result.

\end{proof}

\subsection{Proof of Theorem \ref{thm:phi_error}}\label{thm:phi_error_proof}

\begin{proof}
The representer theorem implies that for any fixed $Z$, the stationary $\phi$ of the regularized likelihood function has the following form
\begin{equation*}
\phi(\cdot) = \sum_{j=1}^M \beta_j\, K(x_j, \cdot),
\qquad
x_j := \langle A_j, ZZ^\top\rangle.
\end{equation*}
Plugging into stationary condition and equating the coefficients of $K(x_k,\cdot)$ on both sides yields, 
\begin{equation}\label{eq:st_1}
\alpha \beta_k = 2 \bigl( y_k - (K\beta)_k \bigr),\quad k=1,\dots,M, 
\end{equation}
where $K \in \mathbb{R}^{M\times M}$ is the Gram matrix $K_{kj} := K(x_j,x_k)$. Thus
\begin{equation*}
(\alpha I + 2K)\,\beta = 2 y
\quad\Longrightarrow\quad
\beta = (\alpha I + 2K)^{-1} (2y).
\end{equation*}
We define the residual vector $e \in \mathbb{R}^M$ at the stationary point $(Z,\phi)$ for a given $\alpha$ with the following entries
\[
(e_\alpha)_k := y_k - \phi(x_k) = y_k - (K\beta)_k.
\]
Combining this definition with \eqref{eq:st_1} yields
\[
\alpha \beta_k = 2 e_k \quad\Longrightarrow\quad e_k = \frac{\alpha}{2} \beta_k.
\]
Rewriting the above relation in a matrix form gives us
\begin{equation}
e(\alpha) = \frac{\alpha}{2} \beta
= \frac{\alpha}{2} (\alpha I + 2K)^{-1} (2y)
= \alpha (\alpha I + 2K)^{-1} y.
\label{eq:residual-vector}
\end{equation}

As $K$ is symmetric, let $K = U \Lambda U^\top$ be the spectral decomposition, with $\Lambda = \mathrm{diag}(\lambda_1,\dots,\lambda_M)$ and $\lambda_i \ge 0$. Then
\begin{equation*}
e(\alpha) = U \,\mathrm{diag}\Bigl( \frac{\alpha}{\alpha + 2\lambda_i} \Bigr) U^\top y.
\end{equation*}
Note that for any $i$, $\alpha / (\alpha + 2\lambda_i)$ is strictly increasing function of $\alpha$, at $0$ it is $0$, and it tends to $1$ as $\alpha\to\infty$. Therefore,
\begin{equation*}
\|e(\alpha)\|_2^2 = \sum_{i=1}^M \Bigl( \frac{\alpha}{\alpha + 2\lambda_i} \Bigr)^2 (U^\top y)_i^2
\end{equation*} 
is also a non-decreasing function of $\alpha$. Moreover, using the RKHS point-wise bound, we have
$$
|y_k| = |\phi^\star(x_k^\star)|
\le \|\phi^\star\|_{\mathcal{H}} \sqrt{K(x_k^\star,x_k^\star)}
\le \sqrt{B_K}\,\|\phi^\star\|_{\mathcal{H}} \implies \|y\|_2 \le \sqrt{MB_K}\,\|\phi^\star\|_{\mathcal{H}},
$$
and consequently, get
\begin{equation}
\|e(\alpha)\|_2
\le
\frac{\alpha}{\alpha + 2\lambda_{\min}(K)} \, \sqrt{MB_K}\,\|\phi^\star\|_{\mathcal{H}}.
\label{eq:residual-alpha-bound}
\end{equation}
\end{proof}

\subsection{Proof of Theorem \ref{thm:avg_phi_gap_boundary}}\label{thm:avg_phi_gap_boundary_p}

We first present a more detailed version of Theorem \ref{thm:avg_phi_gap_boundary} and prove it. 
\begin{theorem}
\label{thm:avg_phi_gap_boundary_proof}
Let $K := \mathcal H_{\xi,\Xi}\subset \mathcal H$ be the feasible set for $\phi$,
and let $\mathcal C$ be the feasible set for $Z$. Let
\[
G_\phi \;:=\;\sup_{(\phi,Z)\in K\times\mathcal C}\|\nabla_\phi \tilde{\mathcal{L}}(\phi,Z)\|_\mathcal{H} \;<\;\infty,
\]
then for all $T\ge 1$,
\[
\frac{1}{T}\sum_{t=1}^T\Bigl(\tilde{\mathcal{L}}(\phi_t,Z_t)-\min_{\phi\in K}\tilde{\mathcal{L}}(\phi,Z_t)\Bigr)
\;\le\;
G_\phi\,\sqrt{\frac{C_E}{T}},
\]
where
\begin{align}
&C_E:=\frac{a}{1-a}\,E_0
\;+\;
\frac{2\zeta b}{1-a}\tilde{\mathcal{L}}(\phi_0,Z_0).\\ \label{eq:CE_def}
&a:=q_\phi(1+\delta)=(1+q_\phi)/2\in(0,1),\\
&b:=q_\phi\Bigl(1+\frac{1}{\delta}\Bigr)\Bigl(\frac{L_{Z\to\phi}}{\alpha}\Bigr)^2
= \frac{q_\phi(1+q_\phi)}{1-q_\phi}\Bigl(\frac{L_{Z\to\phi}}{\alpha}\Bigr)^2.
\end{align}
\end{theorem}

\begin{proof}
Fix $t\ge 1$ and define the function $g_t:K\to\mathbb R$ by
\[
g_t(\phi):=\tilde{\mathcal{L}}(\phi,Z_t).
\]
By Lemma \ref{lem:strong_convexity}, for each fixed $Z_t$ the map $\phi\mapsto g_t(\phi)$ is $\alpha$-strongly convex on $\mathcal H$,
and therefore is convex on $\mathcal H$; in particular it is convex on the convex set $K$.
Let $\phi_\sharp(Z_t)\in\arg\min_{\phi\in K} g_t(\phi)$.
Since $g_t$ is convex and differentiable, the first-order inequality for convex functions states that
for all $x,y\in K$,
\begin{equation}
g_t(y)\;\ge\; g_t(x) + \langle \nabla g_t(x),\,y-x\rangle_\mathcal{H}.
\label{eq:convex_first_order}
\end{equation}
Apply \eqref{eq:convex_first_order} with $x=\phi_t$ and $y=\phi_\sharp(Z_t)$. Then
\[
g_t(\phi_\sharp(Z_t))
\;\ge\;
g_t(\phi_t) + \langle \nabla g_t(\phi_t),\,\phi_\sharp(Z_t)-\phi_t\rangle_\mathcal{H}.
\]
Rearranging terms yields
\begin{equation}
g_t(\phi_t)-g_t(\phi_\sharp(Z_t))
\;\le\;
\langle \nabla g_t(\phi_t),\,\phi_t-\phi_\sharp(Z_t)\rangle_\mathcal{H}.
\label{eq:gap_le_innerprod}
\end{equation}
Now note that $\nabla g_t(\phi_t)=\nabla_\phi \tilde{\mathcal{L}}(\phi_t,Z_t)$. By Cauchy--Schwarz,
\[
\langle \nabla_\phi \tilde{\mathcal{L}}(\phi_t,Z_t),\,\phi_t-\phi_\sharp(Z_t)\rangle_\mathcal{H}
\;\le\;
\|\nabla_\phi \tilde{\mathcal{L}}(\phi_t,Z_t)\|_\mathcal{H}\cdot \|\phi_t-\phi_\sharp(Z_t)\|_\mathcal{H}.
\]
Using the uniform gradient bound $\|\nabla_\phi \tilde{\mathcal{L}}(\phi_t,Z_t)\|_\mathcal{H}\le G_\phi$ and
the definition $E_t=\|\phi_t-\phi_\sharp(Z_t)\|_\mathcal{H}^2$, we obtain from \eqref{eq:gap_le_innerprod} that
\begin{equation}
\tilde{\mathcal{L}}(\phi_t,Z_t)-\tilde{\mathcal{L}}(\phi_\sharp(Z_t),Z_t)
\;\le\;
G_\phi\,\sqrt{E_t}.
\label{eq:gap_le_GsqrtE}
\end{equation}
Because $\min_{\phi\in K}\tilde{\mathcal{L}}(\phi,Z_t)=\tilde{\mathcal{L}}(\phi_\sharp(Z_t),Z_t)$, \eqref{eq:gap_le_GsqrtE}
is exactly
\[
\tilde{\mathcal{L}}(\phi_t,Z_t)-\min_{\phi\in K}\tilde{\mathcal{L}}(\phi,Z_t)
\;\le\;
G_\phi\,\sqrt{E_t}.
\]

Summing \eqref{eq:gap_le_GsqrtE} from $t=1$ to $T$ and dividing by $T$ gives
\[
\frac{1}{T}\sum_{t=1}^T\Bigl(\tilde{\mathcal{L}}(\phi_t,Z_t)-\min_{\phi\in K}\tilde{\mathcal{L}}(\phi,Z_t)\Bigr)
\;\le\;
\frac{G_\phi}{T}\sum_{t=1}^T \sqrt{E_t}.
\]
Apply Cauchy--Schwarz to the vectors $(1,\dots,1)\in\mathbb R^T$ and $(\sqrt{E_1},\dots,\sqrt{E_T})\in\mathbb R^T$:
\[
\sum_{t=1}^T \sqrt{E_t}
\;\le\;
\sqrt{T}\,\sqrt{\sum_{t=1}^T E_t}.
\]
Therefore,
\begin{equation}
\frac{1}{T}\sum_{t=1}^T\Bigl(\tilde{\mathcal{L}}(\phi_t,Z_t)-\min_{\phi\in K}\tilde{\mathcal{L}}(\phi,Z_t)\Bigr)
\;\le\;
G_\phi\,\sqrt{\frac{1}{T}\sum_{t=1}^T E_t}.
\label{eq:avg_gap_le_sqrt_avgE}
\end{equation}
By Proposition \ref{prop:E-rec}, with the above choice of $\delta$ we have for all $t\ge 0$,
\begin{equation}
E_{t+1}\;\le\; a\,E_t + b\,\|Z_{t+1}-Z_t\|_F^2,
\label{eq:E_recursion_ab}
\end{equation}
where $a\in(0,1)$ and $b>0$ are defined in the lemma statement. We claim that for all $T\ge 1$,
\begin{equation}
\sum_{t=1}^T E_t
\;\le\;
\frac{a}{1-a}\,E_0 + \frac{b}{1-a}\sum_{t=0}^{T-1}\|Z_{t+1}-Z_t\|_F^2.
\label{eq:sumE_bound_by_sumZdiff}
\end{equation}
To prove \eqref{eq:sumE_bound_by_sumZdiff}, first unroll \eqref{eq:E_recursion_ab}:
for any $t\ge 0$, repeated substitution yields
\begin{equation}
E_{t+1}
\;\le\;
a^{t+1}E_0
+
b\sum_{k=0}^{t} a^{t-k}\|Z_{k+1}-Z_k\|_F^2.
\label{eq:E_unrolled}
\end{equation}
Now sum \eqref{eq:E_unrolled} over $t=0,1,\dots,T-1$:
\[
\sum_{t=0}^{T-1}E_{t+1}
\;\le\;
E_0\sum_{t=0}^{T-1}a^{t+1}
+
b\sum_{t=0}^{T-1}\sum_{k=0}^{t} a^{t-k}\|Z_{k+1}-Z_k\|_F^2.
\]
We bound the first term using the geometric series:
\[
\sum_{t=0}^{T-1}a^{t+1}
= a\sum_{t=0}^{T-1}a^{t}
\le a\sum_{t=0}^{\infty}a^t
=\frac{a}{1-a}.
\]
For the double sum, swap the order of summation. Observe that the index set is
$\{(k,t):0\le k\le t\le T-1\}$, hence
\[
\sum_{t=0}^{T-1}\sum_{k=0}^{t} a^{t-k}\|Z_{k+1}-Z_k\|_F^2
=
\sum_{k=0}^{T-1}\sum_{t=k}^{T-1} a^{t-k}\|Z_{k+1}-Z_k\|_F^2.
\]
For each fixed $k$, the inner sum is again a geometric series:
\[
\sum_{t=k}^{T-1} a^{t-k}
=
\sum_{j=0}^{T-1-k} a^{j}
\le
\sum_{j=0}^{\infty}a^j
=
\frac{1}{1-a}.
\]
Therefore,
\[
\sum_{t=0}^{T-1}\sum_{k=0}^{t} a^{t-k}\|Z_{k+1}-Z_k\|_F^2
\le
\sum_{k=0}^{T-1}\frac{1}{1-a}\,\|Z_{k+1}-Z_k\|_F^2
=
\frac{1}{1-a}\sum_{t=0}^{T-1}\|Z_{t+1}-Z_t\|_F^2.
\]
Combining these two bounds proves \eqref{eq:sumE_bound_by_sumZdiff}. Next, we bound $\sum_{t=0}^{T-1}\|Z_{t+1}-Z_t\|_F^2$ from $t=0$ to $T-1$, using Lemma \ref{lem:Z_suff_decrease_noise},
\[
\sum_{t=0}^{T-1}\Bigl(\tilde{\mathcal{L}}(\phi_t,Z_t)-\tilde{\mathcal{L}}(\phi_{t+1},Z_{t+1})\Bigr)
\;\ge\;
\sum_{t=0}^{T-1}\Bigl(\tilde{\mathcal{L}}(\phi_t,Z_t)-\tilde{\mathcal{L}}(\phi_t,Z_{t+1})\Bigr)
\;\ge\;
\frac{1}{2\zeta}\sum_{t=0}^{T-1}\|Z_{t+1}-Z_t\|_F^2.
\]
Hence,
\begin{equation}
\sum_{t=0}^{T-1}\|Z_{t+1}-Z_t\|_F^2
\;\le\;
2\zeta\Bigl(\tilde{\mathcal{L}}(\phi_0,Z_0)- L(\phi_T,Z_T)\Bigr)
\;\le\;
2\zeta\tilde{\mathcal{L}}(\phi_0,Z_0).
\label{eq:sumZdiff_bound}
\end{equation}
Substituting \eqref{eq:sumZdiff_bound} into \eqref{eq:sumE_bound_by_sumZdiff} yields
\[
\sum_{t=1}^T E_t
\;\le\;
\frac{a}{1-a}\,E_0 + \frac{2\zeta b}{1-a}\tilde{\mathcal{L}}(\phi_0,Z_0)
=: C_E,
\]
which is exactly \eqref{eq:CE_def} and is independent of $T$. Plugging $\sum_{t=1}^T E_t \le C_E$ into \eqref{eq:avg_gap_le_sqrt_avgE} gives
\[
\frac{1}{T}\sum_{t=1}^T\Bigl(\tilde{\mathcal{L}}(\phi_t,Z_t)-\min_{\phi\in K}\tilde{\mathcal{L}}(\phi,Z_t)\Bigr)
\;\le\;
G_\phi\,\sqrt{\frac{C_E}{T}},
\]
which completes the proof.
\end{proof}

\begin{lemma}\label{lem:Z_suff_decrease_noise}
From Lemma \ref{lem:Z-smoothness}, we know that $Z\mapsto \tilde{\mathcal{L}}(\phi,Z)$ is $L_Z$-smooth on $\mathcal C$ for any $\phi$ .
If $0<\zeta\le 1/L_Z$, then 
\[
\tilde{\mathcal{L}}(\phi^{t+1},Z^t)-\tilde{\mathcal{L}}(\phi^{t+1},Z^{t+1})
\ge \frac{1}{2\zeta}\|Z^{t+1}-Z^t\|_F^2 \ge 0, \quad \forall t.
\]
\end{lemma}

\begin{proof}
This is the standard projected-gradient sufficient decrease.
Let $S:=Z^{t+1}-Z^t$ and $g:=\nabla_Z\tilde{\mathcal{L}}(\phi^{t+1},Z^t)$.
From the projection optimality (choose $U=Z^t$ in the projection characterization),
\[
\langle g, S\rangle \le -\frac{1}{\zeta}\|S\|_F^2.
\]
By $L_Z$-smoothness in $Z$,
\[
\tilde{\mathcal{L}}(\phi^{t+1},Z^{t+1})
\le
\tilde{\mathcal{L}}(\phi^{t+1},Z^t) + \langle g, S\rangle + \frac{L_Z}{2}\|S\|_F^2
\le
\tilde{\mathcal{L}}(\phi^{t+1},Z^t) -\Bigl(\frac{1}{\zeta}-\frac{L_Z}{2}\Bigr)\|S\|_F^2.
\]
If $\zeta\le 1/L_Z$ then $(1/\zeta - L_Z/2)\ge 1/(2\zeta)$, concluding the claim.
\end{proof}

\section{Noisy Observation Setting with Known Variance \texorpdfstring{$\sigma^2>0$}{sigma^2>0}}\label{app:noise_extension}


We consider the noisy observation model
\begin{equation}
y_{i,t} \;=\; \phi^\star\!\big(\langle b_i^\star, f_t^\star\rangle\big) \;+\; u_{i,t},
\qquad (i,t)\in\Omega,
\label{eq:obs_model_noise}
\end{equation}
where $\Omega\subset[n]\times[T]$ is the observed index set with $|\Omega|=M$.
The noise variables $\{u_{i,t}\}$ are i.i.d. and satisfy:
\[
\mathbb E[u_{i,t}] = 0,\qquad \mathrm{Var}(u_{i,t})=\sigma^2 \ \text{(known)},\qquad
u_{i,t}\ \text{is sub-Gaussian with parameter }\sigma^2.
\]

For each $k=(i,t)\in\Omega$, define the scalar mapping $z_k(Z):=\langle A_k, ZZ^\top\rangle$,
so that \eqref{eq:obs_model_noise} reads $y_k=\phi^\star(z_k(Z^\star))+u_k$.

\subsection{Algorithm and Lyapunov potential}
\label{app:noise_algorithm_lyapunov}

We consider the projected BCD updates
\begin{equation}
\phi^{t+1} = \mathcal P_{\mathcal H}\!\big(\phi^t-\eta \nabla_\phi \tilde{\mathcal{L}}(\phi^t,Z^t)\big),
\qquad
Z^{t+1} = \mathcal P_{\mathcal C}\!\big(Z^t-\zeta \nabla_Z \tilde{\mathcal{L}}(\phi^{t+1},Z^t)\big),
\label{eq:BCD_updates_noise}
\end{equation}
with stepsizes $0<\eta\le 1/(\alpha+L_{\phi,\alpha})$ and $0<\zeta\le 1/L_Z$ (or a backtracking rule). We further define the Lyapunov quantities
\begin{equation}
E_t := \|\phi^t-\phi_\sharp(Z^t)\|_\mathcal{H}^2,\qquad
D_t := \|\Delta(Z^t)\|_F^2,\qquad
\mathcal V_t := E_t + \gamma D_t,
\label{eq:Lyapunov_noise}
\end{equation}
where $\gamma>0$ will be specified in the main theorem.

\begin{lemma}\label{lem:local-curvature-noisy}
Under the Assumptions of Lemma \ref{le:local_curvature} plus the assumption in \ref{ass:noise_subgaussi}, for any $\delta\in(0,1)$, with probability at least $1-2\delta$ over the noise $\{u_k\}$, we have
\begin{align*}
&\big\langle \nabla_Z \tilde{\mathcal{L}}(\phi, Z), \Delta \big\rangle\ge\mu_Z \|\Delta\|_F^2\!-\!A_Z \|\Delta\|_F \|\phi\!-\phi^*\|_{\mathcal H}-\varepsilon_1(\sigma)\|\Delta\|_F,
\end{align*}
where $A_Z,\mu_Z$ are the same constants as in the noiseless case and $\varepsilon_1(\sigma)\in\mathcal{O}(\Xi\sigma\sqrt{\log(2/\delta)})$.
\end{lemma}
\begin{proof}
We start from the explicit gradient expression (Appendix 4.2):
\[
\nabla_Z \tilde{\mathcal{L}}(Z,\phi)
=
2\sum_{k=1}^M
\big(\phi(z_k)-y_k\big)\phi'(z_k)\,(A_k + A_k^\top)Z
\;+\;
\lambda DZZ^\top DZ,
\]
hence
\begin{equation}
\label{eq:innerprod-gradZ}
\big\langle \nabla_Z \tilde{\mathcal{L}}(Z,\phi), \Delta \big\rangle
=
2\sum_{k=1}^M h_k \, \big\langle (A_k + A_k^\top)Z, \Delta \big\rangle
\;+\;
\lambda \big\langle DZZ^\top DZ, \Delta \big\rangle,
\end{equation}
where $h_k := \phi'(z_k)\big(\phi(z_k)-y_k\big)$ and $z_k=\langle A_k,ZZ^\top\rangle$.

Under the noisy model $y_k=\phi^*(z_k^*)+u_k$, we can write
\[
h_k
=
\phi'(z_k)\big(\phi(z_k)-\phi^*(z_k^*)\big)
\;-\;
\phi'(z_k)u_k
=: h_k^{(0)} \;+\; h_k^{(u)} ,
\]
where $h_k^{(u)} := -\phi'(z_k)u_k$.

Plugging this decomposition into \eqref{eq:innerprod-gradZ} yields
\begin{align}
\label{eq:split}
\big\langle \nabla_Z \tilde{\mathcal{L}}(Z,\phi), \Delta \big\rangle
=&
\underbrace{
2\sum_{k=1}^M h_k^{(0)} \, \big\langle (A_k + A_k^\top)Z, \Delta \big\rangle
\;+\;
\lambda \big\langle DZZ^\top DZ, \Delta \big\rangle
}_{=:~\mathsf{T}_0}
\\ 
&+\;
\underbrace{
2\sum_{k=1}^M h_k^{(u)} \, \big\langle (A_k + A_k^\top)Z, \Delta \big\rangle
}_{=:~\mathsf{T}_u}.
\end{align}

The term $\mathsf{T}_0$ is exactly the quantity analyzed in the noiseless Appendix \ref{pr:le:local_curvature} (after replacing $y_k$ by $\phi^*(z_k^*)$).
Therefore, by repeating the same algebraic expansions and inequalities as in the noiseless proof,
we obtain the following lower bound with probability at least $1-\delta$
\begin{equation}
\label{eq:T0-bound}
\mathsf{T}_0
\;\ge\;
\mu_Z \|\Delta\|_F^2
\;-\;
A_Z \|\Delta\|_F \|\phi-\phi^*\|_\mathcal{H},
\end{equation}
 when the number of observations is
$$
M\in\mathcal{O}\Big(\frac{(\mu r\kappa)^2\min\{n,T\}}{\epsilon^2}\log\big(\frac{n+T}{\delta}\big)\Big).
$$

From the definition of $\mathsf{T}_u$ and $h_k^{(u)}=-\phi'(z_k)u_k$, we have
\[
\mathsf{T}_u
=
-2\sum_{k=1}^M
u_k\,\phi'(z_k)\,\big\langle (A_k + A_k^\top)Z, \Delta \big\rangle .
\]
Define the deterministic coefficients 
\[
s_k
:=
\phi'(z_k)\,\big\langle (A_k + A_k^\top)Z, \Delta \big\rangle .
\]
Then $\mathsf{T}_u=-2\sum_{k=1}^M u_k s_k$.
Since $u_k$ are independent mean-zero $\sigma$-sub-Gaussian, the weighted sum
$\sum_{k=1}^M u_k s_k$ is sub-Gaussian with parameter at most $\sigma^2\|s\|_2^2$.
Hence, for any $\delta\in(0,1)$, with probability at least $1-\delta$,
\begin{equation}
\label{eq:subg-tail}
\Big|\sum_{k=1}^M u_k s_k\Big|
\;\le\;
\sigma \|s\|_2 \sqrt{2\log(2/\delta)} .
\end{equation}
Using $\phi'(z_k)\le \Xi$ and the identity
\[
\big\langle (A_k + A_k^\top)Z, \Delta \big\rangle
=
\langle A_k, Z\Delta^\top\rangle + \langle A_k, \Delta Z^\top\rangle
=
\langle A_k, Z\Delta^\top + \Delta Z^\top\rangle,
\]
we get
\[
\|s\|_2^2
=
\sum_{k=1}^M s_k^2
\le
\Xi^2 \sum_{k=1}^M
\langle A_k, Z\Delta^\top + \Delta Z^\top\rangle^2
=
\Xi^2\,\mathcal D(Z\Delta^\top + \Delta Z^\top).
\]
Therefore, combining with \eqref{eq:subg-tail}, we obtain with probability at least $ 1-\delta$:
\[
|\mathsf{T}_u|
=
2\Big|\sum_{k=1}^M u_k s_k\Big|
\le
2\sigma \|s\|_2 \sqrt{2\log(2/\delta)}
\le
2\sigma \Xi \sqrt{2\log(2/\delta)} \, \sqrt{\mathcal D(Z\Delta^\top + \Delta Z^\top)} .
\]
Finally, we upper bound the operator term by a Frobenius bound:
since $\mathcal D(Y)=\sum_{k=1}^M \langle A_k,Y\rangle^2 \le \|Y\|_F^2$ for the canonical sampling operators,
and in any case $\|Z\Delta^\top + \Delta Z^\top\|_F \le 2\|Z\|_F\|\Delta\|_F$, we get
\[
\sqrt{\mathcal  D(Z\Delta^\top + \Delta Z^\top)}
\;\le\;
\|Z\Delta^\top + \Delta Z^\top\|_F
\;\le\;
2\|Z\|_F\|\Delta\|_F.
\]
Hence, with probability at least $1-\delta$,
\begin{equation}
\label{eq:Tu-bound}
\mathsf{T}_u
\ge
-4\,\Xi\,\sigma\,\|Z\|_F\,\sqrt{\log(2/\delta)}\,\|\Delta\|_F .
\end{equation}

Combining \eqref{eq:split}, \eqref{eq:T0-bound}, \eqref{eq:Tu-bound}, and using the union bound, we conclude that with probability
at least $1-2\delta$,
\[
\big\langle \nabla_Z \tilde{\mathcal{L}}(Z,\phi), \Delta \big\rangle
\ge
\mu_Z \|\Delta\|_F^2
-
A_Z \|\Delta\|_F \|\phi-\phi^*\|_\mathcal{H}
-
4\,\Xi\,\sigma\,\|Z\|_F\,\sqrt{\log(2/\delta)}\,\|\Delta\|_F .
\]
\end{proof}

\begin{lemma}\label{lem:local-smoothness-noisy}
Under the same assumptions as in Lemma \ref{lem:local-curvature-noisy},  for any $\delta\in(0,1)$, with probability at least $1-2\delta$,
\[
\!\|\nabla_Z \tilde{\mathcal{L}}(\phi, Z)\|_F
\!\le\! \sqrt{B_Z}\|\Delta\|_F\!+\!\sqrt{A'_Z}\|\phi-\phi^\star\|_{\mathcal{H}}\!+
\!\varepsilon_2(\sigma),
\]
where $A'_Z,B_Z$ are the same constants as in the noiseless case and 
\[
\varepsilon_2(\sigma)\in\mathcal{O}\Big(\Xi\,\sigma\,
\sqrt{\frac{M\mu r\sigma_1^\star}{n+T}\log\!\big(\frac{2(n+T)r}{\delta}\big)}\Big).
\]
\end{lemma}
\begin{proof}

From the expression of $\nabla_Z\tilde{\mathcal{L}}$,
\[
\nabla_Z\tilde{\mathcal{L}}(Z,\phi)
=
2\sum_{k=1}^M(\phi(z_k)-y_k)\phi'(z_k)(A_k+A_k^\top)Z
+\lambda DZZ^\top DZ,
\qquad z_k=\langle A_k,ZZ^\top\rangle.
\]
Under $y_k=\phi^\star(z_k^\star)+u_k$ with $z_k^\star=\langle A_k,\Phi\Phi^\top\rangle$,
define the noiseless surrogate gradient
\[
g_0(Z,\phi)
:=
2\sum_{k=1}^M(\phi(z_k)-\phi^\star(z_k^\star))\phi'(z_k)(A_k+A_k^\top)Z
+\lambda DZZ^\top DZ,
\]
and the noise term
\[
g_u(Z,\phi)
:=
-2\sum_{k=1}^M u_k\,\phi'(z_k)\,(A_k+A_k^\top)Z.
\]
Then $\nabla_Z\tilde{\mathcal{L}}(Z,\phi)=g_0(Z,\phi)+g_u(Z,\phi)$, hence
\[
\|\nabla_Z\tilde{\mathcal{L}}(Z,\phi)\|_F
\le \|g_0(Z,\phi)\|_F+\|g_u(Z,\phi)\|_F.
\]

By the result of Lemma \ref{lemma:smoothness}, when the number of observation follows \eqref{eq:number_observation}, we get
\[
\|g_0(Z,\phi)\|_F
\le \sqrt{B_Z}\|\Delta(Z)\|_F + \sqrt{A'_Z}\|\phi-\phi^\star\|_\mathcal{H}.
\]
with probability at least $1-\delta$

Let $B_k:=(A_k+A_k^\top)Z$ and
\[
S:=\sum_{k=1}^M u_k\,\phi'(z_k)\,B_k,
\quad\text{so that}\quad g_u(Z,\phi)=-2S.
\]
Let $d:=(n+T)r$ be the number of entries of an $(n+T)\times r$ matrix. For each $(i,j)$,
\[
S_{ij}=\sum_{k=1}^M u_k c_{k,ij},\qquad c_{k,ij}:=\phi'(z_k)(B_k)_{ij}.
\]
Since $u_k$ are independent mean-zero $\sigma$-sub-Gaussian, the tail bound gives
\[
\mathbb P(|S_{ij}|\ge t)\le
2\exp\!\Big(-\frac{t^2}{2\sigma^2\sum_{k=1}^M c_{k,ij}^2}\Big).
\]
Let
\[
t_{ij}:=\sigma\sqrt{2\log\Big(\frac{2d}{\delta}\Big)}\,
\sqrt{\sum_{k=1}^M c_{k,ij}^2}.
\]
Then $\mathbb P(|S_{ij}|\ge t_{ij})\le \delta/d$. A union bound implies that with
probability at least $1-\delta$, $|S_{ij}|\le t_{ij}$ for all $(i,j)$, hence
\begin{align*}
\|S\|_F^2
&=\sum_{i,j}S_{ij}^2
\le \sum_{i,j} t_{ij}^2 \\
&=
2\sigma^2\log\Big(\frac{2d}{\delta}\Big)\sum_{i,j}\sum_{k=1}^M c_{k,ij}^2 \\
&=
2\sigma^2\log\Big(\frac{2d}{\delta}\Big)\sum_{k=1}^M \phi'(z_k)^2\|B_k\|_F^2 \\
&\le
2\sigma^2\Xi^2\log\Big(\frac{2d}{\delta}\Big)\sum_{k=1}^M \|B_k\|_F^2.
\end{align*}

By the structure of $A_k$ (each observation touches at most two rows in the bipartite
block), $(A_k+A_k^\top)Z$ has at most two nonzero rows, each equal to a row of $Z$.
Therefore
\[
\|B_k\|_F^2=\|(A_k+A_k^\top)Z\|_F^2 \le 2\|Z\|_{2,\infty}^2,
\]
and thus $\sum_{k=1}^M\|B_k\|_F^2\le 2M\|Z\|_{2,\infty}^2$.
Plugging into Step~3 gives (with prob.\ $\ge 1-\delta$)
\[
\|S\|_F^2\le
4\sigma^2\Xi^2 M\|Z\|_{2,\infty}^2\log\Big(\frac{2(n+T)r}{\delta}\Big)
\Rightarrow
\|S\|_F\le
2\Xi\sigma\|Z\|_{2,\infty}\sqrt{M\log\Big(\frac{2(n+T)r}{\delta}\Big)}.
\]
Since $g_u=-2S$ and $Z\in \mathcal{C}$,
\begin{align*}
\|g_u(Z,\phi)\|_F &\le
4\Xi\sigma\|Z\|_{2,\infty}\sqrt{M\log\Big(\frac{2(n+T)r}{\delta}\Big)}\\
&\leq\frac{16\Xi\sigma}{3}\|Z_0\|_{F}\sqrt{\frac{\mu M}{(n+T)}\log\Big(\frac{2(n+T)r}{\delta}\Big)}
=: \varepsilon_2(\sigma).    
\end{align*}

Combining Step~2 and Step~4 yields
\[
\|\nabla_Z \tilde{\mathcal{L}}(Z,\phi)\|_F
\le \sqrt{B_Z}\|\Delta(Z)\|_F + \sqrt{A'_Z}\|\phi-\phi^\star\|_\mathcal{H}
+
\varepsilon_2(\sigma).
\]

\end{proof}



\begin{proposition}\label{lemma:d_noisy_recu}
If $\zeta\le \mu_Z/(12B_Z)$, then
\[
D_{t+1}
\le \rho_Z D_t + 2C_{Z\leftarrow\phi}E_t + 2C_{Z\leftarrow\phi}\|\phi^\sharp(Z^t)-\phi^\star\|_\mathcal{H}^2 +
C_{Z,\mathrm{noise}},
\]
where
\[
\rho_Z:=1-\frac{\zeta\mu_Z}{4},\qquad
C_{Z\leftarrow\phi}:=3\zeta^2A_Z'+\frac{\zeta A_Z^2}{\mu_Z},\qquad
C_{Z,\mathrm{noise}}:=3\zeta^2\varepsilon_2(\sigma)^2+\frac{2\zeta}{\mu_Z}\varepsilon_1(\sigma)^2.
\]
\end{proposition}

\begin{proof}
From the previous lemma, we have the noisy curvature and smoothness bounds:
\begin{align}\label{eq:curvature_noisy_appndix}
\langle \nabla_Z\tilde{\mathcal{L}}(\phi,Z),\Delta(Z)\rangle
&\ge
\mu_Z\|\Delta(Z)\|_F^2
-
A_Z\|\Delta(Z)\|_F\|\phi-\phi^\star\|_\mathcal{H}
-
\varepsilon_1(\sigma)\|\Delta(Z)\|_F,
\end{align}
\begin{align}\label{eq:smoothness_noisy_appndix}
\|\nabla_Z\tilde{\mathcal{L}}(\phi,Z)\|_F
&\le \sqrt{B_Z}\|\Delta(Z)\|_F+\sqrt{A'_Z}\|\phi-\phi^\star\|_\mathcal{H}+
\varepsilon_2(\sigma). 
\end{align}
Let $g_t:=\nabla_Z\tilde{\mathcal{L}}(\phi^t,Z^t)$ and $\Delta_t:=\Delta(Z^t)$.
By projection non-expansiveness
\[
\|\Delta_{t+1}\|_F \le \|\Delta_t-\zeta g_t\|_F.
\]
Squaring and expanding,
\[
D_{t+1}\le D_t+\zeta^2\|g_t\|_F^2-2\zeta\langle g_t,\Delta_t\rangle.
\]
By \eqref{eq:curvature_noisy_appndix},
\[
-2\zeta\langle g_t,\Delta_t\rangle
\le
-2\zeta\mu_Z D_t
+
2\zeta A_Z\sqrt{D_t}\,\|\phi^t-\phi^\star\|_\mathcal{H}
+
2\zeta \varepsilon_1(\sigma)\sqrt{D_t}.
\]
By \eqref{eq:smoothness_noisy_appndix} and $(x+y+z)^2\le 3x^2+3y^2+3z^2$,
\[
\|g_t\|_F^2
\le
3B_Z D_t + 3A'_Z\|\phi^t-\phi^\star\|_\mathcal{H}^2 + 3\varepsilon_2(\sigma)^2.
\]
Use Young's inequality:
\[
2\zeta A_Z\sqrt{D_t}\,\|\phi^t-\phi^\star\|_\mathcal{H}
\le \zeta\mu_Z D_t + \frac{\zeta A_Z^2}{\mu_Z}\|\phi^t-\phi^\star\|_\mathcal{H}^2,
\]
\[
2\zeta \varepsilon_1(\sigma)\sqrt{D_t}
\le \frac{\zeta\mu_Z}{2}D_t + \frac{2\zeta}{\mu_Z}\varepsilon_1(\sigma)^2.
\]
Combining yields
\[
D_{t+1} \le
\Bigl(1-\frac{\zeta\mu_Z}{2}+3\zeta^2B_Z\Bigr)D_t
+ \Bigl(3\zeta^2A_Z'+\frac{\zeta A_Z^2}{\mu_Z}\Bigr)\|\phi^t-\phi^\star\|_\mathcal{H}^2 + 3\zeta^2\varepsilon_2(\sigma)^2
+ \frac{2\zeta}{\mu_Z}\varepsilon_1(\sigma)^2.
\]
If $\zeta\le \mu_Z/(12B_Z)$, then $3\zeta^2B_Z\le \zeta\mu_Z/4$, hence
$1-\zeta\mu_Z/2+3\zeta^2B_Z\le 1-\zeta\mu_Z/4=:\rho_Z$.
Finally,
\[
\|\phi^t-\phi^\star\|_\mathcal{H}^2
\le
2\|\phi^t-\phi^\sharp(Z^t)\|_\mathcal{H}^2 + 2\|\phi^\sharp(Z^t)-\phi^\star\|_\mathcal{H}^2
=
2E_t + 2\|\phi^\sharp(Z^t)-\phi^\star\|_\mathcal{H}^2,
\]
which proves the stated inequality.
\end{proof}

\begin{lemma}\label{lem:Z_diff_noisy}
We have
\[
\|Z^{t+1}-Z^t\|_F^2
\le
c_D D_t
+
c_E E_t
+
c_E\|\phi^\sharp(Z^t)-\phi^\star\|_\mathcal{H}^2
+
c_{\mathrm{sm}},
\]
where $c_D:=3\zeta^2B_Z$, $c_E:=6\zeta^2A_Z'$, and $c_{\mathrm{sm}}:=3\zeta^2\varepsilon_2(\sigma)^2$.
\end{lemma}

\begin{proof}
Let $g_t:=\nabla_Z\tilde{\mathcal{L}}(\phi^t,Z^t)$. By optimality of the projection
$Z^{t+1}=\mathcal P_{\mathcal C}(Z^t-\zeta g_t)$, for all $Y\in\mathcal C$,
\[
\langle Z^t-Z^{t+1}-\zeta g_t,\; Z^{t+1}-Y\rangle \ge 0.
\]
Choose $Y=Z^t$ to obtain
$\|Z^{t+1}-Z^t\|_F^2 \le \zeta\langle g_t, Z^t-Z^{t+1}\rangle
\le \zeta\|g_t\|_F\|Z^{t+1}-Z^t\|_F$.
If $\|Z^{t+1}-Z^t\|_F=0$, we are done; otherwise divide and square:
\[
\|Z^{t+1}-Z^t\|_F^2 \le \zeta^2\|g_t\|_F^2.
\]
By (22') and $(x+y+z)^2\le 3x^2+3y^2+3z^2$,
\[
\|g_t\|_F^2 \le 3B_ZD_t + 3A_Z'\|\phi^t-\phi^\star\|_\mathcal{H}^2 + 3\varepsilon_2(\sigma)^2.
\]
Moreover,
$\|\phi^t-\phi^\star\|_\mathcal{H}^2 \le 2E_t + 2\|\phi^\sharp(Z^t)-\phi^\star\|_\mathcal{H}^2$.
Substitute these two bounds to conclude the claim.
\end{proof}


\subsection{Proof of Theorem \ref{thm:main_noisy}}\label{thm:main_noisy_p}
\begin{proof}
From Proposition \ref{prop:E-rec},
\[
E_{t+1}
\le
q_\phi(1+\delta)E_t
+
q_\phi\Bigl(1+\frac1\delta\Bigr)A\|Z^{t+1}-Z^t\|_F^2.
\]
Apply Lemma \ref{lem:Z_diff_noisy} to bound $\|Z^{t+1}-Z^t\|_F^2$, yielding
\[
E_{t+1}
\le
(q_\phi(1+\delta)+q_\phi\Big(1+\tfrac1\delta\Big)A c_E)E_t + q_\phi\Big(1+\tfrac1\delta\Big)A \big(c_DD_t + c_E\|\phi^\sharp(Z^t)-\phi^\star\|_\mathcal{H}^2 + 3\zeta^2\varepsilon_2(\sigma)^2\big),
\]
From Proposition \ref{lemma:d_noisy_recu},
\[
D_{t+1}
\le
\rho_Z D_t
+
2C_{Z\leftarrow\phi}E_t
+
2C_{Z\leftarrow\phi}\|\phi^\sharp(Z^t)-\phi^\star\|_\mathcal{H}^2
+
3\zeta^2\varepsilon_2(\sigma)^2+\frac{2\zeta}{\mu_Z}\varepsilon_1(\sigma)^2.
\]

Multiplying the second inequality by $\gamma$ and adding to the first, we obtain
\begin{equation}\label{eq:V-rec}
\begin{aligned}
\mathcal V_{t+1}
&=E_{t+1}+\gamma D_{t+1}\\
&\le \underbrace{\Big[q_\phi(1+\delta)+q_\phi\Big(1+\tfrac1\delta\Big)A c_E\ +\ 2\gamma\,C_{Z\leftarrow\phi}\Big]}_{a_{EE}}\,E_t\\
&\quad + \underbrace{\Big[q_\phi\Big(1+\tfrac1\delta\Big)A c_D + \gamma\,\rho_Z\Big]}_{a_{DD}/\gamma}\,D_t\\
&\quad + \underbrace{\Big[q_\phi\Big(1+\tfrac1\delta\Big)A c_E + 2\gamma C_{Z\leftarrow\phi}\Big]}_{C_\phi}\,\chi^2(Z_t)\\
&\quad +\underbrace{\frac{2\gamma\zeta}{\mu_Z}\varepsilon_1(\sigma)^2+3\Big[q_\phi\Big(1+\tfrac1\delta\Big)A c_E+\gamma \Big]\zeta^2\varepsilon_2(\sigma)^2}_{C_\sigma},
\end{aligned}
\end{equation}
where $\chi^2(Z_t):=\|\phi^\sharp(Z^t)-\phi^\star\|_\mathcal{H}^2$.

We require both $a_{EE}<1$ and $a_{DD}/\gamma<1$. We provide explicit, non-empty feasibility. 
Recall that we had $0<\eta\le \frac{2\alpha}{L_{\phi,\alpha}(\alpha+L_{\phi,\alpha})}$ and 
$
q_\phi=1-\frac{\eta\alpha L_{\phi,\alpha}}{\alpha+L_{\phi,\alpha}}\in(0,1),
$
Let
\begin{equation}\label{eq:delta}
\delta\ :=\ \frac{1/q_\phi -1}{2}\ (>0)
\quad\Rightarrow\quad
q_\phi(1+\delta)=\frac{1+q_\phi}{2}\ (<1),\qquad
q_\phi\Big(1+\frac{1}{\delta}\Big)=\frac{q_\phi(1+q_\phi)}{1-q_\phi}.
\end{equation}
Recall that
\begin{align*}
a_{EE}&=q_\phi(1+\delta)+q_\phi\Big(1+\tfrac1\delta\Big)A c_E\ +\ 2\gamma\,C_{Z\leftarrow\phi}\\
&=\frac{1+q_\phi}{2}+\frac{q_\phi(1+q_\phi)}{1-q_\phi}\Big(\frac{L_{Z\to\phi}}{\alpha}\Big)^2(6\zeta^2A_Z')\ +\ \gamma \,\Big(6\zeta^2A'_Z+\frac{2\zeta A_Z^2}{\mu_Z}\Big)\\
&\leq1-\frac{\eta\alpha L_{\phi,\alpha}}{2(\alpha+L_{\phi,\alpha})}+\frac{\alpha+L_{\phi,\alpha}}{\eta\alpha L_{\phi,\alpha}}\Big(\frac{L_{Z\to\phi}}{\alpha}\Big)^2(6\zeta^2A_Z')+\ \gamma \,\Big(6\zeta^2A_Z'+\frac{2\zeta A_Z^2}{\mu_Z}\Big).
\end{align*}
In order to have $\gamma(6\zeta^2A_Z'+\frac{2\zeta A_Z^2}{\mu_Z})\leq\frac{\eta\alpha L_{\phi,\alpha}}{4(\alpha+L_{\phi,\alpha})}$, following the same calculus as in the proof of Theorem \ref{thm:main},  we use
$\zeta\leq\frac{\eta\alpha^2L_{\phi,\alpha}}{4\sqrt{3}L_{Z\to\phi}\sqrt{A_Z'}(\alpha+L_{\phi,\alpha})}$ and select 
$$
\gamma= \frac{ L_{Z\to\phi}\mu_Z}{2\alpha \max\{\sqrt{A'_Z},A_Z\}}
,
$$
we ensure that $\gamma\leq \frac{\eta\alpha L_{\phi,\alpha}\mu_Z}{16(\alpha+L_{\phi,\alpha})\frac{\eta\alpha^2L_{\phi,\alpha}}{4\sqrt{3}L_{Z\to\phi}\sqrt{A_Z'}(\alpha+L_{\phi,\alpha})}\max\{A_Z^2,A'_Z\}}\leq\frac{\eta\alpha L_{\phi,\alpha}\mu_Z}{16(\alpha+L_{\phi,\alpha})\zeta A_Z^2}$.
By setting $\gamma \,\Big(6\zeta^2A_Z'+\frac{2\zeta A_Z^2}{\mu_Z}\Big)\leq\frac{\eta\alpha L_{\phi,\alpha}}{4(\alpha+L_{\phi,\alpha})}$ and $\zeta\leq\frac{\eta\alpha^2L_{\phi,\alpha}}{4\sqrt{3}L_{Z\to\phi}\sqrt{A'_Z}(\alpha+L_{\phi,\alpha})}$, we have
\begin{align*}
   a_{EE}&\leq 1-\frac{\eta\alpha L_{\phi,\alpha}}{4(\alpha+L_{\phi,\alpha})}+\frac{\alpha+L_{\phi,\alpha}}{\eta\alpha L_{\phi,\alpha}}\Big(\frac{L_{Z\to\phi}}{\alpha}\Big)^2\left(6\frac{\eta^2\alpha^4L_{\phi,\alpha}^2}{48L_{Z\to\phi}^2A_Z^2(\alpha+L_{\phi,\alpha})^2}A_Z^2\right)\\
   &=1-\frac{\eta\alpha L_{\phi,\alpha}}{8(\alpha+L_{\phi,\alpha})}<1.
\end{align*}
 and consequently, we get $a_{EE}<1$.

According to the definition of $a_{DD}$, we have
\[
\frac{a_{DD}}{\gamma}=\rho_Z+\frac{q_\phi(1+q_\phi)}{1-q_\phi}\cdot \frac{A\,c_D}{\gamma}
=1-\tfrac{\zeta\mu_Z}{2}+\frac{q_\phi(1+q_\phi)}{1-q_\phi}\cdot \frac{3\zeta^2AB_Z}{\gamma}.
\]
To ensure $\frac{a_{DD}}{\gamma}\leq 1-\tfrac{\zeta\mu_Z}{4}$, it suffices to have
\begin{equation*}
\frac{q_\phi(1+q_\phi)}{1-q_\phi}\cdot \frac{3\zeta^2AB_Z}{\gamma}\leq \tfrac{\zeta\mu_Z}{4}
\end{equation*}
i.e.,
\begin{align*}
    \zeta\leq \frac{\mu_Z(1-q_\phi)\gamma}{12AB_Z} < \frac{\mu_Z(1-q_\phi)\gamma}{6q_\phi(1+q_\phi)AB_Z}
\end{align*}

While $A=(L_{Z\to\phi}/\alpha)^2$ and $\gamma=L_{Z\to\phi}\mu_Z/2\alpha \max\{\sqrt{A'_Z},A_Z\}$, we obtain
\begin{align*}
\frac{\mu_Z(1-q_\phi)\,\gamma}{12\,A\,B_Z}
&=
\frac{\eta\,\alpha^2\,\mu_Z^2\,L_{\phi,\alpha}}{24\,\max\{A_Z,\sqrt{A'_Z}\}\,L_{Z\to\phi}\,(\alpha+L_{\phi,\alpha})\,B_Z}.
\end{align*}
Therefore, by selecting
\begin{equation}\label{zeta_requirement}
 \zeta \;\le\;
\frac{\eta\,\alpha^2\,\mu_Z^2\,L_{\phi,\alpha}}{24\,\max\{A_Z,\sqrt{A'_Z}\}\,L_{Z\to\phi}\,(\alpha+L_{\phi,\alpha})\,B_Z} ,  
\end{equation}

we ensure that $a_{DD}/\gamma<1$ and $\rho=\min\{a_{EE},a_{DD}/\gamma\}=\min\{1-\frac{\eta\alpha L_{\phi,\alpha}}{8(\alpha+L_{\phi,\alpha})},1-\frac{\zeta\mu_Z}{4}\}<1$. This holds if 
$$
\zeta\in\mathcal{O}\Big(\frac{\alpha^3}{(\alpha+L_{\phi,\alpha})^2(\mu r\kappa)^2}\big(\frac{\xi}{\Xi}\big)^5\frac{\sqrt{nT}}{B_K}\Big).
$$
For $C_\phi$, we have
\begin{align*}
C_\phi&=\Big[q_\phi\Big(1+\tfrac1\delta\Big)A c_E + 2\gamma C_{Z\leftarrow\phi}\Big]\\
    &\leq \Big(\frac{6q_\phi(1+q_\phi)}{1-q_\phi}\Big(\frac{L_{Z\to\phi}}{\alpha}\Big)^2\zeta^2A'_Z+\frac{\eta\alpha L_{\phi,\alpha}}{4(\alpha+L_{\phi,\alpha})}\Big)\\
    &\leq \Big(\frac{12(\alpha+L_{\phi,\alpha})}{\eta\alpha L_{\phi,\alpha}}\Big(\frac{L_{Z\to\phi}}{\alpha}\Big)^2(\frac{\eta\alpha^2L_{\phi,\alpha}}{4\sqrt{3}L_{Z\to\phi}A'_Z(\alpha+L_{\phi,\alpha})})^2A_Z'+\frac{\eta\alpha L_{\phi,\alpha}}{4(\alpha+L_{\phi,\alpha})}\Big)\\
    &\leq\frac{\eta\alpha L_{\phi,\alpha}}{2(\alpha+L_{\phi,\alpha})}:=C'_\phi\in\mathcal{O}(\eta\alpha ).
\end{align*}
The last term can be bounded as follows
\begin{align*}
    C_\sigma&=\frac{2\gamma\zeta}{\mu_Z}\varepsilon_1(\sigma)^2+3\Big[q_\phi\Big(1+\tfrac1\delta\Big)A c_E+\gamma \Big]\zeta^2\varepsilon_2(\sigma)^2\in\mathcal{O}(\eta\alpha\sigma^2)
\end{align*}

Thus, we have
\[
\mathcal{V}_{t+1}=\rho\mathcal{V}_t+C'_\phi\chi^2(Z_t)+C_\sigma.
\]
By telescoping over $t$, we can conclude the result.

\subsection{The discussion on the effect of  $\alpha$}\label{app:discussion_alpha}

From Proposition \ref{prop:E_converge}, we have $\lim_{t\to\infty}E_t=0$, when $\zeta\leq \min\{1/A_Z, \mu_Z/B_Z, 1/\mu_Z\}$.  
By considering Equation \eqref{zeta_requirement} and $\alpha$ small enough so that $\zeta\in\mathcal{O}(\eta\alpha^2)$ and $\zeta\leq \min\{1/A_Z, \mu_Z/B_Z, 1/\mu_Z\}$ is satisfied, we will have $\rho=\min\{1-\frac{\eta\alpha L_{\phi,\alpha}}{8(\alpha+L_{\phi,\alpha})},1-\frac{\zeta\mu_Z}{4}\}=1-C_1\eta\alpha^2$ for some constant $C_1$. 
As a result, 
\begin{align*}
\lim_{t\to\infty}D_t=\frac{C'_\phi\chi^2(Z_\infty)+C_\sigma}{\gamma(1-\rho)}=\mathcal{O}\Big(\frac{\eta\alpha^2(\chi^2(Z_\infty)+\sigma^2)}{1-(1-\eta\alpha^2)}\Big)=\mathcal{O}(\chi^2(Z_\infty)+\sigma^2),   
\end{align*}
and consequently $\|\Delta_t\|_F\in\mathcal{O}(\chi(Z_\infty)+\sigma)$ as $t\to\infty$. 

On the other hand, when $\alpha$ is large so that the upper bound for Equation \eqref{zeta_requirement} is larger than $\min\{1/A_Z, \mu_Z/B_Z, 1/\mu_Z\}$, then we will have $\rho=\min\{1-\frac{\eta\alpha L_{\phi,\alpha}}{8(\alpha+L_{\phi,\alpha})},1-\frac{\zeta\mu_Z}{4}\}=1-C_2$ for some constant $C_2$ which does not depend on $\alpha$. Thus, we have 
\begin{align*}
\lim_{t\to\infty}D_t=\frac{C'_\phi\chi^2(Z_\infty)+C_\sigma}{\gamma(1-\rho)}=\mathcal{O}\Big(\frac{\eta\alpha^2(\chi^2(Z_\infty)+\sigma^2)}{1-(1-C_2)}\Big)=\mathcal{O}\left(\alpha(\chi^2(Z_\infty)+\sigma^2)\right),
\end{align*}
where we used $\eta\leq\frac{2\alpha}{L_{\phi,\alpha}(\alpha+L_{\phi,\alpha})}$.
It is noteworthy that when $\alpha$ is small, the corresponding $\rho$, i.e., $1-C_1\eta\alpha^2$ is larger than when $\alpha$ is large, i.e., $1-C_2$ leading to faster convergence, since the convergence rate is of the order $(1-\rho)^t$. At the same time, as discussed above, the convergence point achieves a smaller estimation error for smaller $\alpha$.

\end{proof}

\section{Additional Experiments}\label{app:additional_exp}
To show visually the difference between the true link function and the final estimation produced by Algorithm \ref{alg:pbcd}, herein, we present Figure \ref{fig:both:piece_phi}. 
In this experiment, the parameters are $n=T=100, r=3,M=5000, \sigma=0.1,\lambda=0.5,\zeta=10^{-5},\eta=10^{-4}, \alpha=10^{-3}$.

\begin{figure}[h]
     \centering
        \includegraphics[width=.7\linewidth,height=.4\linewidth]{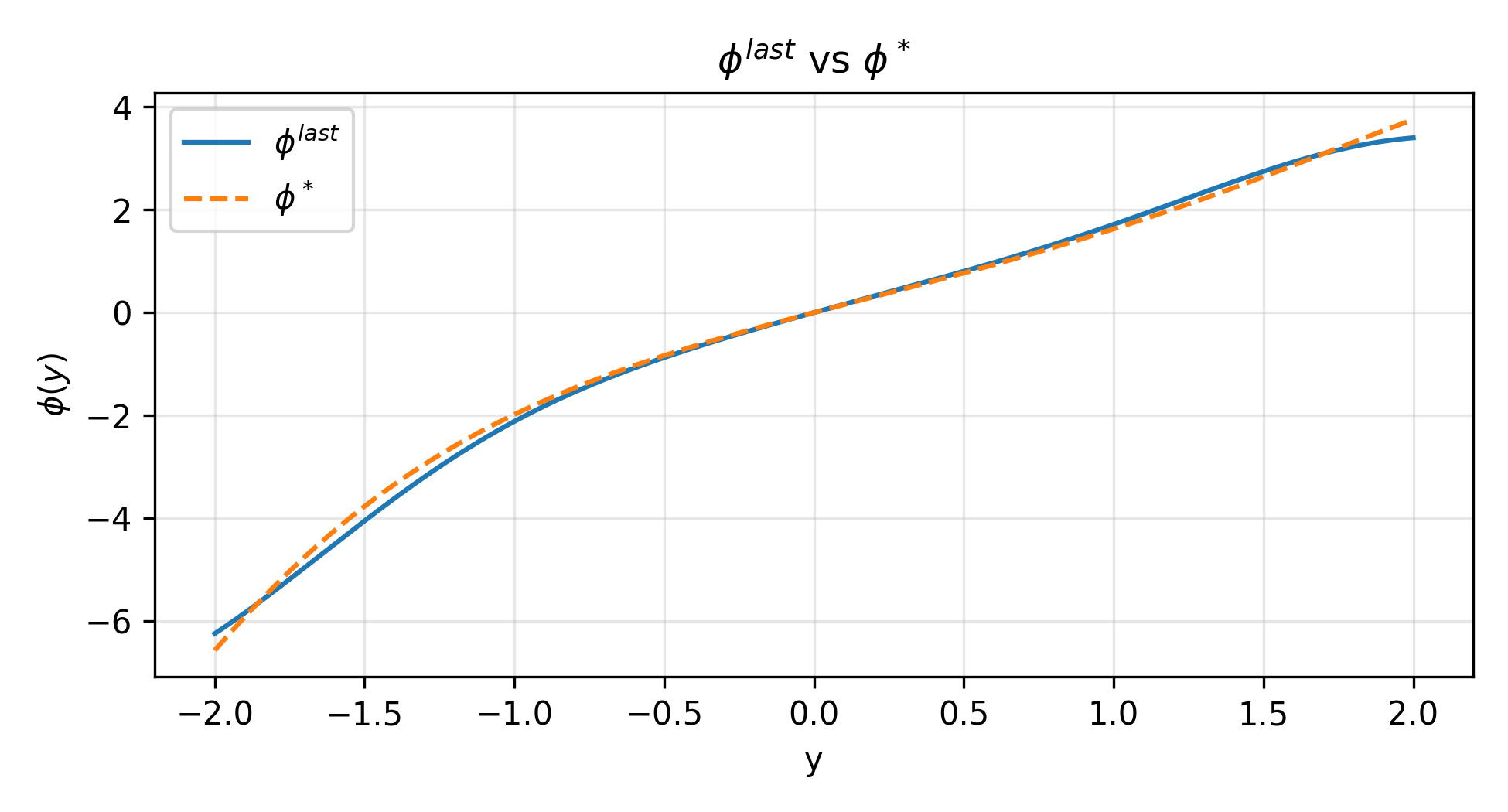}
        \caption{The comparison of the true link function $\phi^\star$ and the final estimation $\phi^{last}$. 
        }
    \label{fig:both:piece_phi}
\end{figure}

\textit{Effect of different type of link function on regret:} 
To illustrate this effect, Figure \ref{fig:both:piece_phi} shows the regret of Algorithm \ref{alg:pbcd} over iterations for different link function.

\begin{figure}[h]
     \centering
        \includegraphics[width=.7\linewidth,height=.5\linewidth]{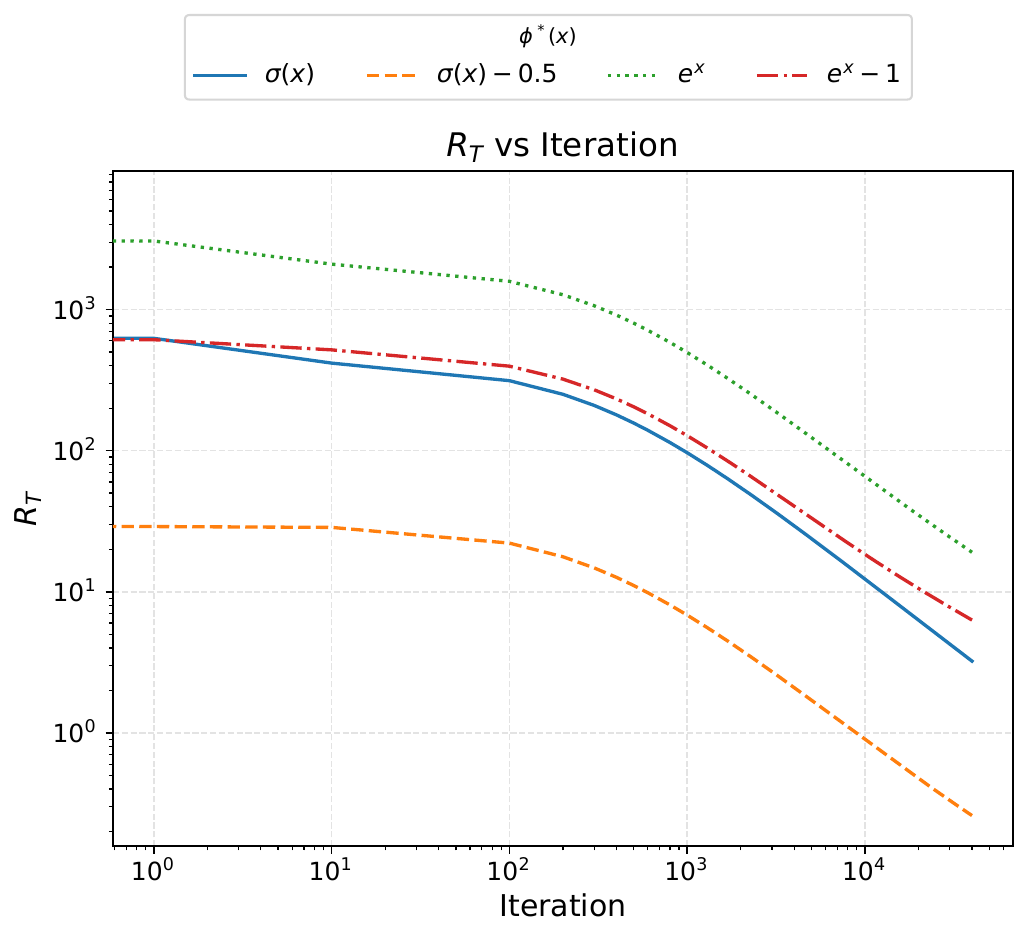}
        \caption{Regret $R_T$ of Algorithm \ref{alg:pbcd} for different type of link functions $\phi^\star$ when $n=T=100, r=3, M=5000, \alpha=0.001$, and $\sigma=0.1$. Both axes are plotted on a log scale.}
\end{figure}

\end{document}